%% file: main.tex
\documentclass[journal]{IEEEtran} 
\input{header.tex}
\usepackage{comment}
\usepackage{pifont}
\newcommand{\cmark}{\ding{51}}%
\newcommand{\xmark}{\textcolor{red}{\ding{55}}}%
\usepackage[export]{adjustbox}
\makeatletter
\IEEEtriggercmd{\reset@font\normalfont\fontsize{5.0pt}{5.0pt}\selectfont}
\makeatother
\IEEEtriggeratref{1}

\makeatletter
\newcommand\fs@ruled@notop{\def\@fs@cfont{\bfseries}\let\@fs@capt\floatc@ruled
  \def\@fs@pre{}%
  \def\@fs@post{\kern2pt\hrule\relax}%
  \def\@fs@mid{\kern2pt\hrule\kern2pt}%
  \let\@fs@iftopcapt\iftrue}
\renewcommand\fst@algorithm{\fs@ruled@notop}
\makeatother

\title{SDRS: Shape-Differentiable Robot Simulator}
\author{Xiaohan Ye$^1$, Xifeng Gao$^2$, Kui Wu$^2$, Zherong Pan$^{2\dagger}$, and Taku Komura$^{1\dagger}$  \\
\thanks{$^\dagger$ indicates corresponding author. $^2$LightSpeed Studios, Tencent. \{xifgao, kwwu, zrpan\}@global.tencent.com. $^1$The Department of Computer Science, Hong Kong University. \{u3010417@connect.hku.hk, taku@cs.hku.hk\}}}

\begin{document}
\maketitle
\allowdisplaybreaks
\thispagestyle{empty}
\pagestyle{empty}
\begin{abstract}
Robot simulators are indispensable tools across many fields, and recent research has significantly improved their functionality by incorporating additional gradient information. However, existing differentiable robot simulators suffer from non-differentiable singularities, when robots undergo substantial shape changes. To address this, we present the Shape-Differentiable Robot Simulator (SDRS), designed to be differentiable under significant robot shape changes.
The core innovation of SDRS lies in its representation of robot shapes using a set of convex polyhedrons. This approach allows us to generalize smooth, penalty-based contact mechanics for interactions between any pair of convex polyhedrons. Using the separating hyperplane theorem, SDRS introduces a separating plane for each pair of contacting convex polyhedrons. This separating plane functions as a zero-mass auxiliary entity, with its state determined by the principle of least action. This setup ensures global differentiability, even as robot shapes undergo significant geometric and topological changes.
To demonstrate the practical value of SDRS, we provide examples of robot co-design scenarios, where both robot shapes and control movements are optimized simultaneously.
\end{abstract}

\input{introduction.tex}
\input{related.tex}
\input{problem.tex}
\input{method.tex}
\input{evaluation.tex}
\input{conclusion.tex}

\printbibliography
\input{appendix/contact.tex}
\input{appendix/friction.tex}
\input{appendix/differentiable.tex}

\end{document}

%% file: header.tex
\usepackage{amsmath,amsfonts,amsthm}
\usepackage{amssymb}
\usepackage{prettyref}
\usepackage{mathrsfs}
\usepackage{graphicx}
\usepackage{wrapfig}
\usepackage{subfig}
\usepackage{framed}
\usepackage{multirow}
\usepackage{booktabs}
\usepackage{algorithm}
\usepackage[table]{xcolor}
\usepackage[noend]{algpseudocode}
\usepackage
[backend=bibtex,
bibstyle=ieee,
citestyle=numeric,sortcites,
mincitenames=1,
maxcitenames=2,
natbib=true,
doi=false,
isbn=false,
url=false,
eprint=false]{biblatex}
\usepackage[colorlinks,allcolors=gray,hypertexnames=true]{hyperref} 

\newtheorem{theorem}{Theorem}[section]
\newtheorem{lemma}[theorem]{\TE{Lemma}}

\algnewcommand{\LineComment}[1]{\State \(\triangleright\) #1}
\algdef{SE}[DOWHILE]{Do}{doWhile}{\algorithmicdo}[1]
{\algorithmicwhile\ #1}
\newcommand*{\colorboxed}{}
\def\colorboxed#1#{%
  \colorboxedAux{#1}%
}
\newcommand*{\colorboxedAux}[3]{%
  \begingroup
    \colorlet{cb@saved}{.}%
    \color#1{#2}%
    \boxed{%
      \color{cb@saved}%
      #3%
    }%
  \endgroup
}

\newrefformat{fig}{Figure~\ref{#1}}
\newrefformat{par}{Section~\ref{#1}}
\newrefformat{appen}{Appendix~\ref{#1}}
\newrefformat{sec}{Section~\ref{#1}}
\newrefformat{sub}{Section~\ref{#1}}
\newrefformat{table}{Table~\ref{#1}}
\newrefformat{ass}{Assumption~\ref{#1}}
\newrefformat{alg}{Algorithm~\ref{#1}}
\newrefformat{def}{Definition~\ref{#1}}
\newrefformat{thm}{Theorem~\ref{#1}}
\newrefformat{lem}{Lemma~\ref{#1}}
\newrefformat{step}{Step~\ref{#1}}
\newrefformat{ln}{Line~\ref{#1}}
\newrefformat{eq}{Equation~\ref{#1}}
\newrefformat{pb}{Problem~\ref{#1}}
\newrefformat{it}{Item~\ref{#1}}
\newrefformat{te}{Term~\ref{#1}}
\def\Eqref Eq:#1:{\eqref{eq:#1}}
\newrefformat{Eq}{Equation~\Eqref#1:}

\newcommand{\TE}[1]{\textbf{#1}}

\newcommand{\ResizedEq}[1]{\resizebox{0.87\hsize}{!}{$\begin{aligned}#1\end{aligned}$}}

\newcommand{\FPP}[2]{\frac{\partial{#1}}{\partial{#2}}}

\newcommand{\FPPT}[2]{\frac{\partial^2{#1}}{\partial{#2}^2}}

\newcommand{\FPPTT}[3]{\frac{\partial^2{#1}}{\partial{#2}\partial{#3}}}

\newcommand{\FDD}[2]{\frac{d{#1}}{d{#2}}}
\newcommand{\FDDT}[2]{\frac{d^2{#1}}{d{#2}^2}}
\newcommand{\FDDTT}[3]{\frac{d^2{#1}}{d{#2}d{#3}}}

\newcommand{\TWO}[2]{\left(\setlength{\arraycolsep}{1pt}\begin{array}{cc}{#1}, & {#2}\end{array}\right)}

\newcommand{\THREE}[3]{\left(\setlength{\arraycolsep}{1pt}\begin{array}{ccc}{#1}, & {#2}, & {#3}\end{array}\right)}

\newcommand{\FOUR}[4]{\left(\setlength{\arraycolsep}{1pt}\begin{array}{cccc}{#1}, & {#2}, & {#3}, & {#4}\end{array}\right)}

\newcommand{\MTT}[4]{\left(\setlength{\arraycolsep}{1pt}\begin{array}{cc}#1 & #2 \\ #3 & #4\end{array}\right)}

\newcommand{\dist}{\text{dist}}

\newcommand{\fmin}[1]{\underset{#1}{\min}}

\newcommand{\argmin}[1]{\underset{#1}{\text{argmin}}}
\newcommand{\argminP}{\text{argmin}}

\newcommand{\ST}{\text{s.t.}}

\newcommand{\proofread}[1]{}
\newif\ifArxiv
\usepackage{xcolor}
\definecolor{Blue}{rgb}{0.2, 0.2, 0.8}
\definecolor{Black}{rgb}{0.0, 0.0, 0.0}

%% file: introduction.tex
\section{Introduction}
Robot design and control stand as two paramount areas of research. In robot design, engineers manipulate the hardware specifications of a robot to equip it with the desired motion capabilities, while in robot control, algorithms compute the requisite motions and control signals for accomplishing a wide array of tasks. Although both domains share the overarching objective of task fulfillment, they have traditionally been treated as separate stages in the blueprint of a robotic system. Robot designers often rely heavily on their experience and observations to fine-tune design parameters~\cite{causey1998gripper,burdick2003minimalist,10.1177/0278364916648388}, while robot controllers are typically crafted with the assumption of fixed robot designs~\cite{poulakakis2009spring,kajita2010biped}. Unfortunately, recent studies~\cite{ha2018computational,zhao2020robogrammar,kim2021mo,Xu-RSS-21} have revealed that neglecting the inherent interdependence between these two stages can inevitably result in sub-optimal designs.

\begin{figure}[t]
\centering
\includegraphics[width=.48\textwidth]{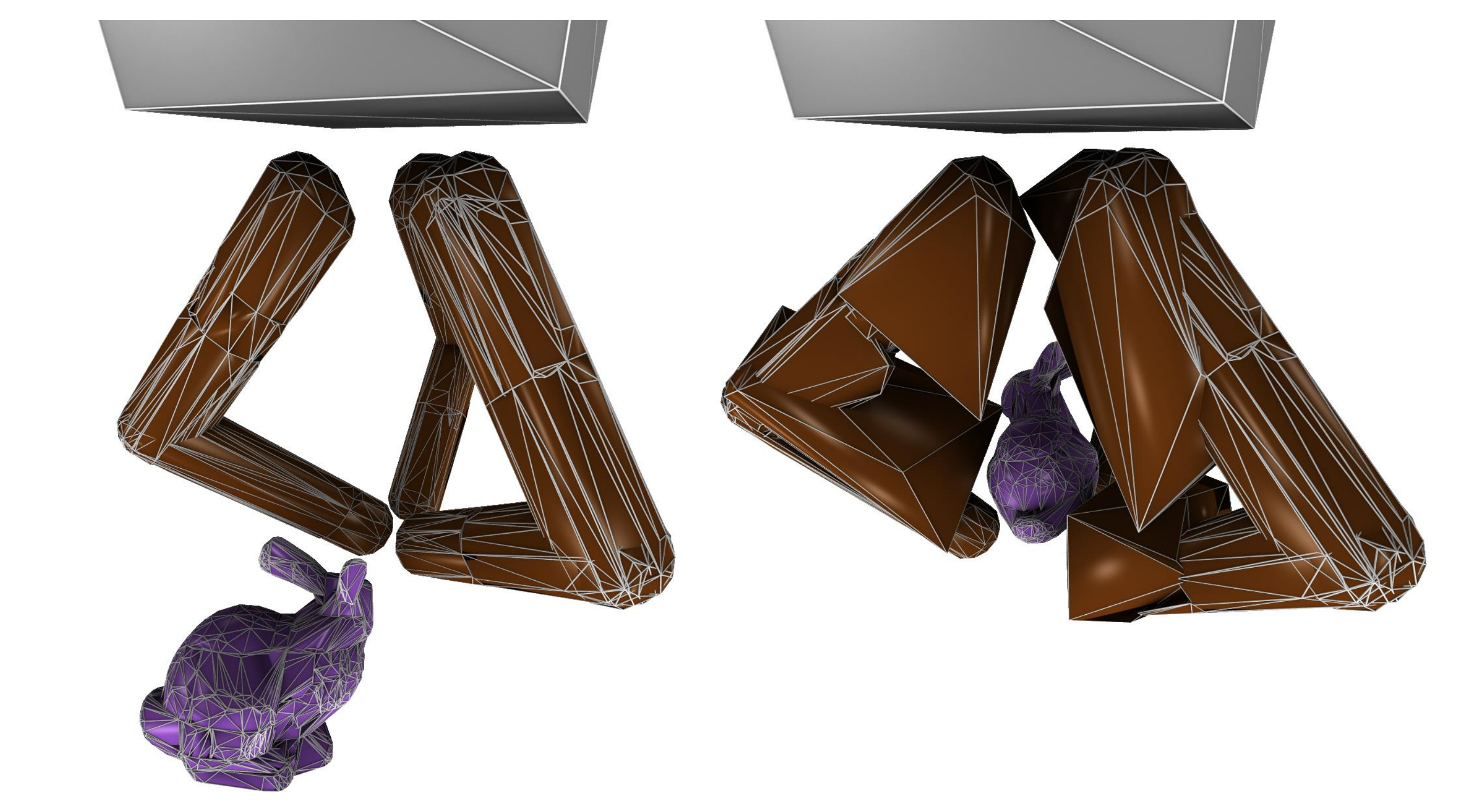}
\caption{Starting from an initial guess (left), we optimize the shape of gripper to firmly grasp the Stanford bunny (right).}
\label{fig:teaser}
\end{figure}
The divide between design and control has increasingly captured the attention of researchers. Various techniques have emerged to bridge this gap, and among these, a promising approach is robot co-design, where both the robot's design and controller parameters are jointly optimized. A significant challenge in such applications lies in the expansive search space, encompassing discrete decision variables for the robot's topology and continuous ones for its geometry and controller parameters. To derive meaningful robot designs, existing methods often resort to aggressive constraints on the robot's design space and employ exhaustive sampling-based optimization strategies, such as Bayesian optimization~\cite{kim2021mo,liao2019data} and Monte-Carlo tree search~\cite{zhao2020robogrammar}. The necessity of constraining the design space arises because these sampling-based optimizers grapple with the curse of dimensionality, leading to exponential complexity as the design space grows. Consequently, this restriction on the design space can, once again, result in sub-optimal designs.

Recent advancements in sensitivity analysis~\cite{ha2018computational} and robot simulation~\cite{Xu-RSS-21} have introduced a gradient-based approach as an alternative to traditional sampling-based methods. These algorithms use differentiable simulators to propagate gradient information throughout robot simulations, allowing this information to guide a variety of upstream applications, including controller optimization~\cite{hu2019chainqueen}, behavior cloning~\cite{chen2022imitation}, and robot system identification~\cite{song2020learning}. Unlike traditional methods, gradient-based algorithms are not constrained by the curse of dimensionality, enabling them to scale across large design spaces and optimize complex aspects of robot link configurations. Despite their promising results, current differentiable simulators are either incapable of modeling significant changes to geometric shapes,  restricting the changes to small localized neighborhoods, or not provably differentiable.

We present, SDRS, a guaranteed collision-free robot simulator that can differentiate through significant geometric and topologic shape changes to robot links. Our approach is founded on two cornerstones. First, we introduce a novel parameterization for the shape of robot links using the union of convex polyhedrons in the V-representation~\cite{deng2020cvxnet}. This expressive representation allows both geometric and topologic shape changes in relatively low-dimensional continuous parametric shapes. Second, we propose a penalty-based collision constraint between pairs of convex hulls based on the separating hyperplane theorem~\cite{honig2018trajectory}. We show that this formulation can be extended to model contact mechanics by treating the separating plane as a zero-mass auxiliary entity. Our formulation avoids the need for explicit contact point detection between convex hulls, thereby eliminating non-differentiable singularities. We show that our formulation is globally differentiable under sufficiently small timestep sizes, ensuring consistently well-defined gradient information. To underscore the practical utility of our simulator, we showcase its applications over a row of robot co-design scenarios as shown in~\prettyref{fig:teaser}.

%% file: related.tex
\section{Related Work}
We review related works on robot co-design, differentiable simulation, and contact mechanics.

\paragraph{Robot Co-design}
Joint optimization strategies for robot co-design have been developed to fine-tune both the physical and geometric parameters, as well as controller parameters during testing, all aimed at achieving optimal task performance. The automatic, joint optimization functionality complements the semi-automatic, human-assisted design pipeline~\cite{10.1145/2816795.2818137,10.1145/3453477,maloisel2023optimal} to form a complete robot design and manufacturing workflow. Much of the existing research, such as that by \cite{liao2019data,zhao2020robogrammar,kim2021mo,wang2022curriculum,fadini2022simulation}, has depended on gradient-free, sampling-based methods. A significant limitation of these approaches is their limited scalability. To reduce these costs, Bayesian optimization~\cite{liao2019data,kim2021mo} has been used to model the landscape of low-level objectives and to generate new samples that minimize fitting errors. Additionally, \citet{zhao2020robogrammar} capitalized on the structural insight that open-loop articulated robots have a tree-like topology, similar to the decision trees in MCTS, enabling quick assessment and pruning of sub-optimal design candidates. Despite these enhancements, sampling-based algorithms struggle to scale beyond a few tens of parameters.

On the other hand, gradient-based methods~\cite{ha2018computational, Xu-RSS-21, ma2021diffaqua, dinev2022versatile} have only recently come into prominence. Thanks to their efficient gradient evaluation algorithms, these methods incur lower iterative costs and are capable of scaling to high-dimensional decision spaces. Initial applications of these methods~\cite{ha2018computational} were confined to low-dimensional design spaces, including variables such as robot link lengths and actuator positions. More recent developments~\cite{Xu-RSS-21,ma2021diffaqua} have broadened the scope of these techniques, enabling optimization of more complex robot link shapes. However, these advanced methods still rely on mesh-based shape representations with the fixed topology, which limit changes to minor geometric adjustments.

\paragraph{Differentiable Simulation}
The realm of optimal control has witnessed substantial advancements through the adoption of differentiable simulators. Their application scope has rapidly broadened, encompassing control not only over low-dimensional rigid bodies~\cite{de2018end} and articulated bodies~\cite{Qiao2021Efficient} but also extending to the control of high-dimensional deformable objects, such as elastic objects~\cite{heiden2021disect,du2021_diffpd,stuyck2023diffxpbd,huang2024differentiable,gjoka2024soft,spielberg2019learning}, fluids~\cite{takahashi2021differentiable}, cloth~\cite{liang2019differentiable}, and general multi-physics systems~\cite{su2023generalized}. Furthermore, an increasing body of research has integrated differentiable simulators into downstream end-to-end learning algorithms to support various applications, including behavior cloning~\cite{chen2022imitation}, system identification~\cite{song2020learning,ma2022risp}, and controller optimization~\cite{hu2019chainqueen}. While these applications solely require a simulator to be differentiable with respect to state and control parameters, the field of co-design necessitates the calculation of gradients with respect to link shapes.

Addressing this need,~\citet{Xu-RSS-21} expanded upon prior work~\cite{geilinger2020add} by introducing additional design gradients using a cage-based robot shape deformation model. While promising results have been presented, their deformable model is based on a mesh of fixed topology, which inherently does not allow significantly shape changes such as modifications to the link topology. We are also aware of the differentiable Material Point Method (MPM)~\cite{hu2019chainqueen}, which was originally used to simulate soft robots, can be adapted to approximately simulate rigid bodies by using a large stiffness. Using particle-based representations, MPM can differentiate through large geometric and even topological changes. Unfortunately, particle-based representations are inherently incapable of representing smooth and flat surfaces unless a dense set of particles is used, making them rather inefficient for co-design applications. 

Notably, despite claims of proposing differentiable simulators in previous research, none of these simulators are rigorously differentiable concerning all input parameters. For instance, some formulations, like those presented by~\citet{de2018end} and~\citet{werling2021fast}, incorporate contact modeling using complementary constraints with non-smooth transitions between contact modes. Although complementary constraints can be smoothened as shown in~\cite{tassa2012synthesis,howelllecleach2022}, these methods often require maintaining discrete contact points for constraint instantiation, which can undergo non-smooth variations due to the change of robot link shapes.

Complementing the works on differentiable simulators, physics-constrained trajectory optimization~\cite{10.1145/3309486.3340246,10.1145/2185520.2185539} formulates the equations of motion as a set of (in)equality constraints, which are enforced while optimizing high-level objective functions. These methods are computationally more efficient since they avoid the need for explicit constraint satisfaction at every iteration and have been successfully applied to co-design problems, as demonstrated in~\cite{hazard2020automated,desai2018interactive}. However, these techniques rely on additional mechanisms to ensure constraint satisfaction, which are not guaranteed to succeed.

\paragraph{Contact Mechanics}
Collisions and contacts play a significant role in shaping various phenomena within robot simulations, yet they also constitute a primary source of discontinuity and non-smooth behavior. The efficacy of collision detection techniques is heavily reliant on the chosen geometric representation. Among the two most prevalent representations are volumetric distance field~\cite{guendelman2003nonconvex} and surface meshes~\cite{catto2005iterative}, both of which persist as fundamental components in state-of-the-art differentiable simulators~\cite{geilinger2020add,gradsim,Xu-RSS-21}. These techniques identify points of penetration and employ penalty forces to mitigate the depth of penetration. However, these discrete collision points can undergo non-smooth transitions as the shapes of robot links change. An alternative collision resolution approach revolves around complementary constraints~\cite{werling2021fast}, a method well-known for its non-smooth characteristics. Moreover, a common challenge associated with these collision techniques arises when dealing with geometrically thin objects. As they allow penetration between objects, simulators may inadvertently bypass crucial collisions by tunneling through thin objects.

\begin{table}[h]
\centering
\setlength{\tabcolsep}{3px}
\begin{tabular}{cccccc}
\toprule
Method & Geometry & Smooth & Diff. & Coll.-Free & Topo.-Change\\
\midrule
\cite{gradsim,Xu-RSS-21} & Mesh & \cmark & \xmark & \xmark & \xmark\\
\cite{huang2024differentiable} & Mesh & \cmark & \cmark & \cmark & \xmark\\
\cite{hu2019chainqueen} & Particle & \xmark & \cmark & \xmark & \cmark\\ 
Ours & Convex Hull & \cmark & \cmark & \cmark & \cmark\\
\bottomrule
\end{tabular}
\caption{\label{table:features}We compare different articulated body simulation algorithms in terms of several features: the geometric representation, ability to represent smooth and planar surfaces, provable differentiability, provable collision-free guarantee, and differentiability under topology changes (from left to right). Our method is the first with a complete feature set.}
\end{table}
In contrast, we are inspired by the interior point method~\cite{Harmon2009ACM,Li2020IPC}, wherein collisions are preemptively averted by introducing barrier functions to confine robots within collision-free spaces. However, these techniques use mesh-based representations with fixed topology for robot links, which offer a less compact representation of robot designs. Instead, we adopt a significantly more concise and expressive representation—namely, the union of convex polyhedrons in the V-representation~\cite{deng2020cvxnet}. This representation permits alterations to both the geometry and topology of robot link shapes using a relatively smaller set of decision variables. We are aware of a related recent work~\cite{tracy2023differentiable} that allows differentiable collision detection between convex hulls. However, it is still unclear how to incorporate such collision detector into the full pipeline of contact handling as required by dynamic simulations, which is the goal of our work. In summary, by comparing the features of different algorithms in~\prettyref{table:features}, we highlight that our method is the first algorithm with a complete feature set.

%% file: problem.tex
\section{Problem Definition}
In this section, we introduce the general framework of robot co-design using the novel design space proposed in~\citet{deng2020cvxnet} as illustrated in~\prettyref{fig:convex}. Then in~\prettyref{sec:simulation}, we propose our robot simulator tailored to this design space. 

\subsection{\label{sec:codesign}Robot Co-design}
We denote the robot design as parameterized by a set of continuous variables represented as $d$. Furthermore, we assume the existence of a convex design space denoted as $\mathcal{D}$, where only values of $d\in\mathcal{D}$ correspond to valid robot designs. We maintain the assumption that the robot's number of degrees of freedom, i.e. the dimension of configuration space, remains invariant with respect to the robot's design. At each discrete time step $t$, the kinematic state of the robot is represented as the vector $\theta^t$. Additionally, the robot is under the control of a parameterized signal $u^t$ at time step $t$. With these notations in place, our simulator is defined as the following function $\theta^{t+1}=f(\theta^{t},\theta^{t-1},u^t,d)$, where the first two parameters represent the robot's kinematic state at two distinct time instances, combining to describe its dynamic state. Note that we assume $\theta^{t+1}$ is a function of form $\theta^{t+1}(u^t,u^{t-1},\cdots,d)$, but we use the abbreviation for brevity. The overarching objective of robot co-design is to optimize both controller and design parameters through a joint optimization of a loss function $L(\theta^H)$ over a trajectory of $H$ time steps: $\argminP_{d\in\mathcal{D},c\in\mathcal{C}}L(\theta^H)$, where we assume that our control policy is represented as a differentiable function $u^t=\pi(\theta^t,\theta^{t-1},t,c)$, which is parameterized by another set of continuous parameters $c$, and it operates within a convex control space denoted as $c\in\mathcal{C}$. The co-design optimization process is guided by a terminal-state differentiable cost function that estimates the quality of task completion. To solve this co-optimization, we employ the standard projected gradient method~\cite{iusem2003convergence}. In the next section, we introduce a design space that strikes a balance between compactness and expressiveness, enabling us to make modifications to both the geometry and topology of robot links. Building upon this versatile design space, our subsequent goal is to craft a contact-aware robot simulator function, $f$, which demonstrates provable differentiability concerning all of its parameters. We refer to this as the SDRS simulator.

\input{shape.tex}

%% file: shape.tex
\subsection{\label{sec:convex}Design Space Parameterization}
The design of robots integrates several aspects, such as joint parameters, joint limits, and motor thrusts. Previous studies such as~\cite{Xu-RSS-21} have demonstrated that articulated body simulators are differentiable with respect to these features. We primarily focus on the parametrization of robot link shapes. Drawing on recent advances from~\citet{deng2020cvxnet}, we model the $i$th robot link as a combination of $N$ smoothed convex polyhedrons, each defined within a reference frame and denoted as $H_{i1},\cdots,H_{iN}$. Each polyhedron is characterized by a V-representation, defined by the convex combination of a set of $M$ vertices: $H_{ij}\triangleq\text{CH}(x_{ij}^1,\cdots,x_{ij}^M)$, where $\text{CH}$ denotes the convex hull operation. This parametric method not only retains compactness but also offers extensive expressiveness, effectively capturing a wide range of shapes with diverse geometry and topology. Further, the SDRS allows for the adjustment of the attachment point between the $i$th joint and its parent joint, denoted by $\lambda(i)$, with the attachment point represented as $x_i^{\lambda(i)}$. This results in the design space of dimension $|d|=3\#(NM+1)$ with $\#$ being the number of links.

\begin{figure}[h]
\centering
\includegraphics[width=.99\linewidth]{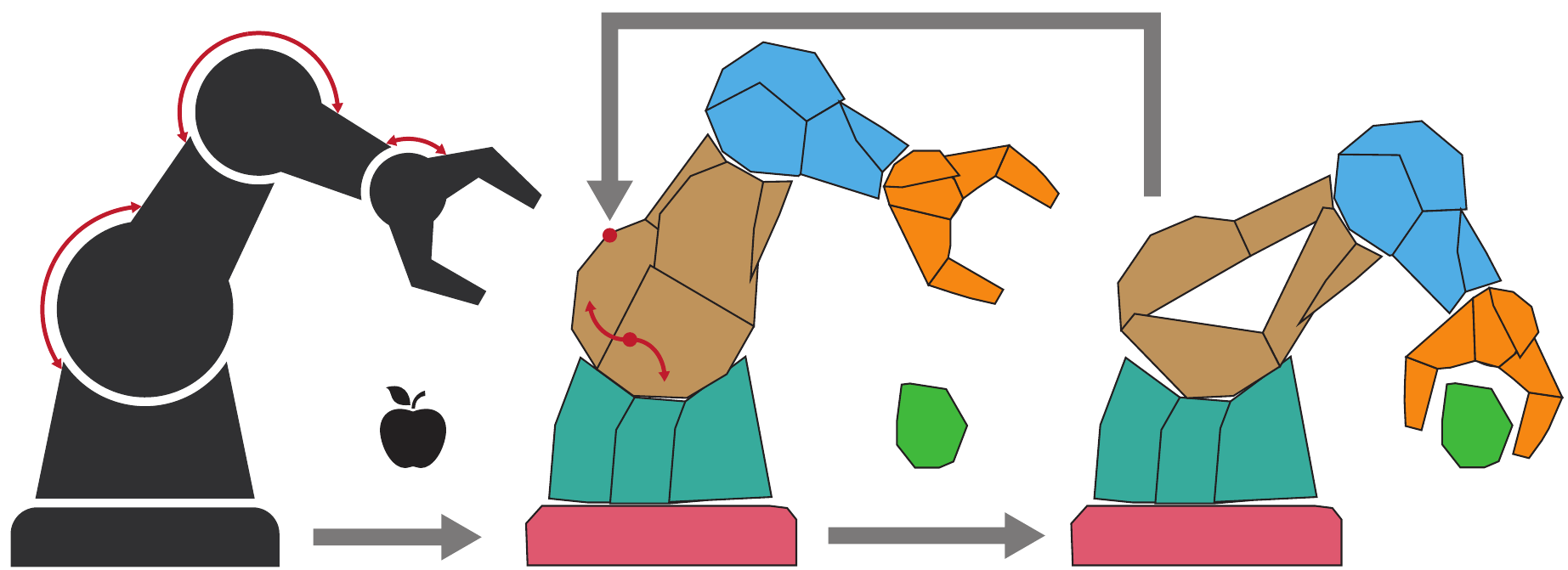}
\put(-240,55){\tiny{$\theta$}}
\put(-212,85){\tiny{$\theta$}}
\put(-185,70){\tiny{$\theta$}}
\put(-148,40){\tiny{$x_{ijk}$}}
\put(-200,8){\tiny{V-HACD}}
\put(-110,8){\tiny{Forward}}
\put(-90,78){\tiny{Backward}}
\caption{We illustrate the co-design problem of a robot arm trying to reach a target apple, with the robot's configuration space being the 3 joint angles (red). We first represent each robot link shape as a set of (possibly overlapping) convex polyhedrons (middle), e.g. using V-HACD, where we use different colors for different rigid bodies. A connectivity constraint denoted as $x_{ijk}$ is introduced for each pair of overlapping convex polyhedrons (red). Our SDRS is a differentiable robot simulator tailored to the convex-polyhedron-based representation, where we can differentiate with respect to the polyhedron vertices. Such representation allows both geometry and topology changes by continuously modifying the vertices, e.g. to drill a hole in the middle of the brown link (right).}
\label{fig:convex}
\end{figure}
\setlength{\columnsep}{10pt}
\begin{wrapfigure}{r}{0.24\textwidth}
\centering
\includegraphics
[width=0.2\textwidth,
trim=0 40px 0 0,clip]
{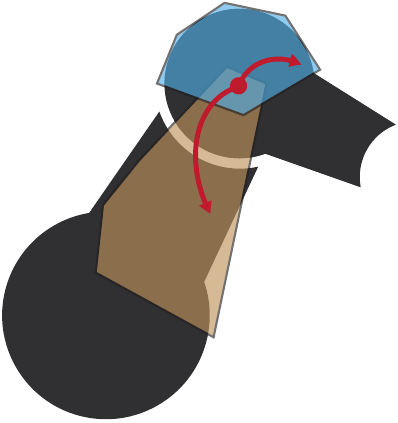}
\put(-52,70){$x_i^{\lambda(i)}$}
\caption{ The constraint to ensure that the attachment point $x_i^{\lambda(i)}$ overlaps both links.}
\label{fig:attachment}
\end{wrapfigure}
As a potential problem, permitting the vertices to move freely may result in disconnected robot links. To address this issue, we introduce connectivity constraints. Specifically, for a pair of polyhedrons, $H_{ij}$ and $H_{ik}$, connectivity is ensured by introducing an auxiliary vertex $x_{ijk}$ and requiring that $x_{ijk}$ is a common vertex shared by the two convex hulls. All imposed constraints are convex and consistent with our assumptions regarding $\mathcal{D}$. Similarly, we ensure that the $i$th link is connected to its parent by requiring that $x_i^{\lambda(i)}$ is located within some convex hull encompassing both the $i$th and $\lambda(i)$th links as shown in~\prettyref{fig:attachment}. Practically, users may supply mesh-based link representations, which we convert to our parametric form using approximate convex decomposition~\cite{mamou2016volumetric}. We then insert connectivity constraints between each pair of initially overlapping convex hulls, as well as between the convex hulls of the $i$th and $\lambda(i)$th links containing the attachment point.

%% file: method.tex
\section{\label{sec:simulation}SDRS Formulation}
In this section, we develop a collision-free robot simulator under minimal coordinates, which is tailored to the convex-polyhedron-based representation. We first delve into position-level articulated body dynamics in~\prettyref{sec:pbad}, which is formulated as an unconstrained optimization with twice-differentiable objectives. In~\prettyref{sec:contact}, we then present a novel formulation of differentiable contact mechanics between convex polyhedrons, and we extend it to handle frictional damping in~\prettyref{sec:friction}. Finally, we discuss the simulation algorithms and the derivative evaluation scheme in~\prettyref{sec:forwardbackward}

\input{pbad.tex}
\input{contact.tex}
\input{friction.tex}

\subsection{\label{sec:forwardbackward}Forward Simulation \& Backward Derivative Evaluation}
To tackle the optimization problem represented by~\prettyref{eq:SDRS} for forward simulation, we employ a conventional Newton's method paired with a line-search globalization technique. This approach guarantees that the solution consistently reduces the objective function, thereby ensuring collision-free robot states throughout the simulation. For each timestep in the simulation, we initiate the optimization process by setting $\theta^{t+1}\gets\theta^t$. That said, we emphasize that our method can only guarantee that the simulator is collision-free at discrete time instances $\theta^t, \theta^{t+1}, \cdots$, but we cannot ensure that the robot motion is collision-free during the motion between $\theta^t$, and $\theta^{t+1}$, which requires a continuous collision checker coupled with the line-search scheme. Unfortunately, prior works~\cite{li2021codimensional,ferguson2021intersection} on Continuous Collision Detection (CCD) can only be adapted to linear or rigid motions, which are incapable of handling more general articulated motions owing to our nonlinear forward kinematics function $T_i(d,\theta^t)$. In practice, we have never observed any tunneling artifact due to our small timestep. After the forward pass, we evaluate the derivatives using the adjoint method. In our~\prettyref{sec:DifferentiabilityProof}, we show that the entire function $\mathcal{O}$ is globally twice-differentiable. As a result, it has well-defined derivatives by the implicit function theorem, under sufficiently small timestep sizes.

%% file: pbad.tex
\subsection{\label{sec:pbad}Position-level Articulated Dynamics}
For now, we omit the collisions and contacts and elucidate the dynamics of the articulated body using a position-level formulation as outlined in~\cite{pan2018time}, which is a special form of physics simulation under generalized coordinates~\cite{martin2011example}. We begin by formulating the forward kinematic function of the $i$th link. The transformation from reference to world space for this link can be defined recursively as:
\begin{align*}
T_i(d,\theta)\triangleq  T_{\lambda(i)}(d,\theta)
\MTT{R_i^d(d)}{t_i^d(d)}{0}{1}
\MTT{R_i(\theta)}{t_i(\theta)}{0}{1},
\end{align*}
where $T_{\lambda(i)}(d,\theta)$ is the transformation of the parent link, $R_i^d(d)$ and $t_i^d(d)$ define the design-dependent local transformation of attachment points of the $i$th link in $\lambda(i)$th reference frame, and finally $R_i(\theta)$ and $t_i(\theta)$ represent the joint transformation of the $i$th link. For the sake of brevity, we have omitted the timestep index and we denote $x$ as the $4$D homogeneous coordinates so that $T_i(d,\theta)x$ is the transformed coordinates in world space. Position-level dynamics hinges on the measurement of velocity in Euclidean space rather than configuration space. Specifically, we compute the acceleration of a local point $x$ on $i$th link as follows:
\begin{align}
\ddot{x}\triangleq\frac{1}{\Delta t^2}\left[T_i(d,\theta^{t+1})-2T_i(d,\theta^{t})+T_i(d,\theta^{t-1})\right]x,
\end{align}
where $\Delta t$ represents the timestep size. Note that we have employed a first-order finite-difference approximation for acceleration, which suffices for our specific application. However, higher-order approximations are also available for those seeking enhanced accuracy.

If we define $U(\bullet)$ as the world-space conservative potential energy density, we can formulate position-level dynamics as the following unconstrained minimization:
\begin{equation*}
\ResizedEq{&\theta^{t+1}=f(\theta^{t},\theta^{t-1},u^t,d)\triangleq\argmin{\theta^{t+1}}\;\sum_{ij}I_{ij}(d,\theta^{t+1},\theta^t,\theta^{t-1},u^t)\\
&I_{ij}\triangleq\int_{x\in H_{ij}(d)}\rho_{ij}(x)\left[\frac{\Delta t^2}{2}\|\ddot{x}\|^2+U(T_i(d,\theta^{t+1})x,u^t)\right]dx,}
\end{equation*}
where $\rho(x)$ represents the mass density of the robot link. As shown in~\cite{pan2018time}, the above integral can be computed for a fixed robot design by pre-computing inertial-like tensors. As a remarkable feature of~\cite{pan2018time}, the bilinear term describing the centrifugal and Coriolis forces can be absorbed into the above integral, inducing a concise expression of time integration and differentiation. However, when dealing with a changing integration domain, $H_{ij}(d)$, this pre-computation is no longer feasible. Instead, we make the assumption that the mass of the robot link is evenly distributed among the convex hull vertices, i.e., $\rho_{ij}(x) = \sum_{m=1}^M \rho\delta(x-x_{ij}^m)$, where $\delta$ represents the Dirac delta function. This allows us to simplify the integral to the following summation:
\begin{align*}
I_{ij}\triangleq\rho\sum_{m=1}^M\left[\frac{\Delta t^2}{2}\|\ddot{x}_{ij}^m\|^2+U(T_i(d,\theta^{t+1})x_{ij}^m,u^t)\right].
\end{align*}
The position-level dynamics framework is highly versatile and can accommodate various external and internal force models by adding additional energy terms, $U(\bullet)$, where we further require $U(\bullet)$ to be a smooth function. In particular, we are interested in applying control signals using a stable-PD controller~\cite{5719567}. In this context, the control signal is defined by a desired target robot position $\theta_\star^{t+1}$ and velocity $\dot{\theta}_\star^{t+1}$, and we define it as: $u^t=\TWO{\theta_\star^{t+1}}{\dot{\theta}_\star^{t+1}}$. The stable-PD controller is specified by adding the following controller-energy term:
\small
\begin{equation*}
\ResizedEq{U_\text{pd}(\theta^{t+1},\theta^t,u^t)\triangleq k_p\|\theta_\star^{t+1}-\theta^{t+1}\|^2+
k_d\|\dot{\theta}_\star^{t+1}-(\theta^{t+1}-\theta^t)/\Delta t\|^2,}
\end{equation*}
\normalsize
where $k_p$ and $k_d$ are the position and velocity gains, respectively.

%% file: contact.tex
\subsection{\label{sec:contact}Contact Mechanics for Convex Polyhedrons}
Moving forward, we address the intricate aspects of (self-)collisions and contacts within our approach. Note that rigid body simulations using convex polyhedral shapes is a common practice in real-time simulators~\cite{erez2015simulation}, but these methods need to explicitly detect and manage contact points, rendering them less amenable to differentiation. Our methodology draws inspiration from the interior point method~\cite{Harmon2009ACM,Li2020IPC}. These techniques introduce a barrier energy function denoted as $U_c(\bullet)$ into the numerical optimizer, a function defined exclusively within the collision-free domain. Moreover, these barrier energies can be intentionally designed to possess sufficient smoothness, rendering them amenable to differentiation. However, a challenge arises when applying this approach to our representation of robot shapes. Indeed, contacts can occur not only between vertices but also faces and edges. With the shape changes of convex polyhedrons, these face-wise or edge-wise contacts can dynamically appear or vanish, leading to non-differentiable singular configurations.

We propose an alternative barrier energy formulation between a pair of convex polyhedrons, $H_{ij}$ and $H_{i'j'}$, that is inspired by separating hyperplane theorem, which was previous used to handle contacts for UAV swarms in~\cite{honig2018trajectory}. This theorem claims that, when two convex shapes exhibit no penetration, a separating plane, denoted as $p\triangleq\TWO{n}{o}\in\mathbb{R}^4$, exists such that $H_{ij}$ and $H_{i'j'}$ reside on opposite sides of $p$~\cite{borwein2006convex}, which can be formally modeled through the following set of constraints:
\begin{align}
\label{eq:Cons}
\|n\|=1\land\forall m=1,\cdots,M:
\begin{cases}
p^TT_i(d,\theta)x_{ij}^m\leq0\\
p^TT_{i'}(d,\theta)x_{i'j'}^m\geq0\\
\end{cases},
\end{align}
where $n\in\mathbb{R}^3$ is the normal of plane and $o\in\mathbb{R}$ is the offset. Notably, the first constraint ensures that the separating plane is well-defined with a unit-normal. Regretfully, the unit-normal constraint is non-convex and can incur significant difficulty in numerically satisfying~\prettyref{eq:Cons}. To resolve this problem,~\citet{honig2018trajectory} replaced the unit-normal constraint with the convex constraint, $\|n\|\leq1$, which is widely used in the Support Vector Machine (SVM)~\cite{suthaharan2016support} to equivalently ensure the separation of convex sets~\cite{chen2005tutorial}. 

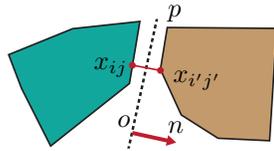
\begin{wrapfigure}{r}{0.24\textwidth}
\centering
\includegraphics[width=0.2\textwidth]{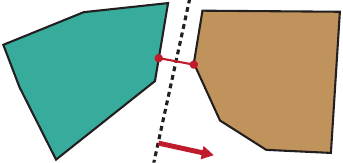}
\put(-70,30){$x_{ij}$}
\put(-40,25){$x_{i'j'}$}
\put(-42,50){$p$}
\put(-42,6){$n$}
\put(-61,8){$o$}
\caption{The separating hyperplane (dashed) between a pair of convex polyhedrons with normal $n$ (red). The closest distance (red) between them is realized by $x_{ij}$ and $x_{i'j'}$.}
\label{fig:contact}
\end{wrapfigure}
Building upon the principles of the interior-point method, we formulate the aforementioned constraints as a smooth barrier function. \citet{Li2020IPC} proposed the clamped log-barrier function that is locally supported in $(0,s)$. Instead, we use the function of $P_s(x)=\max(0,(x-s)^4/x^5)$ with the similar property but is convenient for our analysis, and we derive the following contact potential for the pair of convex polyhedrons:
\begin{align*}
&U_{c\star}^{ij,i'j'}(d,\theta,p)
\triangleq P_s(1-\|n\|)+\\
&\sum_{m=1}^M \left[P_s(-p^TT_i(d,\theta)x_{ij}^m)+
P_s(p^TT_{i'}(d,\theta)x_{i'j'}^m])\right].
\end{align*}
Our contact potential effectively circumvents the need to identify face-wise or edge-wise collisions by consolidating energies directly into the vertices of the convex hulls, which are directly mapped to our parameter space. Note that the definition of $U_{c\star}^{ij,i'j'}$ remains under-determined, as it relies on the additional separating plane $p$. To resolve this issue, we consider the separating plane as an auxiliary physical entity characterized by zero mass as shown in~\prettyref{fig:contact}. Consequently, its presence exerts a negligible influence on other physical entities, preserving the behavior of the physical world. However, since $p$ possesses zero mass, it must experience zero force and torque, as any non-zero force or torque would result in infinite acceleration. We observe that the only potential energy associated with $p$ is the contact energy, so that $p$ must constitute a local minimum of $U_{c\star}^{ij,i'j'}$, yielding a well-determined barrier potential: $U_c^{ij,i'j'}(d,\theta)\triangleq\min_p\;U_{c\star}^{ij,i'j'}(d,\theta,p)$ and we define the optimal separating plane as $p_\star$. In our~\prettyref{sec:ContactProof}, we show that $U_c^{ij,i'j'}$ is a well-defined, twice-differentiable contact potential that strictly prevents any intersections between the pair of polyhedron. 

Our contact potential is amenable to efficient computations. As a first step, we can construct a Bounding Volume Hierarchy (BVH) for all the convex polyhedrons and $U_c^{ij,i'j'}$ only takes non-zero values when the distance between $H_{ij}$ and $H_{i'j'}$ is less than $2s/(1-s)$ (see~\prettyref{sec:ContactProof}). By bulging each bounding box by this distance, we can accelerate energy evaluations by pruning non-overlapping pairs. When two bounding boxes overlap, we then use the GJK algorithm~\cite{lindemann2009gilbert} to compute the exact distance between the two polyhedrons, along with the pair of points $x_{ij}$ and $x_{i'j'}$ realizing the smallest distance. If the energy can be non-zero, we initialize the separating plane $p$ to be the middle orthogonal plane between $x_{ij}$ and $x_{i'j'}$. Next, we use the Newton's method to find the globally optimal $p$ to compute $U_c^{ij,i'j'}$. Finally, the derivatives of our formulated contact potential can be derived using the (high-order) implicit function theorem.

%% file: friction.tex
\subsection{\label{sec:friction}Frictional Contact Mechanics}
\begin{wrapfigure}{r}{0.24\textwidth}
\centering
\includegraphics[width=0.2\textwidth]{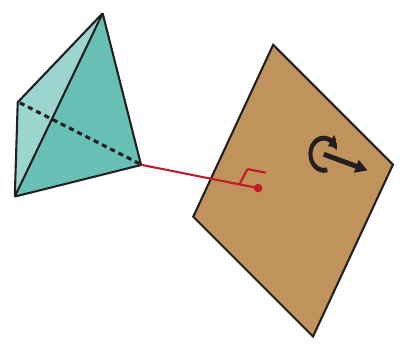}
\put(-80,35){$T_ix_{ij}^m$}
\put(-50,30){$BT_ix_{ij}^m$}
\put(-30,55){$\omega$}
\put(-16,37){$u$}
\caption{We assume the separating plane is moving at an in-plane angular speed of $\omega$ and linear speed of $u$. Each vertex $T_i(d,\theta^t)x_{ij}^m$ is projected onto the plane as $B^{ij,i'j'}T_i(d,\theta^t)x_{ij}^m$, to compute the relative velocity.}
\label{fig:friction}
\end{wrapfigure}
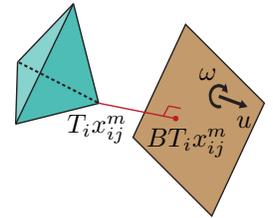
Expanding our contact model to incorporate frictional forces presents a unique challenge that necessitates a novel formulation. Prior simulators have traditionally employed point-wise Coulomb's frictional models, whereas our contact model operates between convex hulls and does not explicitly identify contact points. To bridge this gap, we once again turn to the separating plane, denoted as $p_\star^t=\TWO{n_\star^t}{n_{0\star}^t}$, where the superscript indicates the separating plane computed at the $t$th timestep. 

Modeled as a physical entity with zero-mass, we make the assumption that the separating plane is undergoing tangential motion with a linear speed of $u\in R^2$ and an angular speed of $\omega\in R$. Instead of introducing frictional forces between the two convex hulls, we propose a novel approach illustrated in~\prettyref{fig:friction}: introducing frictional forces between each vertex of the convex hull and the separating plane. For instance, considering $x_{ij}^m$ as an example, the normal force exerted on it at the $t$th timestep is defined as:
\begin{align*}
f_{ij,\perp}^m(d,\theta^t)\triangleq
-\FPP{U_c^{ij,i'j'}}{[T_i(d,\theta^t)x_{ij}^m]_3}=P_s'n_\star^t,
\end{align*}
where we use $[\bullet]_3$ to extract the first 3 rows of homogeneous coordinates. A similar formula is applicable to the force on the other polyhedron, $f_{i'j',\perp}^m$. We introduce the notation for the tangent space basis of the separating plane: $B^{ij,i'j'}(d,\theta^t)\in\mathbb{R}^{3\times2}$. Subsequently, we approximate the tangential velocity of $x_{ij}^m$ using a first-order finite difference as follows:
\begin{align*}
\dot{x}_{ij,\parallel}^m\triangleq 
\frac{1}{\Delta t}[B^{ij,i'j'}]^T
[T_i(d,\theta^{t+1})x_{ij}^m-T_i(d,\theta^{t})x_{ij}^m]_3.
\end{align*}
Leveraging the separating plane, we can calculate the tangential velocity of the plane at the projected point of $x_{ij}^m$. This is defined as:
\begin{align*}
\dot{p}_{ij,\parallel}^m(u,\omega)\triangleq \omega [B^{ij,i'j'}]^T[n_\star^t\times[T_i(d,\theta^{t})x_{ij}^m]_3]+u.
\end{align*}
The first and second terms above arise from tangential rotation and translation, respectively. Armed with this information, we can employ Coulomb's frictional law at the point-wise level and derive an equivalent expression for the frictional force acting on $x_{ij}^m$ as follows:
\begin{align*}
f_{ij,\parallel}^m\triangleq-\mu\|f_{ij,\perp}^m(d,\theta^t)\|B^{ij,i'j'}\frac{\dot{x}_{ij,\parallel}^m-\dot{p}_{ij,\parallel}^m}{\sqrt{\left\|\dot{x}_{ij,\parallel}^m-\dot{p}_{ij,\parallel}^m\right\|^2+\epsilon}},
\end{align*}
with $\mu$ being the frictional coefficient. $f_{ij,\parallel}^m$ can be equivalently expressed in the following conservative form:
\begin{equation*}
\ResizedEq{&D_{ij}^m(d,\theta^{t+1},\theta^t,u,\omega)\triangleq
\mu\Delta t\mathcal{A}_\epsilon(f_{ij,\perp}^m(d,\theta^t))
\sqrt{\left\|\dot{x}_{ij,\parallel}^m-\dot{p}_{ij,\parallel}^m\right\|^2+\epsilon},}
\end{equation*}
where we define $\mathcal{A}_\epsilon(x)\triangleq\sqrt{\|x\|^2+\epsilon}-\sqrt{\epsilon}$ and introduce a small constant $\epsilon$ to enforce smoothness as proposed in~\cite{tassa2012synthesis}, which suffices for our co-design applications. Smoothing function for more accurate static-sliding mode switching~\cite{Li2020IPC} can also be used to replace the function $\sqrt{\|\bullet\|^2+\epsilon}$ in our formulation, which follows an identical differentiability analysis as in our~\prettyref{sec:FrictionProof}. We denote $U_{f\star}^{ij,i'j'}\triangleq\sum_m[D_{ij}^m+D_{i'j'}^m]$ as the frictional-damping energy, which allows us to integrate the frictional force seamlessly into the unconstrained optimization framework. Notably, we rely on the magnitude of the normal force and the normal vector of the separating plane from the $t$th timestep to predict the frictional force at the $t+1$th timestep as proposed in~\cite{Li2020IPC}. Although a more advanced approach in~\cite{kaufman2008staggered} has been proposed to iteratively estimate the normal and frictional forces at the $t+1$th timestep, the convergence of such iteration is not well-studied, complicating our subsequent differentiability analysis. Our evaluation shows that using the separating plane from the $t$th timestep is sufficiently accurate when operating with a small timestep size $\Delta t$.

In a similar manner to the normal contact forces, our frictional-damping energy remains under-determined due to the involvement of the separating plane's velocity parameters, $u$ and $\omega$. To resolve these parameters, we invoke the Maximal Dissipation Principle (MDP), which asserts that an appropriate frictional force should maximize the dissipation of tangential velocity. Since our frictional-damping energy quantifies the magnitude of tangential relative velocity, the MDP immediately implies that $u$ and $\omega$ correspond to local minima of the frictional-damping energy. In light of these considerations, we define the following well-determined frictional-damping energy for a pair of convex hulls, $H_{ij}$ and $H_{i'j'}$:
\begin{align*}
U_f^{ij,i'j'}(d,\theta^{t+1},\theta^t)\triangleq\fmin{u,\omega}\;U_{f\star}^{ij,i'j'}(d,\theta^{t+1},\theta^t,u,\omega).
\end{align*}
It can be shown that our frictional potential is well-defined, twice-differentiable, and imposes opposing forces and torques on the two convex polyhedrons (see~\prettyref{sec:FrictionProof}). With these elements in place, we can now present the comprehensive formulation of the SDRS simulation function $f$. This function is structured as an unconstrained optimization, and it is expressed as follows:
\begin{equation}
\label{eq:SDRS}
\ResizedEq{&\theta^{t+1}\triangleq\argmin{\theta^{t+1}}\;\mathcal{O}(d,\theta^{t+1},\theta^t,\theta^{t-1},u^t)\triangleq\sum_{ij}I_{ij}(d,\theta^{t+1},\theta^t,\theta^{t-1})+\\
&U_\text{PD}(\theta^{t+1},\theta^t,u^t)+\mu
\sum_{ij,i'j'}^{i\neq i'}\left[U_c^{ij,i'j'}(d,\theta^{t+1})+U_f^{ij,i'j'}(d,\theta^{t+1},\theta^{t})\right].}
\end{equation}
Notably, we require $i \neq i'$, i.e., we exclusively examine collisions between convex hulls originating from distinct robot links. This allows for each robot link to be represented by the (potentially overlapping) combination of convex hulls.

%% file: evaluation.tex
\section{Evaluation}
\begin{wraptable}{r}{0.24\textwidth}
\centering
\begin{tabular}{cc}
\toprule
Substep & Percentage  \\
\midrule
$U_c^{ij,i'j'}$ & 30.51\%  \\
$U_f^{ij,i'j'}$ & 36.89\% \\
Backward & 30.85\% \\
Other Steps & 1.75\% \\
\bottomrule
\end{tabular}
\caption{Time breakdown per-iteration on average.}
\label{table:performance}
\end{wraptable}
We have implemented our method using C++ with Python interface, where we use OpenMP to parallelize the update of bounding volume hierarchy and evaluation of the contact potential terms $U_{c,f}^{ij,i'j'}$. During the evaluation of each term, we use Newton's method to solve the internal small-scale optimization with respect to $p$ and $\TWO{u}{\omega}$ for $U_c^{ij,i'j'}$ and $U_f^{ij,i'j'}$, respectively. These internal optimizations are solved until the norm of the gradient ($\|\nabla_pU_{c\star}^{ij,i'j'}\|$ or $\|\nabla_{u,\omega}U_{f\star}^{ij,i'j'}\|$) is less than $10^{-8}$ before the energy values and derivatives can be evaluated using the implicit function theorem. We further warm-start these internal optimizations using the solution from the previous iteration, as long as these initial guesses do not lead to undefined values of $U_c^{ij,i'j'}$, in which case we recompute the distance between convex hulls using GJK and initialize the separating plane $p$ from the middle surface. For the SDRS optimization (\prettyref{eq:SDRS}), we use explicit eigen decomposition to factorize the Hessian matrix of $\mathcal{O}$ and we adjust the eigenvalues to be larger than $10^{-3}$ to ensure descendant directions are computed. Thanks to the use of reduced coordinates, the $|\theta|\times|\theta|$ Hessian matrix is rather small, making direct eigen decomposition possible. We terminate the optimization when $\|\nabla\mathcal{O}\|_\infty<10^{-4}$. For solving the co-design problem defined in~\prettyref{sec:codesign}, we use the adaptive moment estimation (Adam) method~\cite{kingma2014adam}. We incorporate our simulator as an operator in the PyTorch framework~\cite{paszke2017automatic} and we use a highly optimized routine for evaluating the derivatives of our simulation function using the adjoint method. During each co-design optimization, we first compute the derivative $\nabla_{d,c}L(\theta^H)$ and then update $\TWO{d}{c}$ via:
\begin{align*}
\argmin{d+\Delta d\in\mathcal{D},c+\Delta c\in\mathcal{C}}\;&L(\theta^H)+\Delta d^T\nabla_dL(\theta^H)+\Delta c^T\nabla_cL(\theta^H)\\
\ST\quad\quad\quad&\|\Delta d\|\leq\bar{\alpha}_d\land\|\Delta c\|\leq\bar{\alpha}_c,
\end{align*}
where $\bar{\alpha_d}$ and $\bar{\alpha_c}$ are hyper-parameters defining the learning rate for robot design and controller. The above optimization is convex and we solve it efficiently using CVX~\cite{grant2009cvx}. All the experiments are performed on a desktop machine with Intel Core i9-13900K CPU and 128 Gb memory. 

\begin{figure}[ht]
\centering
\setlength{\tabcolsep}{1px}
\begin{tabular}{cc}
\includegraphics[trim=5cm 5cm 5cm 5cm,clip,width=.22\textwidth]{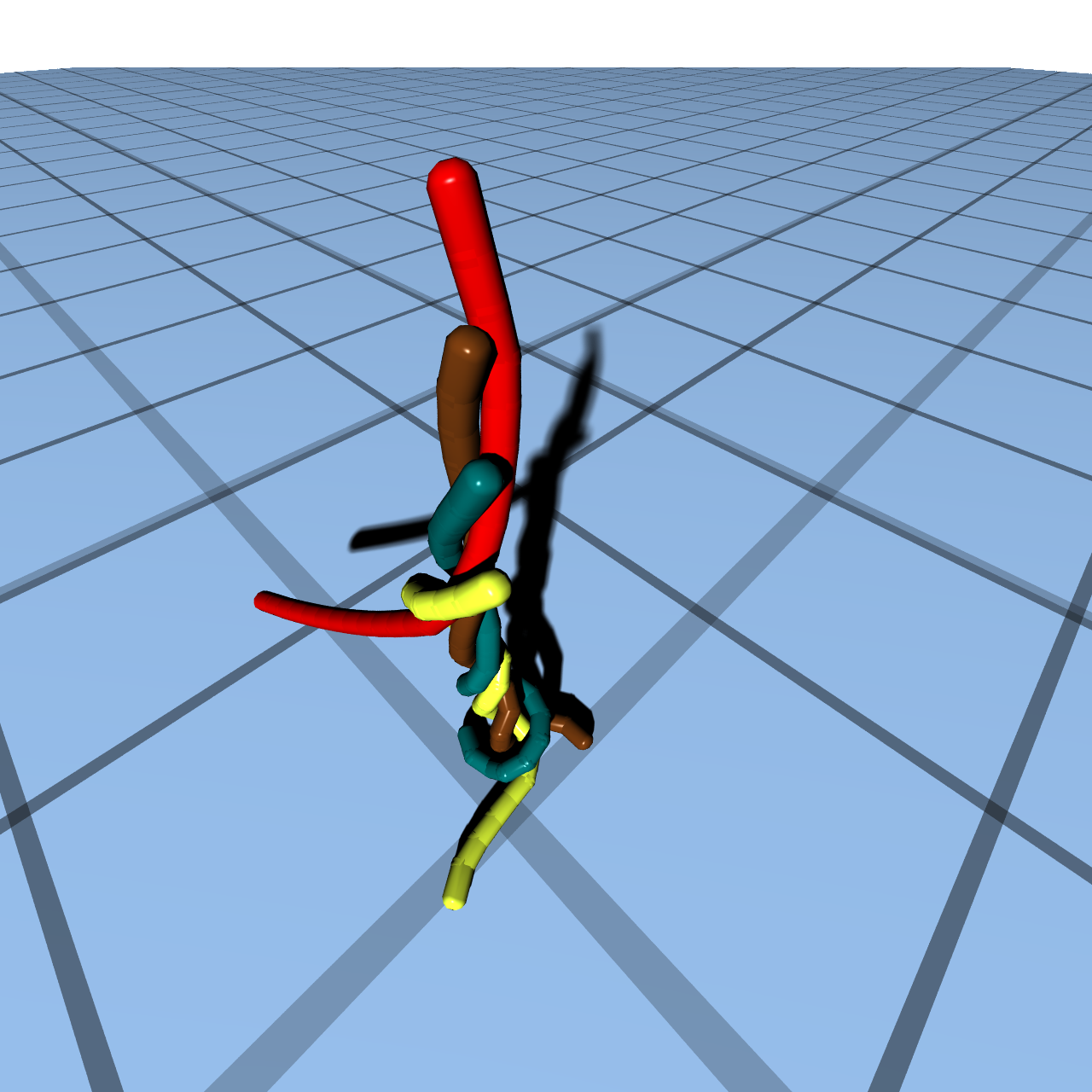} &
\includegraphics[trim=5cm 5cm 5cm 5cm,clip,width=.22\textwidth]{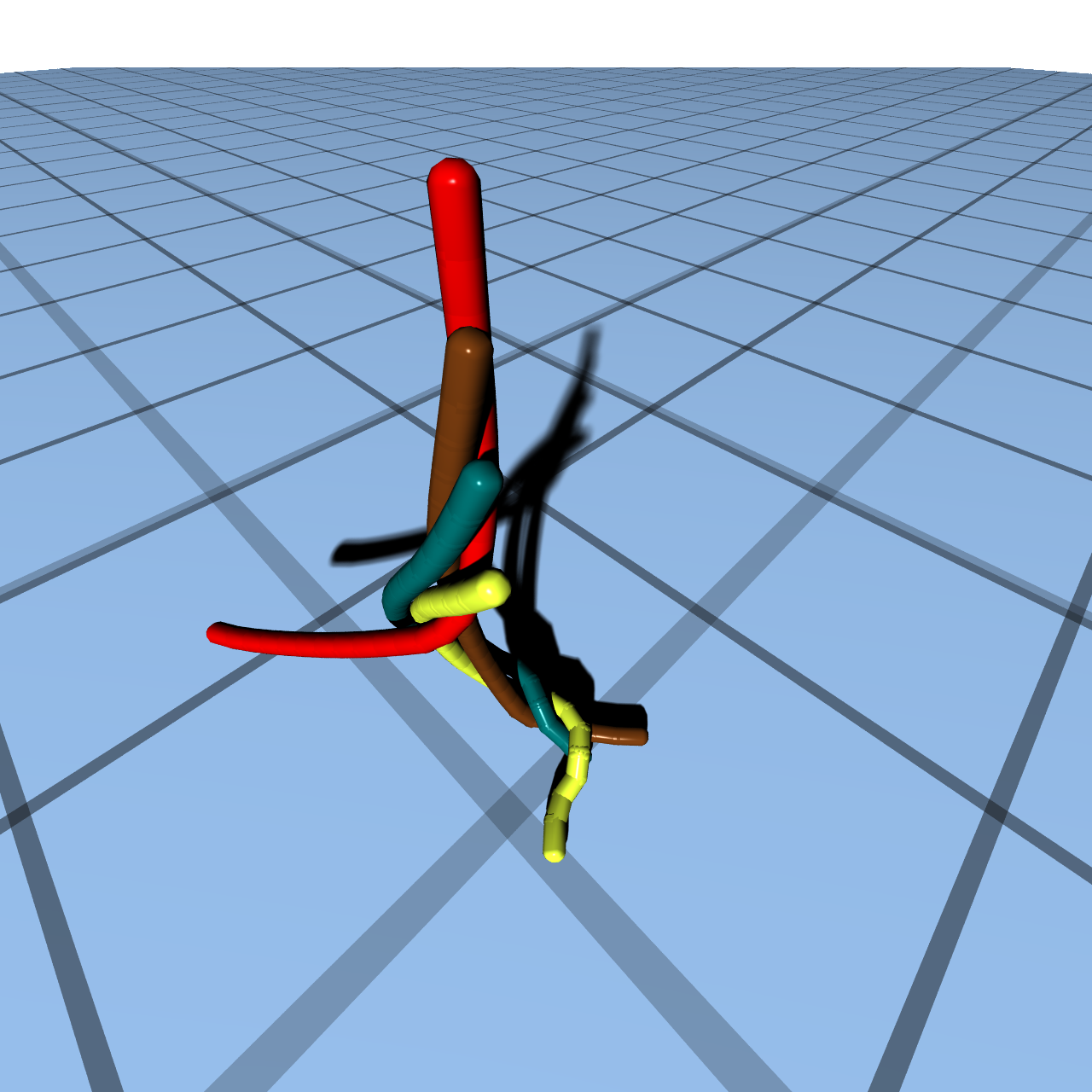}\\
\end{tabular}
\caption{Simulated frames of four rotating chains using mesh-based (left) and our convex-polyhedron-based shape representations (right), where we use~\cite{Li2020IPC} as the potential energy for mesh-based representation.}
\label{fig:bench}
\end{figure}
\begin{figure}[ht]
\centering
\includegraphics[width=.45\textwidth]{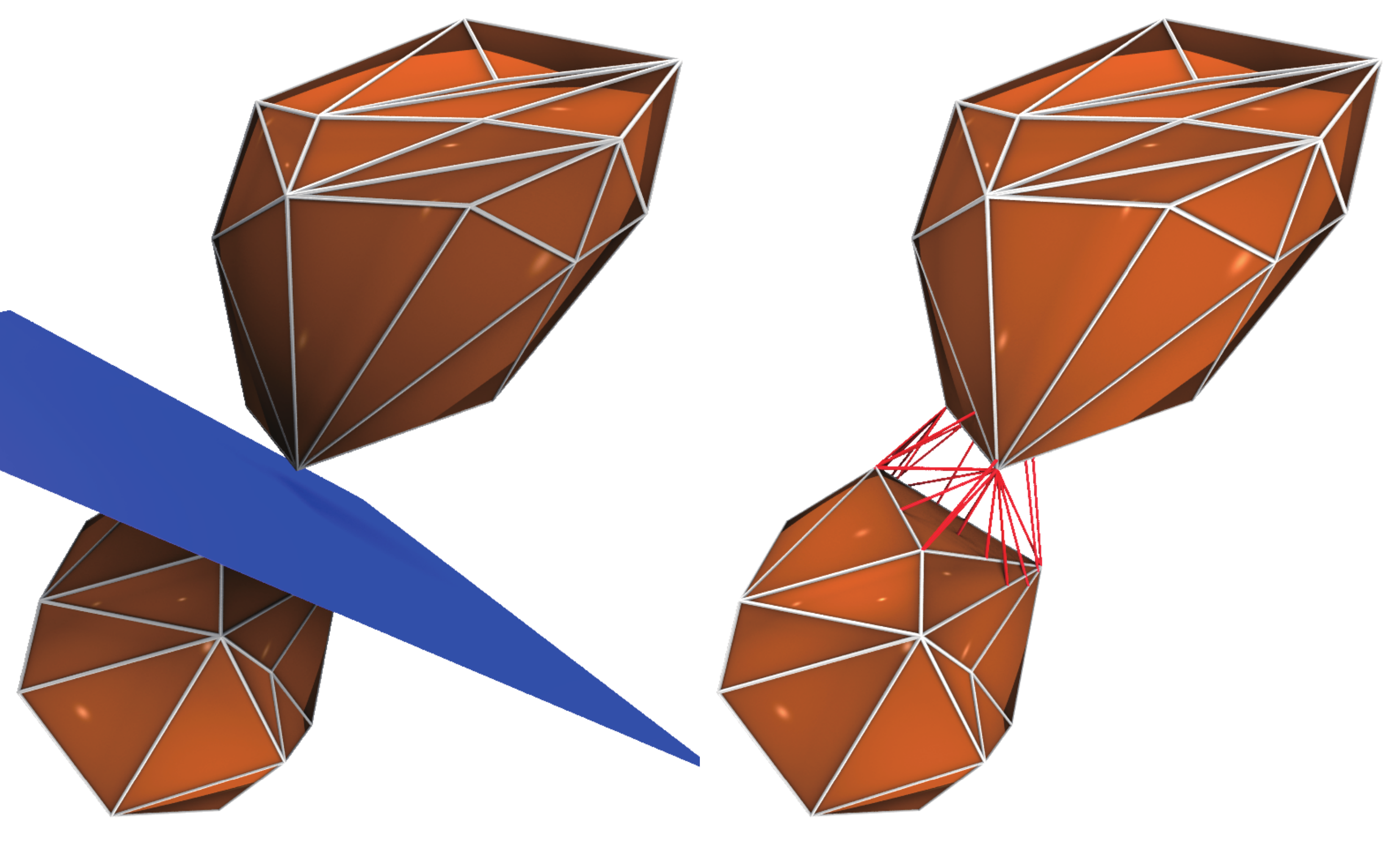}
\caption{ We illustrate the two settings for one example. Our method uses a separating plane (blue) to formulate contact potential under a convex-polyhedron-based representation, while under a mesh-based representation, contact potential is formulated via distance between geometric primitives (red).}
\label{fig:cost}
\end{figure}
\begin{figure}[ht]
\centering
\setlength{\tabcolsep}{1px}
\begin{tabular}{cc}
\includegraphics[width=.22\textwidth]{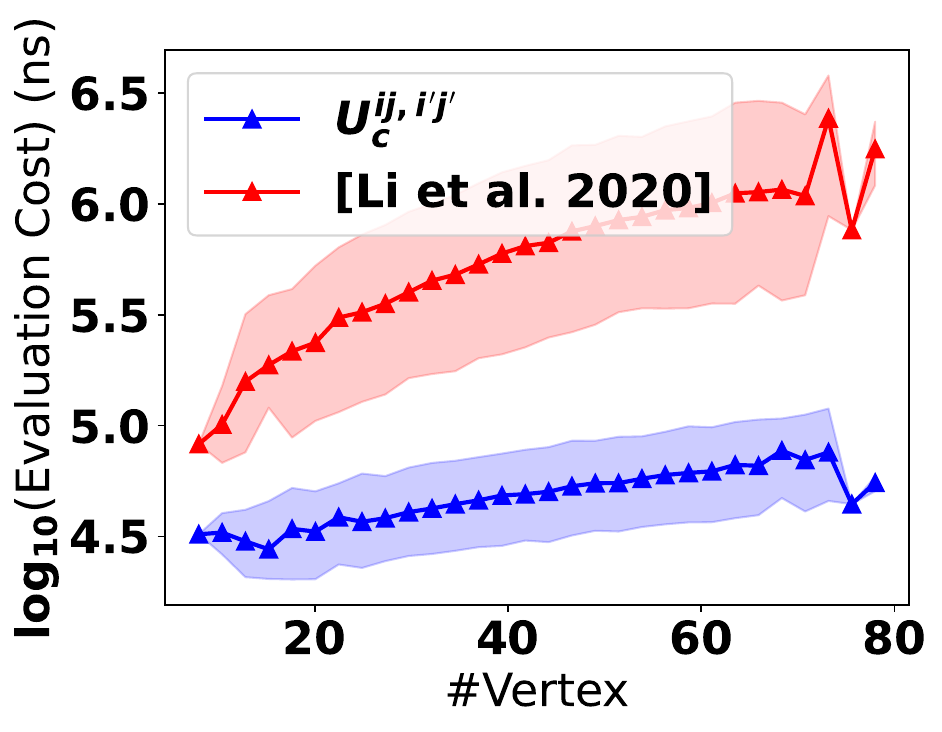} &
\includegraphics[width=.22\textwidth]{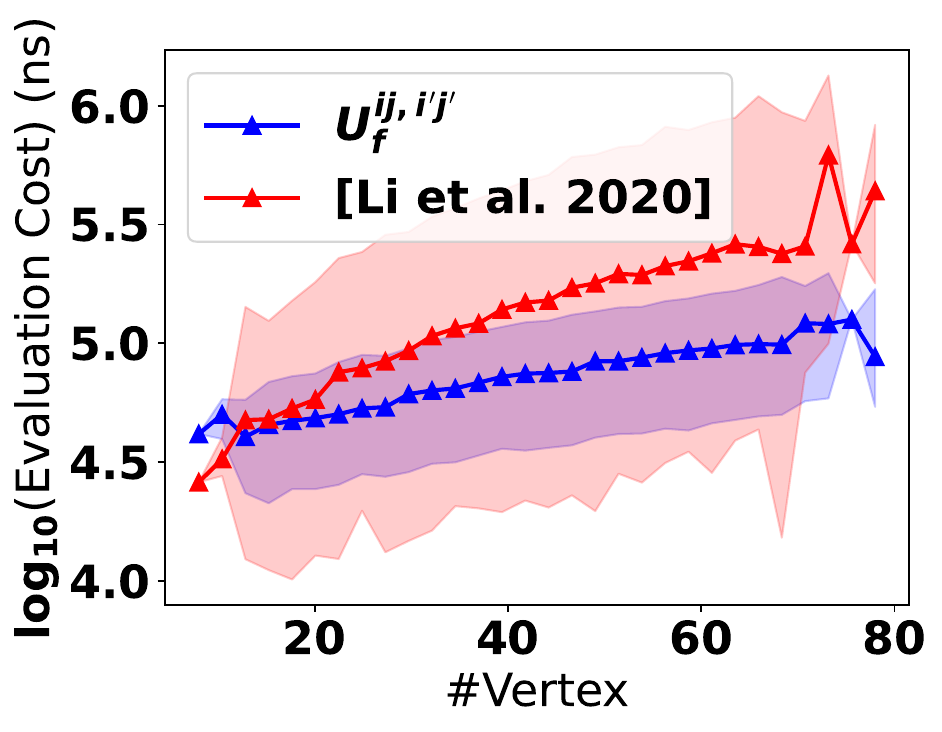} \\
\end{tabular}
\vspace{-10px}
\caption{We compare the evaluation cost of the contact potential ($U_c^{ij,i'j'}$) and the frictional damping potential ($U_f^{ij,i'j'}$) in these two cases.}
\label{fig:cost-time}
\end{figure}
\subsection{Performance Evaluation} 
We evaluate our method in a row of benchmark problems. In \prettyref{fig:bench}, we evaluate the computational cost. We simulate four rotating chains causing a large number of (self-)intersections. Each chain has $\#$ ball joints each having 2 degrees of freedom, totalizing $|\theta|=4\times 2\times\#$ degrees of freedom. The simulation is conducted using two methods, with mesh-based and our shape representations based on convex polyhedrons. Under mesh-based representation, we use the (frictional) contact potentials proposed in~\cite{Li2020IPC} to replace our $U_{c,f}^{ij,i'j'}$. The per-timestep cost of both methods are plotted against $\#$ in~\prettyref{fig:bench-time}. The performance breakdown is summarized in~\prettyref{table:performance}. 
\begin{wrapfigure}{r}{0.24\textwidth}
\centering
\includegraphics[width=.22\textwidth]{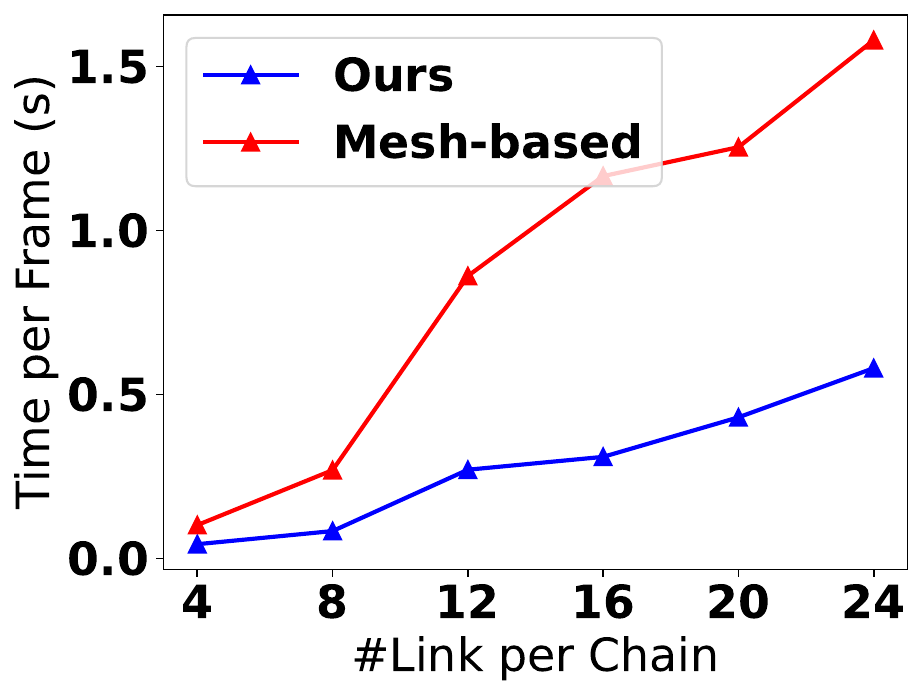}
\caption{The per-timestep cost of both methods are plotted against the number of links in each chain.}
\label{fig:bench-time}
\end{wrapfigure}
Our major bottleneck lies in the valuation of potential computation and the backward differentiation. Thanks to the use of minimal coordinates, the dimension of our configuration space is rather low and the system matrix inverse in the Newton's step is fast to compute, whose overhead is neglectable ($1.75\%$ in~\prettyref{table:performance}). Our contact potential is faster to evaluate than~\cite{Li2020IPC} as shown in~\prettyref{fig:cost-time}, although our potential takes a more complex form as illustrated in~\prettyref{fig:cost}. This is because, for two convex hulls each with $M$ vertices, our potential $U_c^{ij,i'j'}$ only involves $O(2M)$ terms, while $O(M^2)$ terms are needed in the worst case under a mesh-based representation, one for each pair of geometric primitives. We generate random convex hulls with different number of vertices and plot the cost of computing $U_c^{ij,i'j'}$ in~\prettyref{fig:cost}. Our potential evaluation can be $2.55-32$ times faster than one using mesh-based representation.

\begin{figure*}[ht]
\setlength{\tabcolsep}{1px}
\begin{tabular}{cccc}
\includegraphics[frame,trim=8cm 10cm 8cm 15cm,clip,height=.16\textwidth]{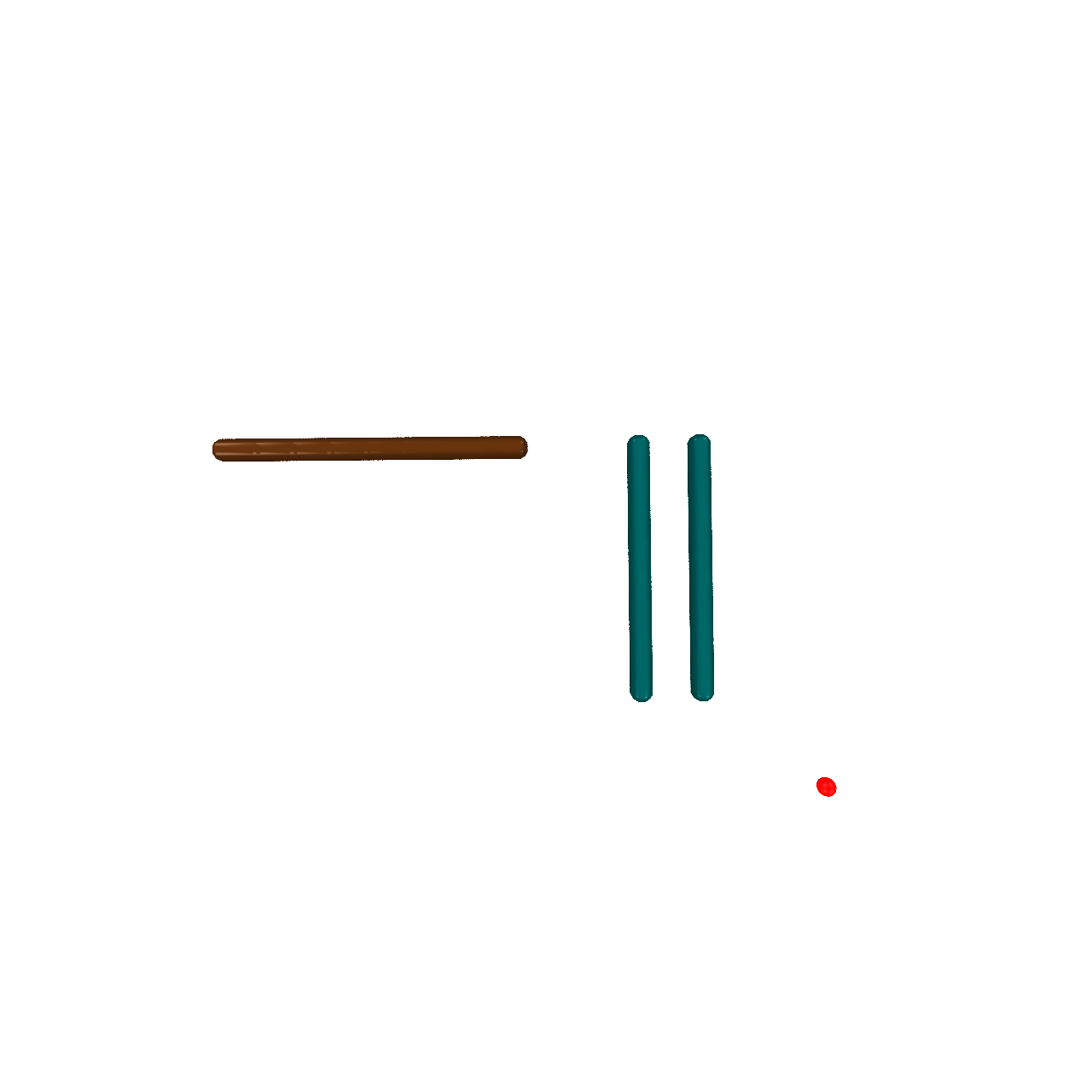} &
\includegraphics[frame,trim=2cm 10cm 8cm 15cm,clip,height=.16\textwidth]{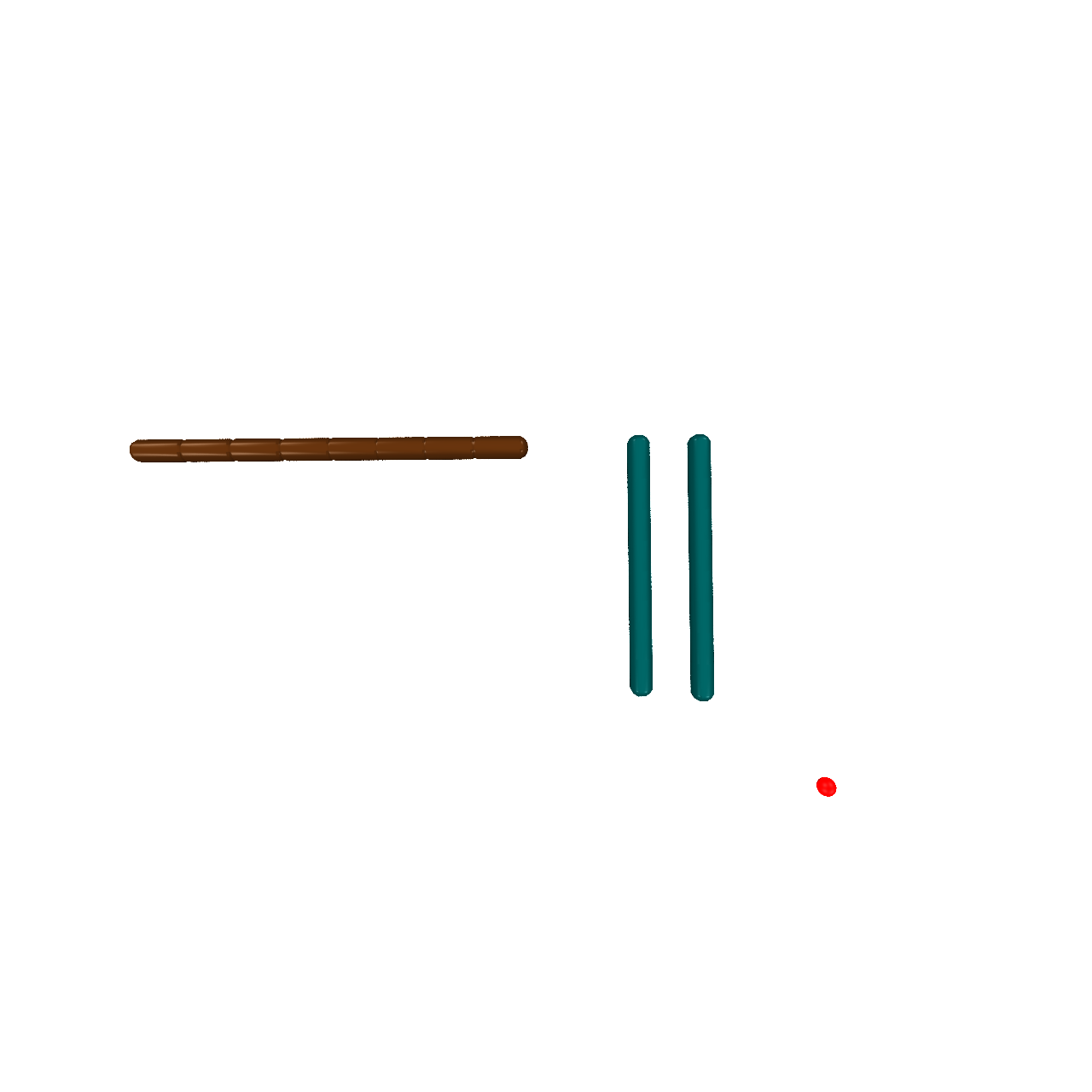} &
\includegraphics[frame,trim=8cm 10cm 8cm 15cm,clip,height=.16\textwidth]{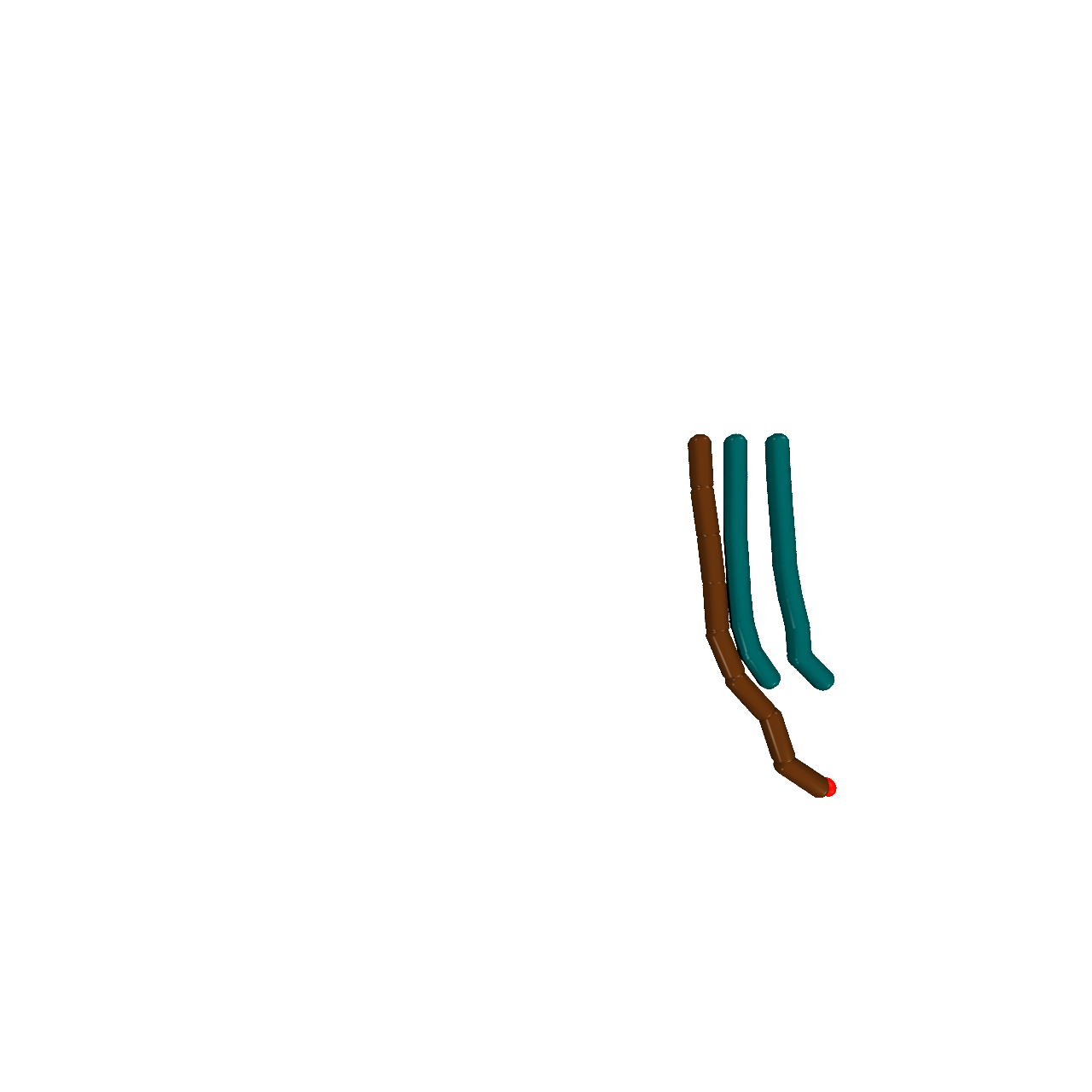}&
\includegraphics[height=.16\textwidth]{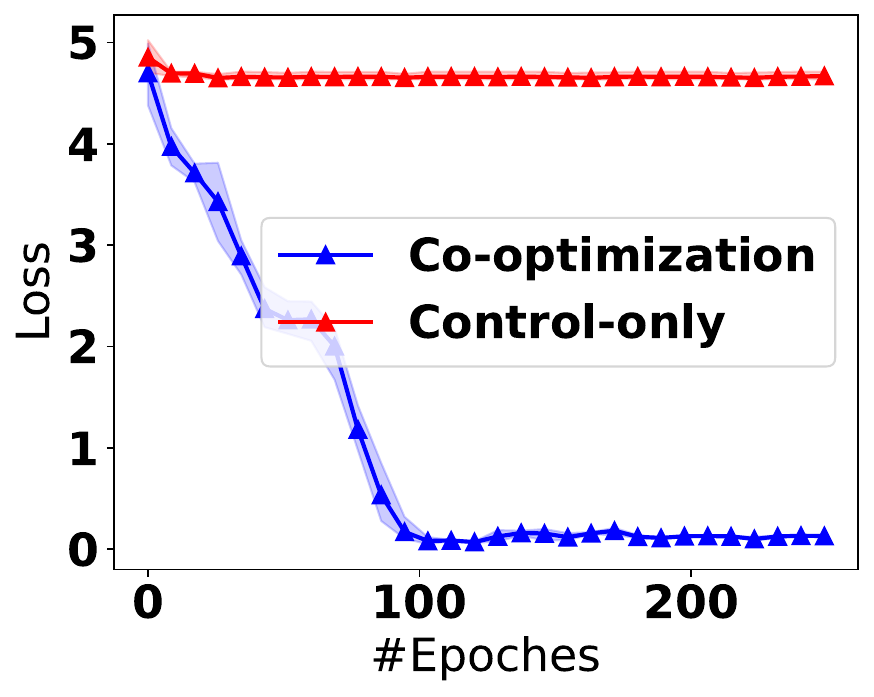} \\
\end{tabular}
\vspace{-10px}
\caption{ We optimize the attachment points and link-lengths of a 8-link chain to reach a target red point (bottom right). From left to right: The initial configuration of two chains; the optimized chain with elongated links; the simulated sequence of the chain reaching the target red point; and the convergence history of the benchmark with and without co-design.}
\label{fig:chain}
\end{figure*}
\subsection{Comprehensive Derivative Information}
In~\prettyref{fig:chain}, we show that SDRS can provide derivatives with respect to various design and control variables. We initialize three 8-link chains, where only the left (gray) chain is actuated and the two right chains are passive. We use an 1D translational actuator for the horizontal base position of the left chain. The goal of control is to have the left chain reach the (red) target point. We parameterize the control signal using a cubic spline with its 4 coefficients being the controller parameters, with an entire horizon of $H=250$. In addition, we include the attachment points of each link on the left chain as additional $8$ design variables, while relying on the shared vertices $x_{ijk}$ to keep the chain connected. As shown in~\prettyref{fig:chain} (right), our co-design can significantly outperform pure controller optimization by elongating the chain to reach the target point, taking $16.63$ (min) to converge. Indeed, our co-design can further decrease the loss function by $98.5\%$ and our optimized chain reaches the red target with an error around $0.2$ times the link length.

\begin{figure}[ht]
\centering
\setlength{\tabcolsep}{1px}
\begin{tabular}{cc}
\includegraphics[height=.2\textwidth]{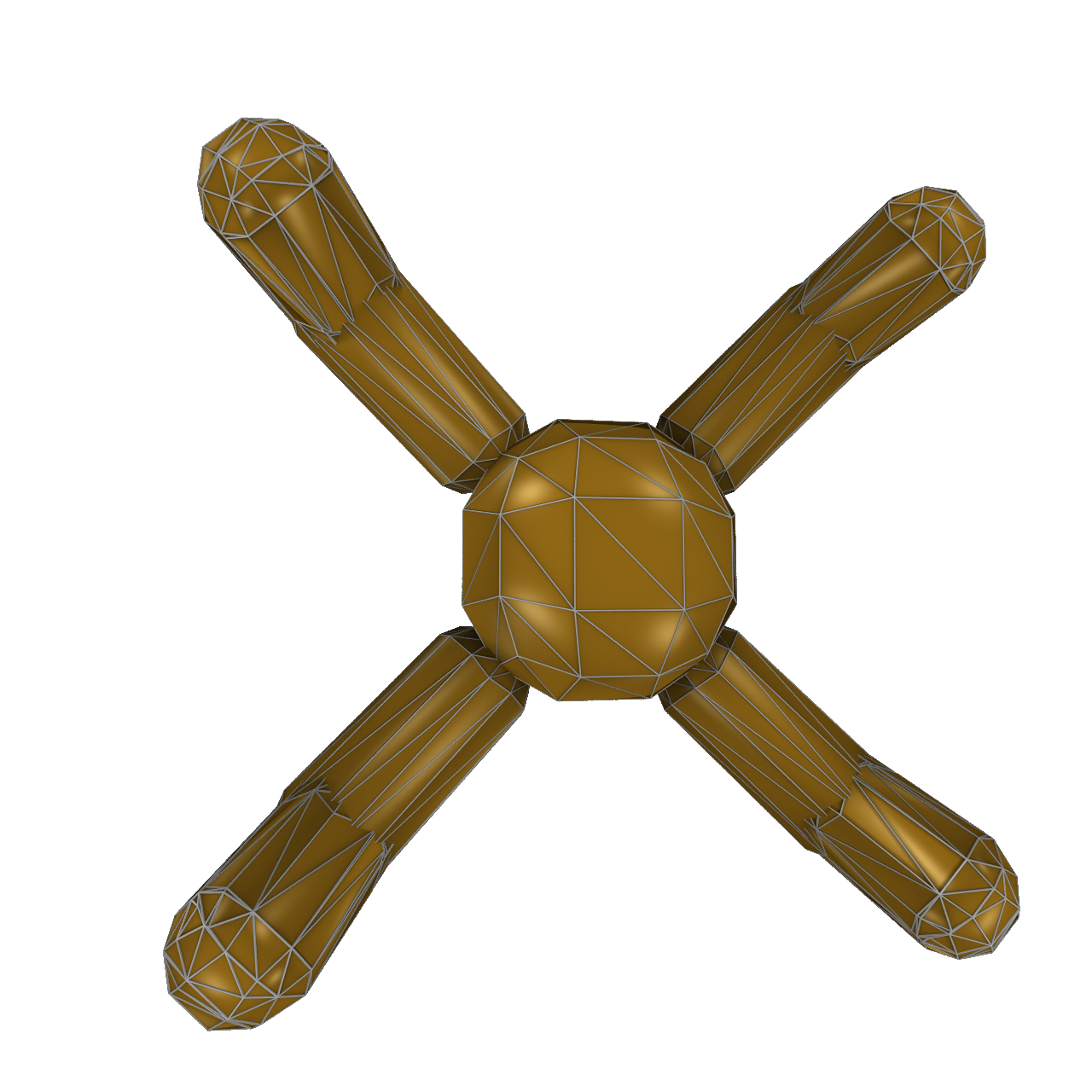} &
\includegraphics[height=.2\textwidth]{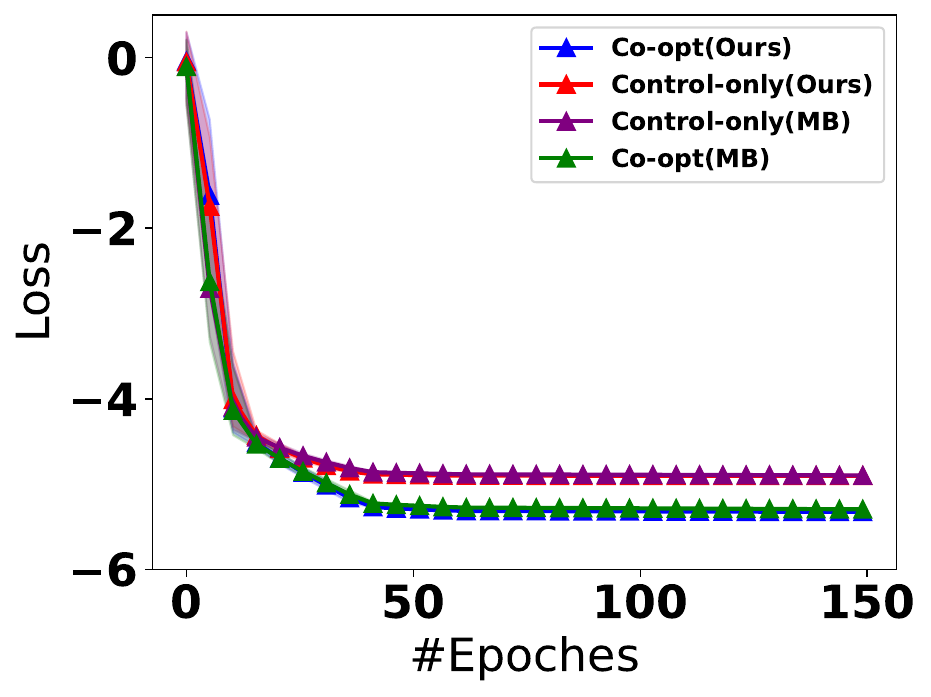}\\
\multicolumn{2}{c}{\includegraphics[height=.2\textwidth]{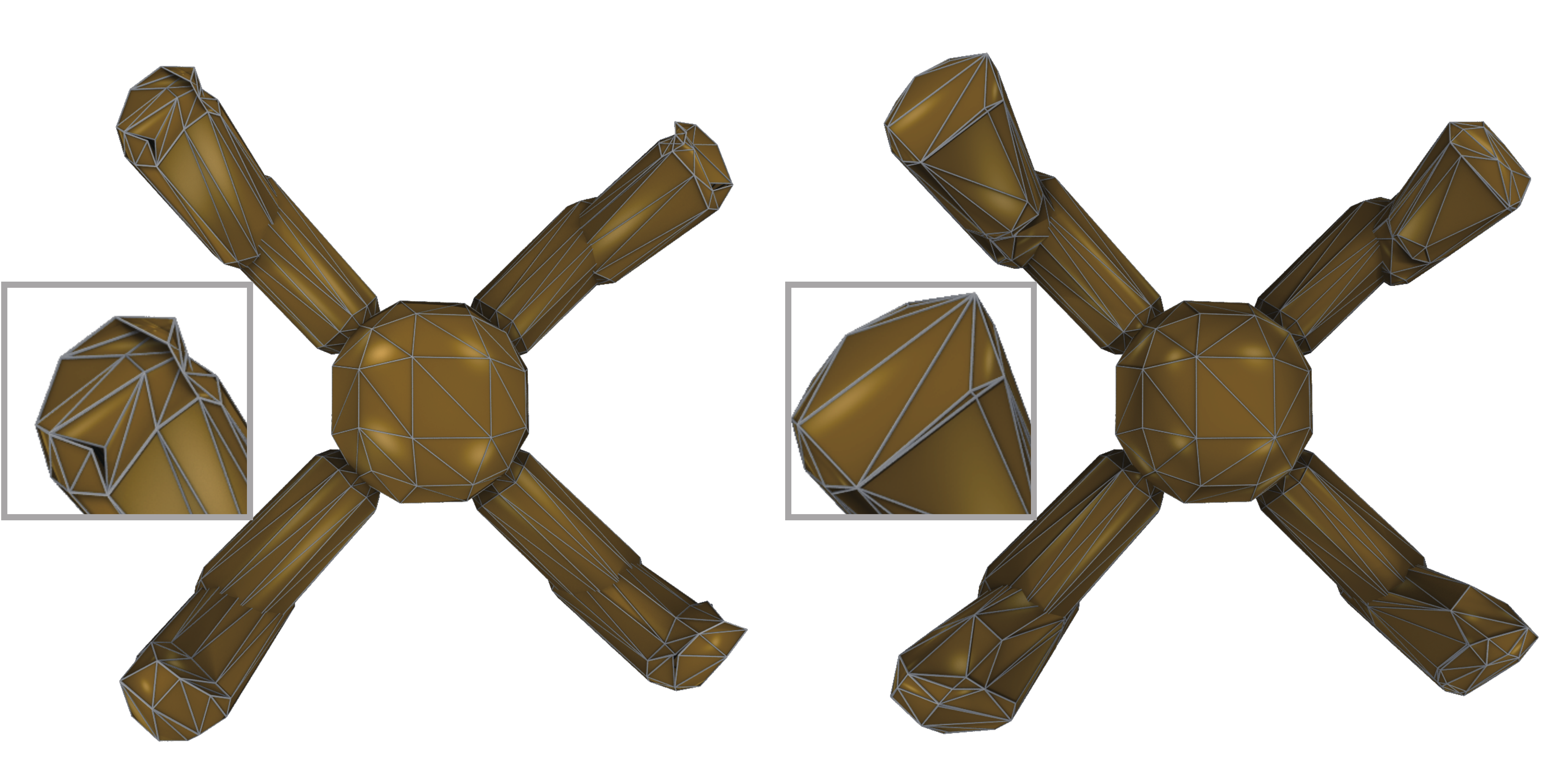}}
\end{tabular}
\vspace{-10px}
\caption{Top row: We show the original spider leg shape (left) and the convergence history (right). Bottom row: We show optimized leg shape with our approach (left) and mesh-based (MB) method (right). The shape of end-effector is highlighted to the left, which shows that mesh-based method creates unnecessary notches and wedges.}
\label{fig:spider}
\end{figure}
\subsection{Benchmark: Quadruped Locomotion}
We illustrate our first benchmark in~\prettyref{fig:spider}, where we optimize the motion of a 9-link spider robot with 16 degrees of freedom walking on a $15^\circ$-sloped ground. We use a set of parameters $\bar{\alpha_c}=3e^{-1}, \bar{\alpha_d}=4e^{-3}$, and $H=550$. The goal of optimization is to have its center-of-mass reach a target position by setting the loss function as $R(\theta^H)=-\|\text{COM}(\theta^H)-(10\cos(15),0,10\sin(15))\|^2$, where the COM is initially at $(0,0,0)$ and $-z$ is the gravitational direction. Our controller is parameterized in a similar way as in~\cite{hu2019chainqueen}, by setting the PD-controller's position target $\theta_\star^t$ to be a linear combination of $4$ sine waves with learnable magnitudes, phase shifts, and frequency. We further set the derivative target $\dot{\theta}_\star^t$ to be the time-derivative of these sine waves. We set the PD gains to be $k_p=1e^2$ and $k_d=1e^1$. During the co-optimization stage, we focus solely on optimizing the shape of the robot's 4 feet. Each robot foot has an initial 96 vertices, and using V-HACD, we represent each foot using 3 convex polyhedrons. The convergence history of the optimization with and without shape co-optimization is plotted in~\prettyref{fig:spider}, both taking 4 (min) to converge, where the version with shape co-optimization further increase the spider's walking distance by $8.78\%$ from $4.90$ to $5.33$. Our improvement in this benchmark is due to the shape co-optimization that deforms the robot foot to have a flat contact surface (\prettyref{fig:spider} bottom right), leading to larger friction and thus better locomotion performance. 

We further compare our approach with~\citet{Xu-RSS-21} on this benchmark, which assumed that the robot shape is controlled by a cage mesh. As they use mesh-based (MB) simulator, we use IPC~\cite{Li2020IPC} to handle contact and friction. We use same task settings and simulator configuration as in our approach, to ensure that both simulators will generate nearly identical trajectories under the same PD signal and exhibit nearly identical performance when only controllers are optimized. For fairness, we tune the resolution of the cage used in~\citet{Xu-RSS-21} to have a same number of vertices as those in our method. The convergence history of the optimization with and without shape co-optimization with fine-cage mesh-based method is plotted in~\prettyref{fig:spider}, where the version with shape co-optimization further increase the spider's walking distance by $8.16\%$ from $4.90$ to $5.30$. We observed that the improvements achieved by mesh-based methods are comparable to those of our approach, which is expected since this benchmark does not demand significant shape changes. However, a notable difference emerges in the optimized leg shape. As shown in~\prettyref{fig:spider}, both methods attempt to expand the spider's end effector to increase the contact surface. However, the mesh-based technique introduces additional notches and wedges that are less suitable for fabrication. In contrast, our approach enables convenient adjustment of local details by moving vertices in or out of the convex polyhedra, effectively resolving these artifacts.

\begin{figure}[ht]
\centering
\setlength{\tabcolsep}{0px}
\begin{tabular}{cccc}
\includegraphics[trim=8.5cm 6.5cm 13cm 1cm,clip,width=.15\textwidth]{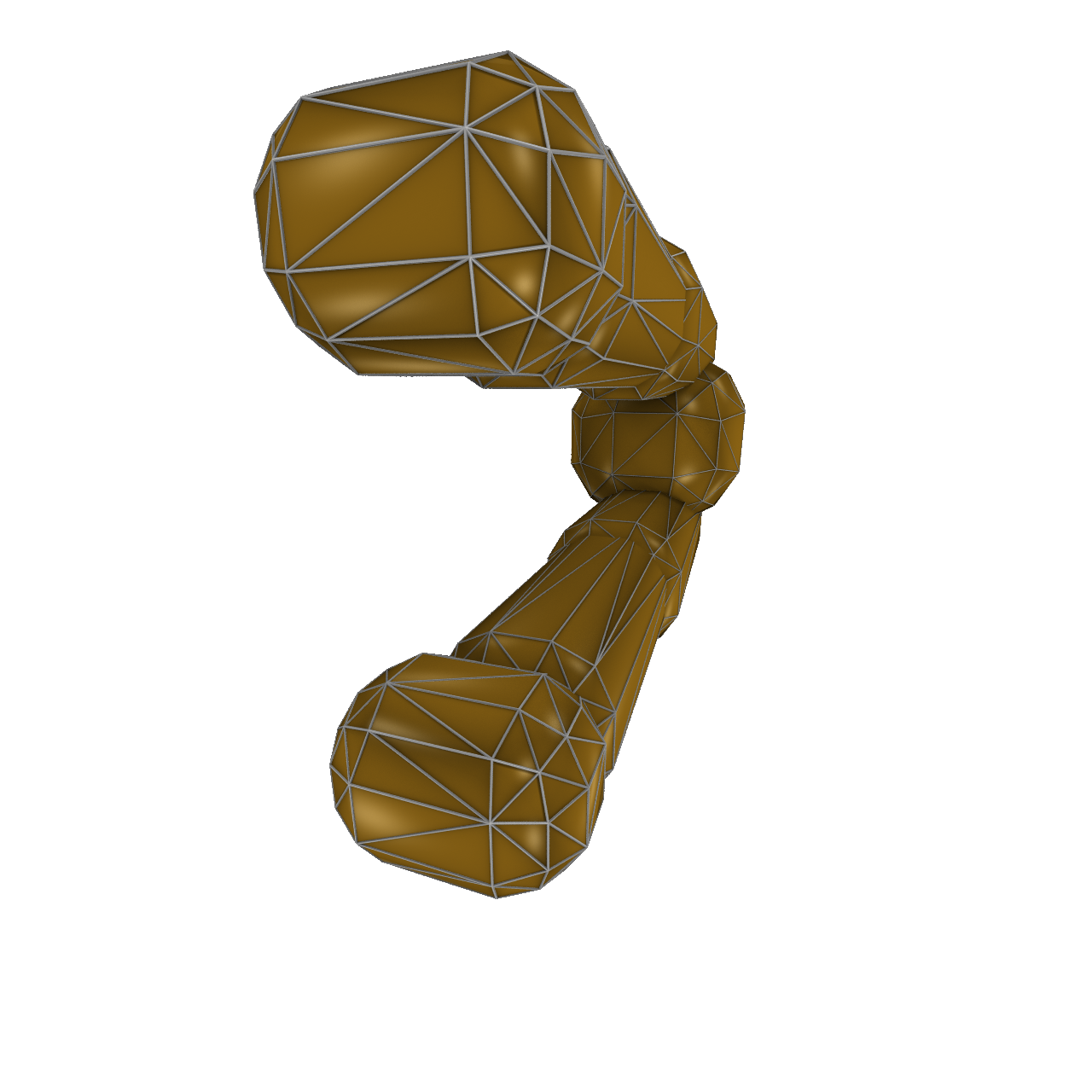}&
\includegraphics[trim=8.5cm 6.5cm 13cm 1cm,clip,width=.15\textwidth]{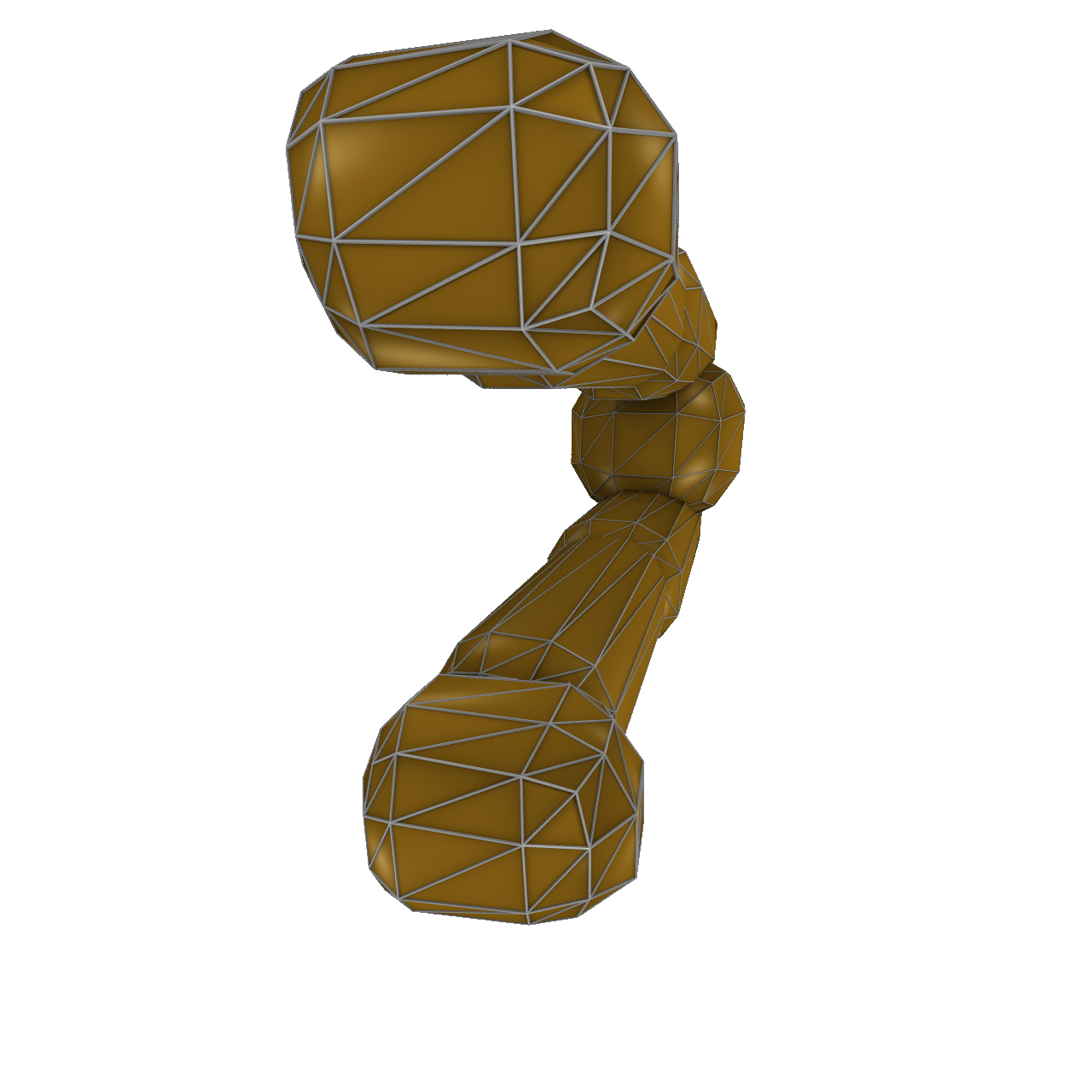}&
\includegraphics[trim=8.5cm 6.5cm 13cm 1cm,clip,width=.15\textwidth]{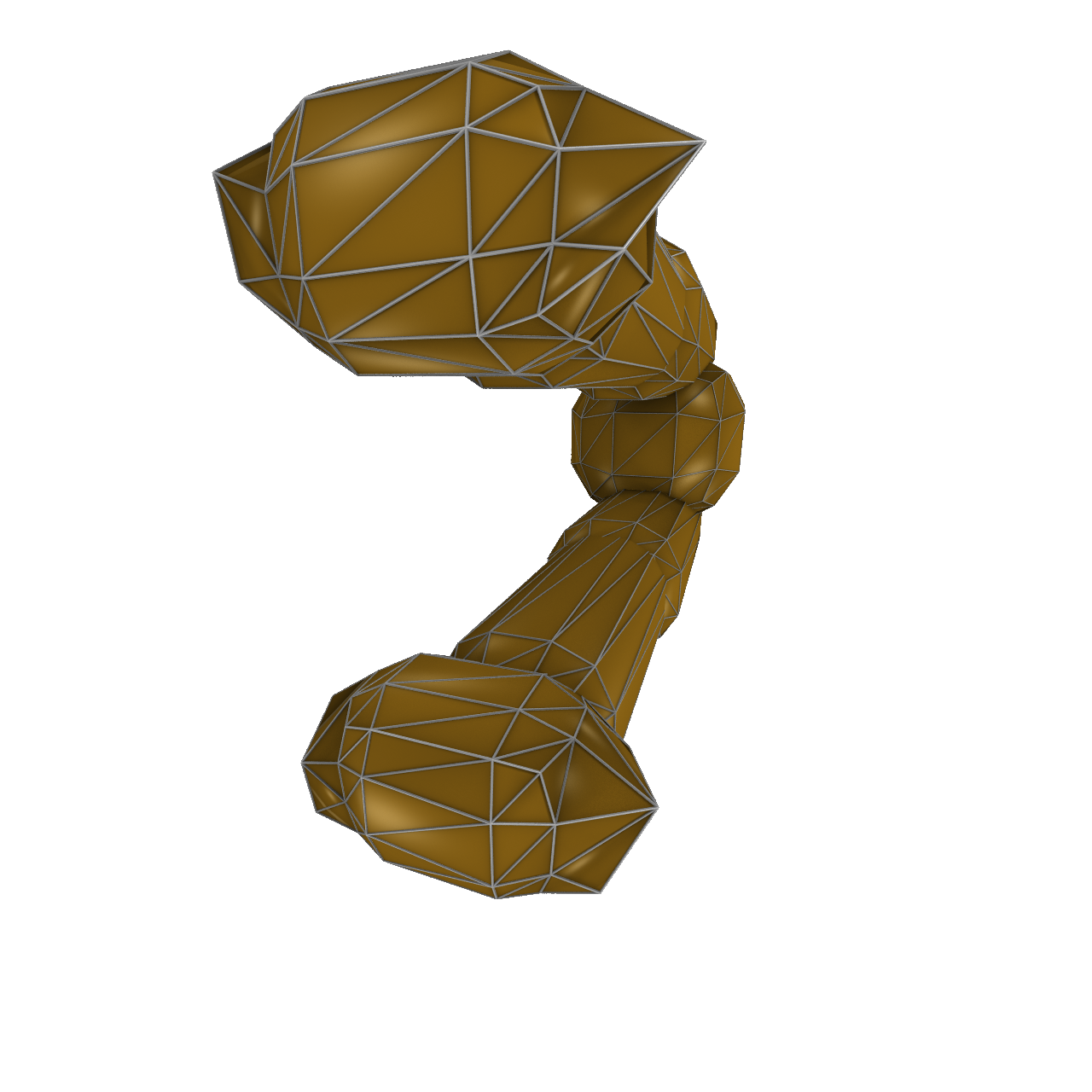}
\end{tabular}
\caption{From left to right, we show the original bipedal foot shape, optimized foot shape using mesh-based methods with coarse cages and fine cages.}
\label{fig:bipedal}
\end{figure}
\begin{figure}[ht]
\centering
\setlength{\tabcolsep}{0px}
\begin{tabular}{cccc}
\includegraphics[trim=8cm 5cm 2cm 1cm,clip,height=.2\textwidth]{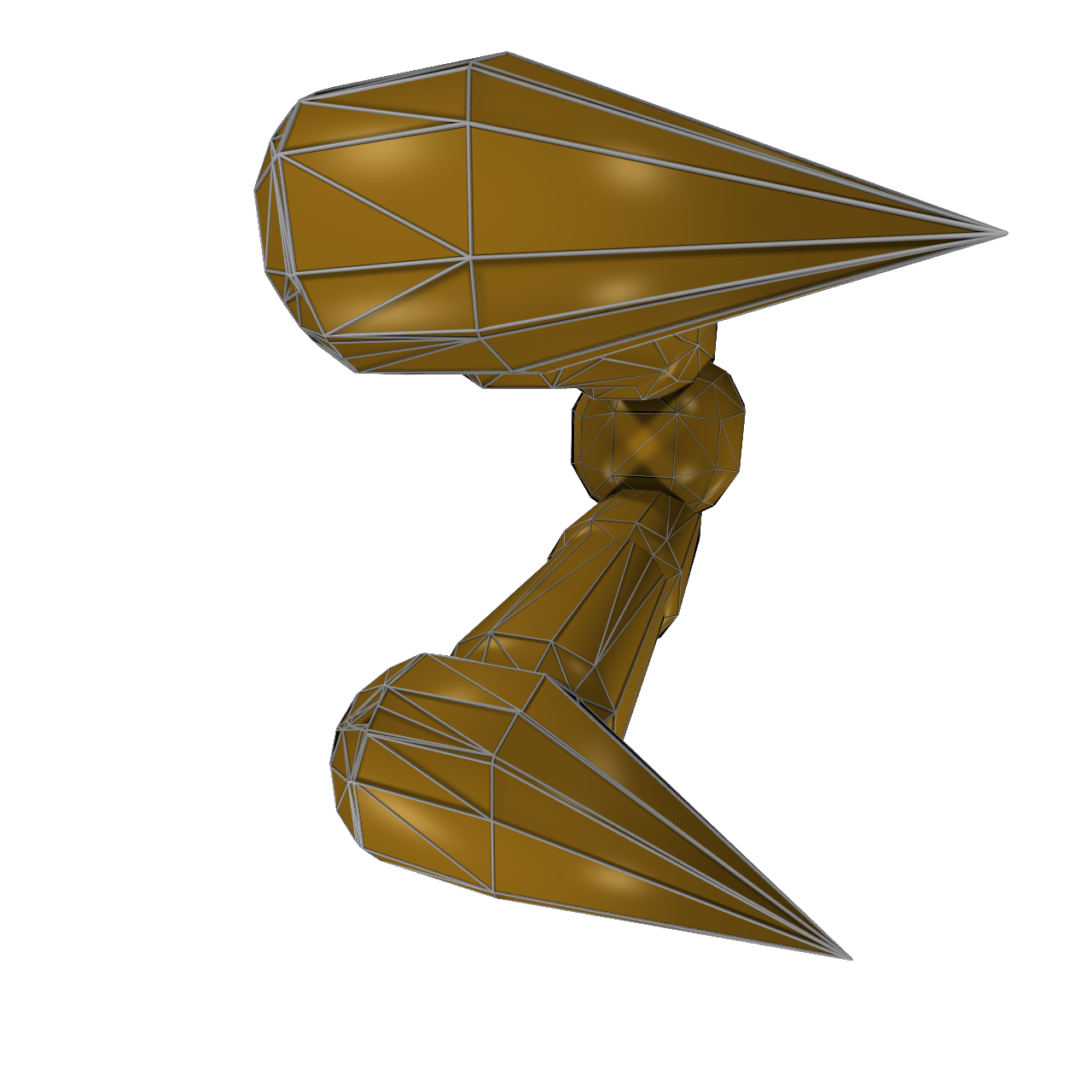}&
\includegraphics[height=.2\textwidth]{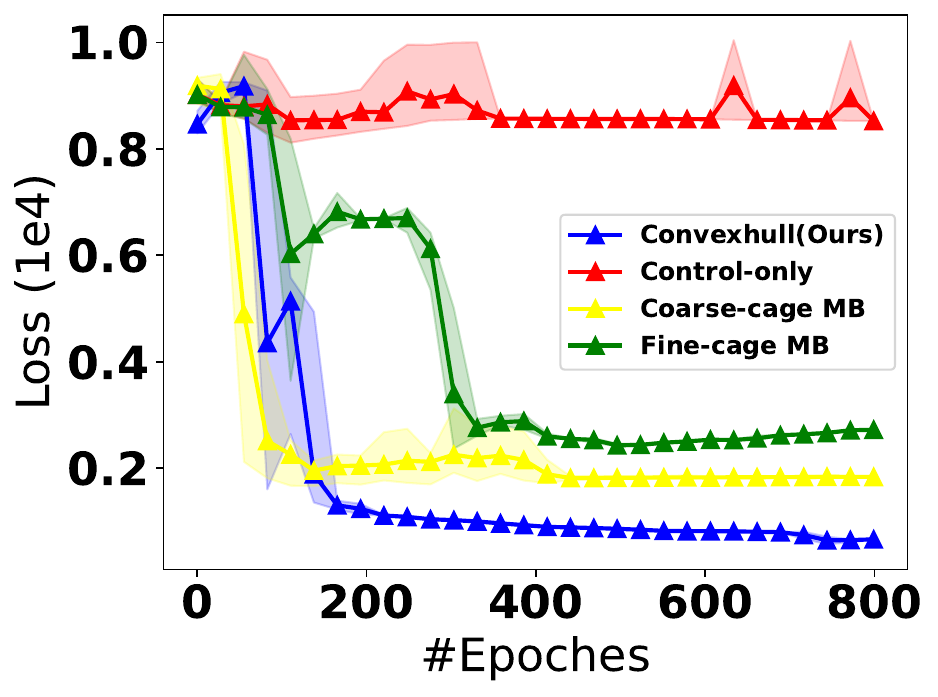}
\end{tabular}
\caption{We show optimized bipedal foot shape with our approach (left) and the convergence history (right).}
\label{fig:bipedal_res}
\end{figure}
\subsection{Benchmark: Bipedal Locomotion}
In~\prettyref{fig:bipedal} and~\prettyref{fig:bipedal_res}, we optimize the motion of a 7-link bipedal robot with 12 degrees of freedom walking on the ground. In this benchmark, the initial shape of the robot's foot are not well-designed. They are relatively short in length and have smooth contact surfaces with the ground. Therefore, solely optimizing the controller is insufficient to maintain passive balance for the robot during walking. We use a different set of parameters $R(\theta^H)=-\|\text{COM}(\theta^H)-(0,3,0)\|^2, H=400, \bar{\alpha_c}=1e^{-2}$ and $\bar{\alpha_d}=4e^{-3}$, where the COM is initial at $(0,0,0)$ and $-z$ is the gravitational direction. We set the PD gains to be $k_p=1e^3$ and $k_d=1e^2$. During co-optimization, we confine our design space to the shape of the robot's 2 feet, along with controller parameters. Each robot foot has 96 vertices. Using V-HACD, we represent each foot using 3 convex polyhedrons. We compare our method with ~\citet{Xu-RSS-21}, where we experiment with two resolutions of an axis-aligned case. According to their original paper, we first experimented with a coarse cage having 26 vertices. For fairness, we also increase the cage resolution to have 96 vertices, matching the degree of freedom used in our method. As illustrated in~\prettyref{fig:bipedal}, the coarse cage cannot effectively deform the feet shape to maintain balance, while the fine cage can slightly deform the feet. Instead, our approach in~\prettyref{fig:bipedal_res} significantly deforms the robot's feet, creating rough protrusions, increasing the length of the feet, and generating flat contact with the ground, enabling the robot to easily maintain balance even when walking with an open-loop controller. The convergence history of co-optimization using various methods are plotted in~\prettyref{fig:bipedal_res} (right).

\begin{figure}[ht]
\centering
\setlength{\tabcolsep}{0px}
\begin{tabular}{ccc}
\includegraphics[trim=6cm 0 14cm 0,clip,height=.18\textwidth]{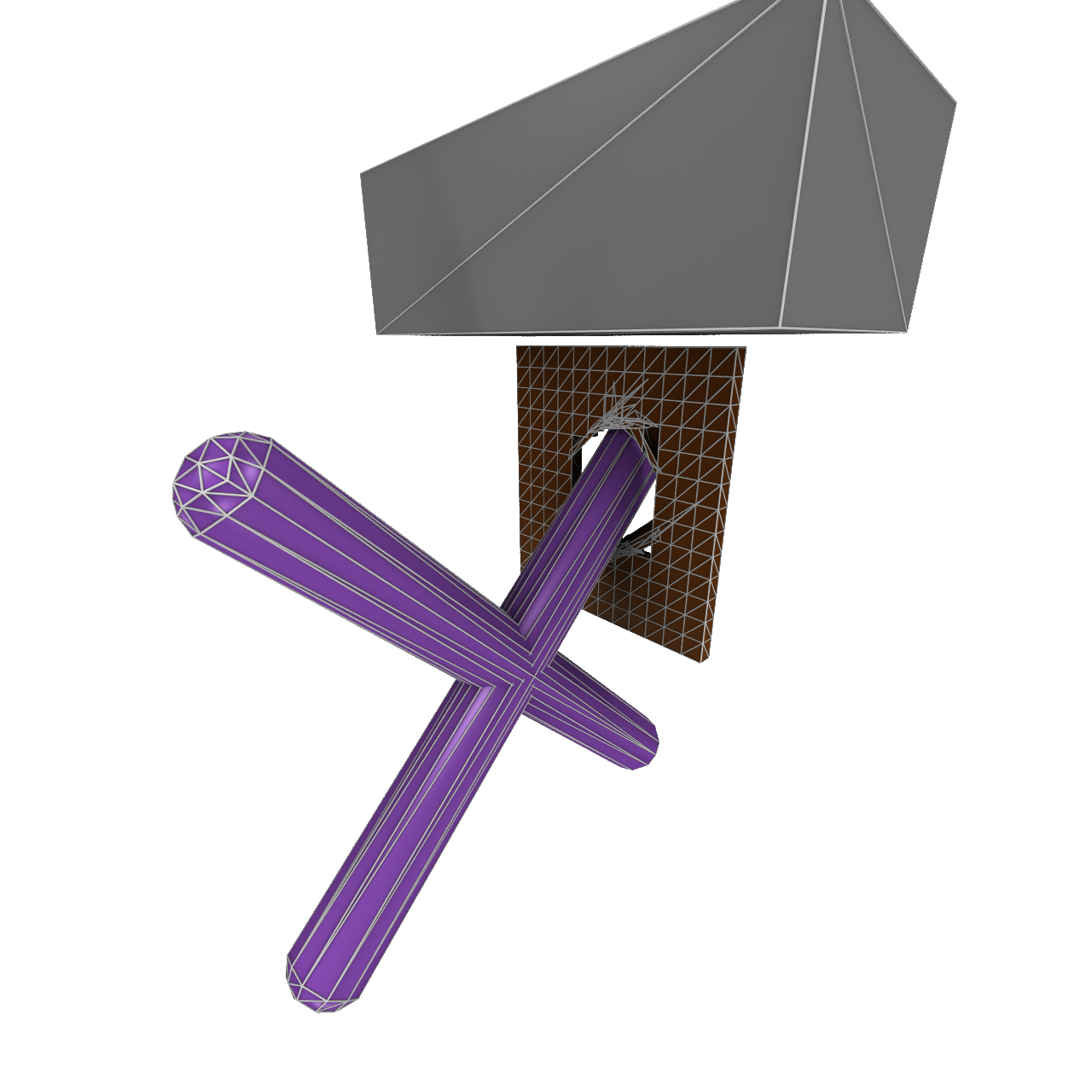}&
\includegraphics[trim=7cm 0 14cm 0,clip,height=.18\textwidth]{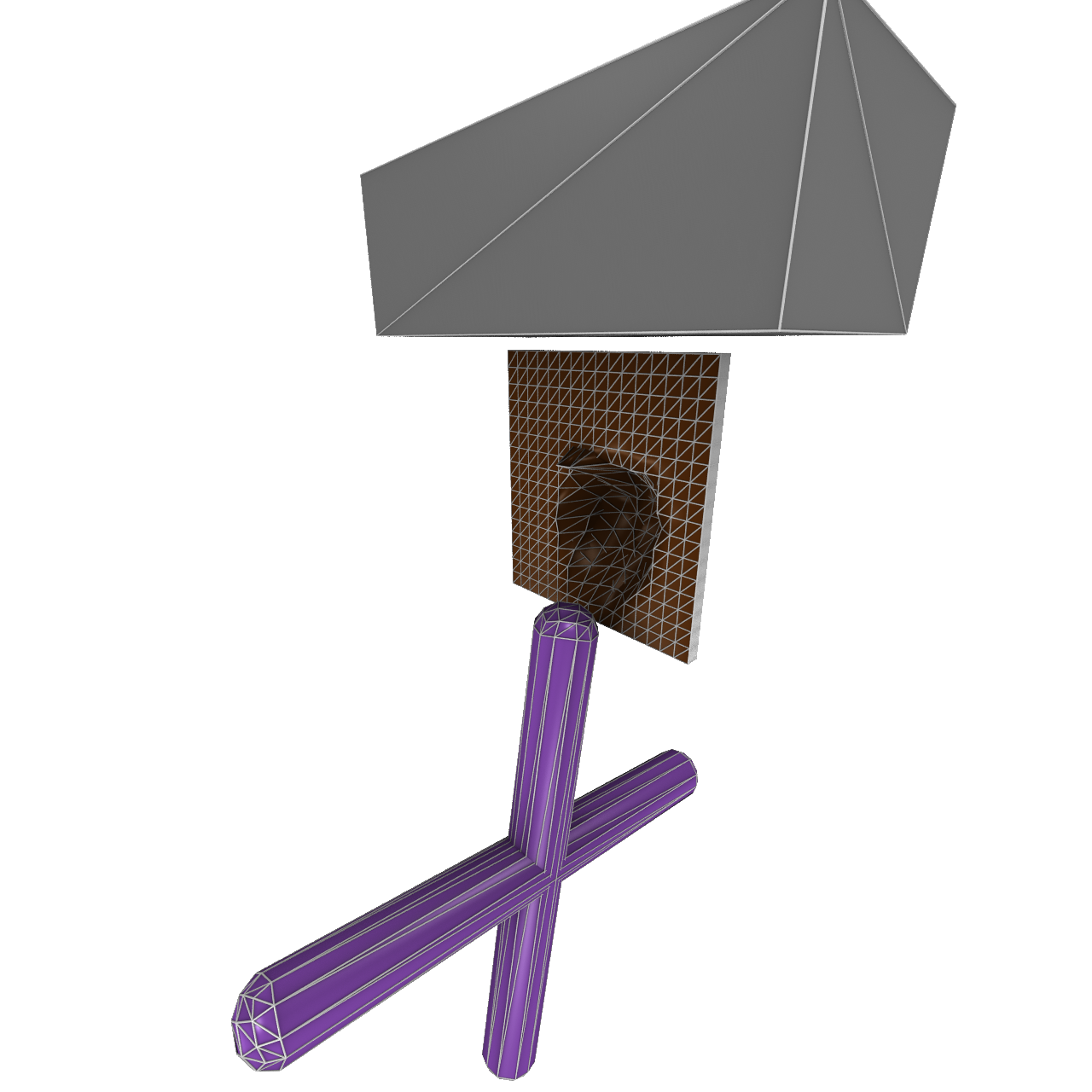}&
\includegraphics[height=.18\textwidth]{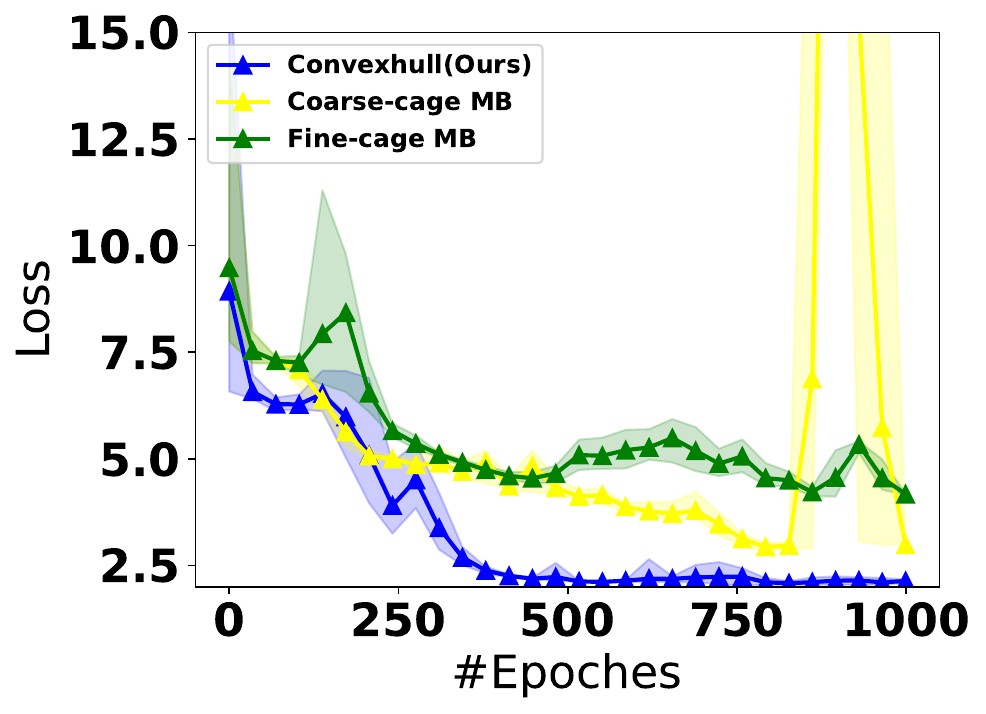}
\end{tabular}
\caption{ Our benchmark requires a single-jaw gripper to grasp an X-shaped object. Our co-optimization (left) change the gripper's topology by creating a hole to secure the object, while the mesh-based method (middle) failed to grasp the object. We show the convergence history of the benchmark using our approach and mesh-based method (right).}
\label{fig:passive_gripper}
\end{figure}
\begin{figure}[ht]
\centering
\setlength{\tabcolsep}{1px}
\begin{tabular}{ccc}
\includegraphics[trim=10cm 0 10cm 8cm,clip,height=.21\textwidth]{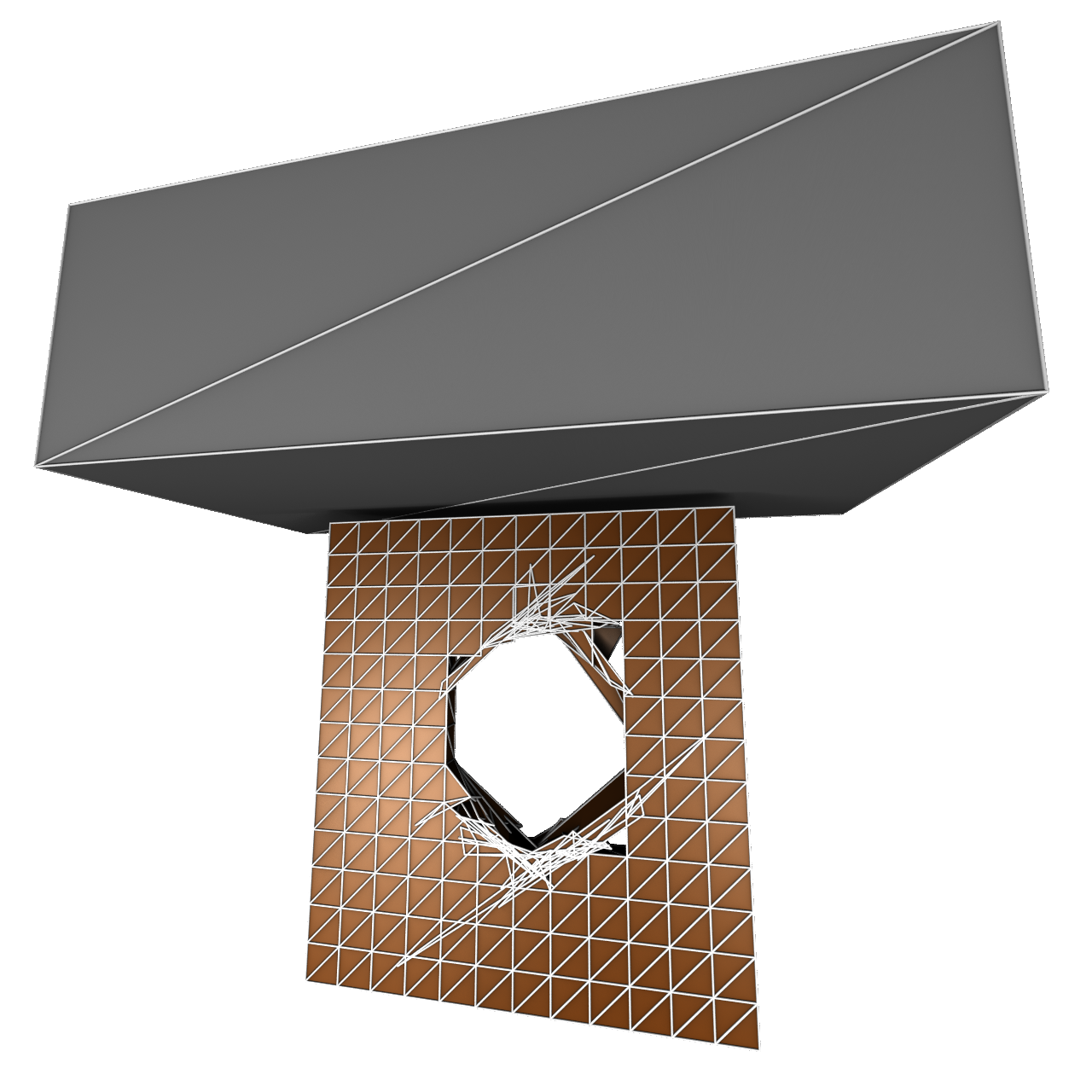}&
\includegraphics[trim=10cm 0 10cm 8cm,clip,height=.21\textwidth]{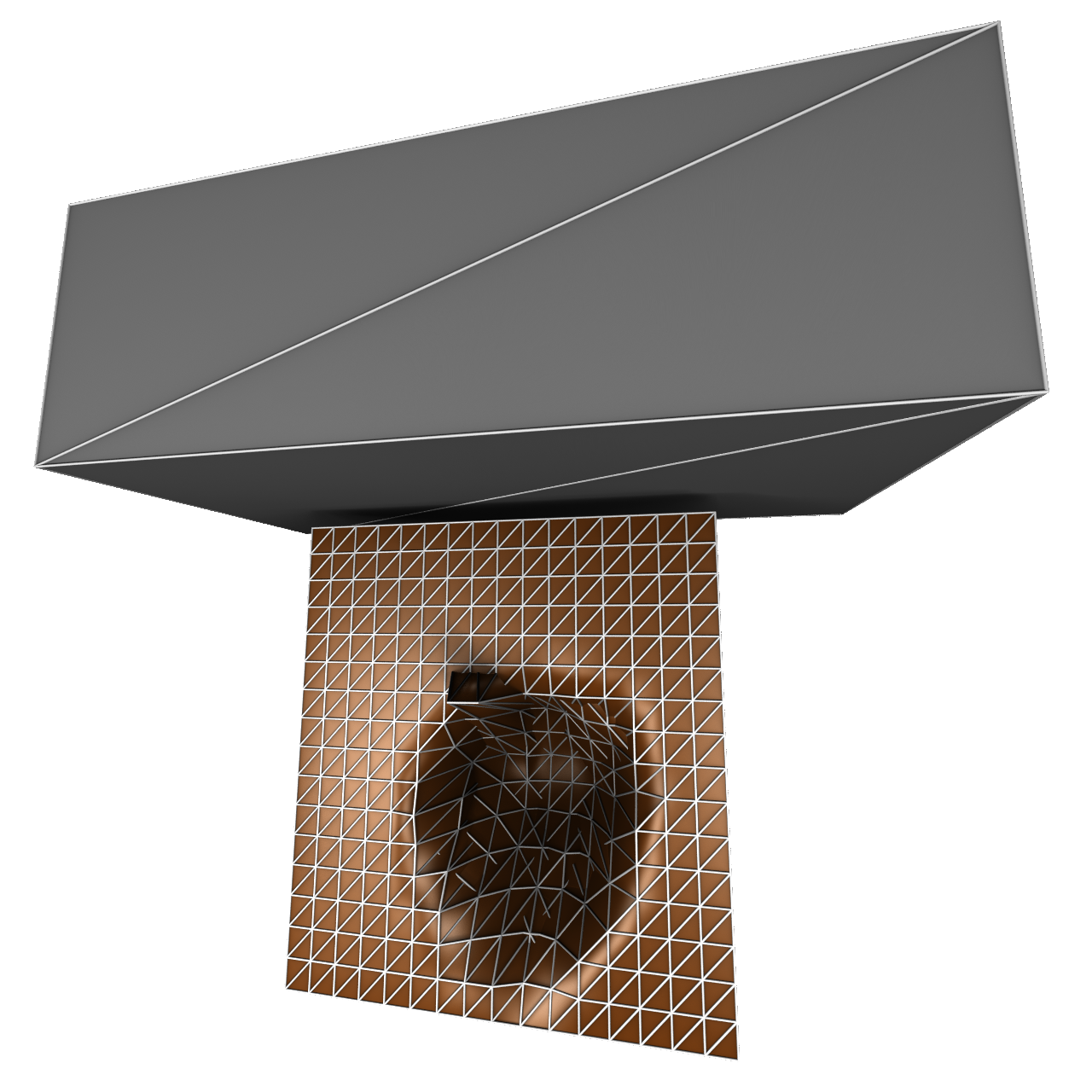}&
\includegraphics[trim=10cm 0 10cm 8cm,clip,height=.21\textwidth]{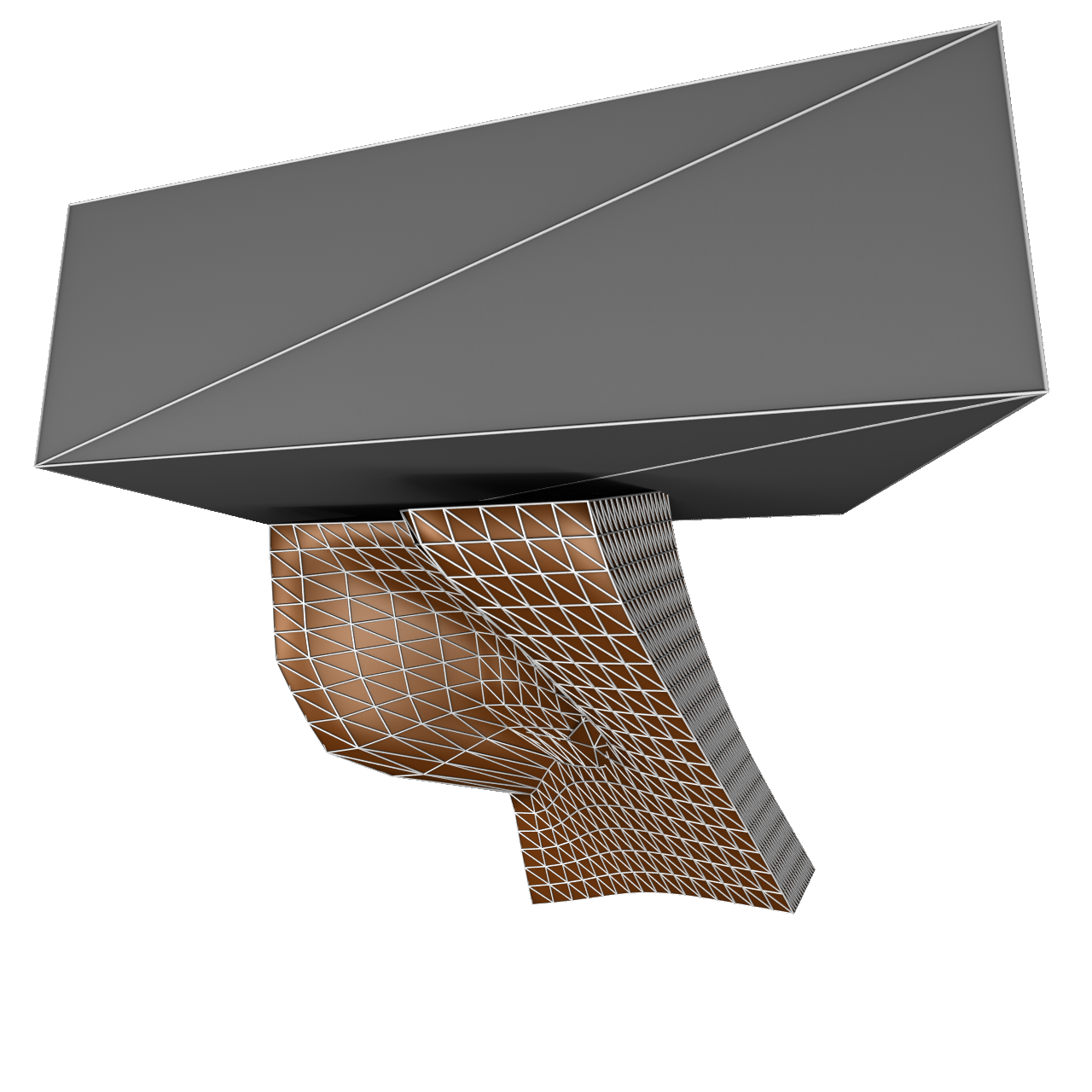}
\end{tabular}
\caption{A comparison of the two techniques in optimizing the single-jaw gripper. Our co-optimization (left) change the topology of original gripper shape, creating a hole to form a firm grasp of the object. Co-optimization by mesh-based method using fine cages (middle) generate local concavity. Coarse cages (right) exhibit global deformation and fail to generate useful local deformations.}
\label{fig:passive_gripper_shape}
\end{figure}
\subsection{Benchmark: Single-Jaw Gripper}
We highlight the capability of our design space to handle topological changes. In~\prettyref{fig:passive_gripper}, we optimize the motion and shape of a single-jaw gripper to grasp an X-shaped object placed on the ground. The gripper has 1 link with 6 degrees of freedom and the object also has 6 degrees of freedom. We set $H=480, \bar{\alpha_c}=5e^{-2}$ and $\bar{\alpha_d}=2e^{-3}$. We divide the grasping process into four stages. The first stage involves lowering the gripper to the appropriate position. In the second stage, the gripper approaches and grasps the object. The third stage involves adjusting the gripper to its initial orientation. Finally, in the fourth stage, the gripper is raised to a target position. During the first two stages, we define a penalty function $L(\theta^H)=\|COM_{\text{object}}(\theta^H)\|^2$ with the object's center of mass initialized at zero, aiming to keep the object staying still when the gripper approach and from a firm grasp. During the last two stages, we penalize any relative motions between the object and the gripper: $L(\theta^H)=\|COM_{\text{object}}(\theta^H)-COM_{\text{gripper}}(\theta^H)\|^2$. Clearly, regardless of individual controller optimization efforts, the initial gripper will fail to grasp the object successfully due to the lack of a well-designed mechanism to secure the object, meaning that co-optimization is necessary. We extensively compared the performance of our approach against mesh-based method~\cite{Xu-RSS-21} using coarse and fine cages in~\prettyref{fig:passive_gripper_shape}. The initial gripper has 1538 vertices uniformly distributed. Our coarse cages uses only $98$ vertices, and as usual, we use 1538 vertices for the fine cage for fairness of comparison. We find that mesh-based method using fine cages can optimize a groove on the gripper that could be used to secure the object. However, due to low friction coefficient, the object still cannot be successfully grasped. Instead, the mesh-based method using coarse cages struggles to generate large deformations conducive to grasping the object. Finally, none of these mesh-based methods can induce topological changes to help solving this task. In contrast, our method allows for the optimization of the gripper's topology. We evenly divide the gripper into 156 convex blocks, which brings a vertex count similar to that of the mesh-based method. Our method carefully drills a hole on the gripper for grasping the object, enabling successful object retrieval even with a lower friction coefficient. The convergence history is plotted in~\prettyref{fig:passive_gripper}, where our method takes around 55 (min) to converge. 

\begin{figure}[ht]
\centering
\setlength{\tabcolsep}{0px}
\begin{tabular}{cc}
\includegraphics[trim=0 15cm 0 0,clip,width=.24\textwidth,frame]{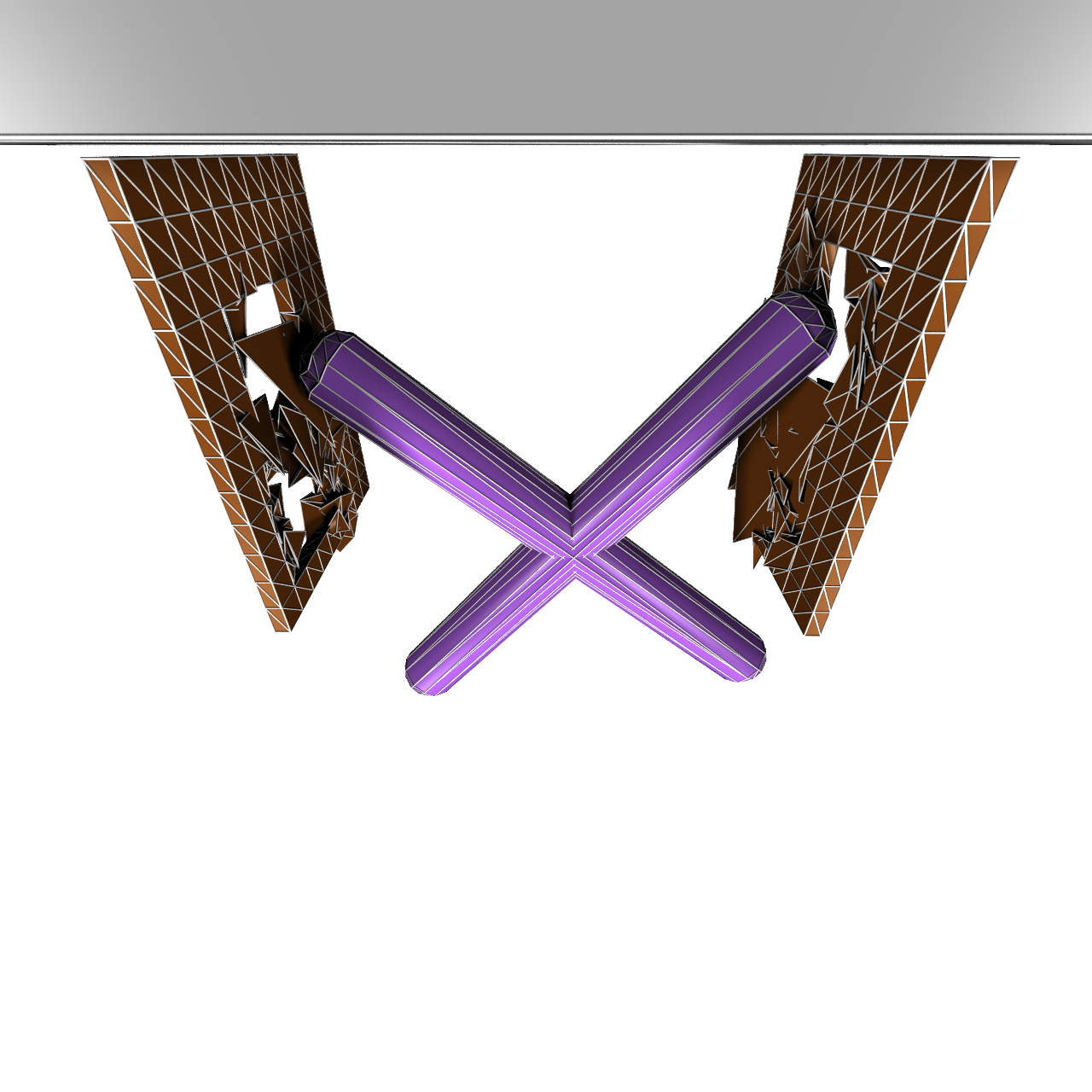}&
\includegraphics[trim=0 15cm 0 0,clip,width=.24\textwidth,frame]{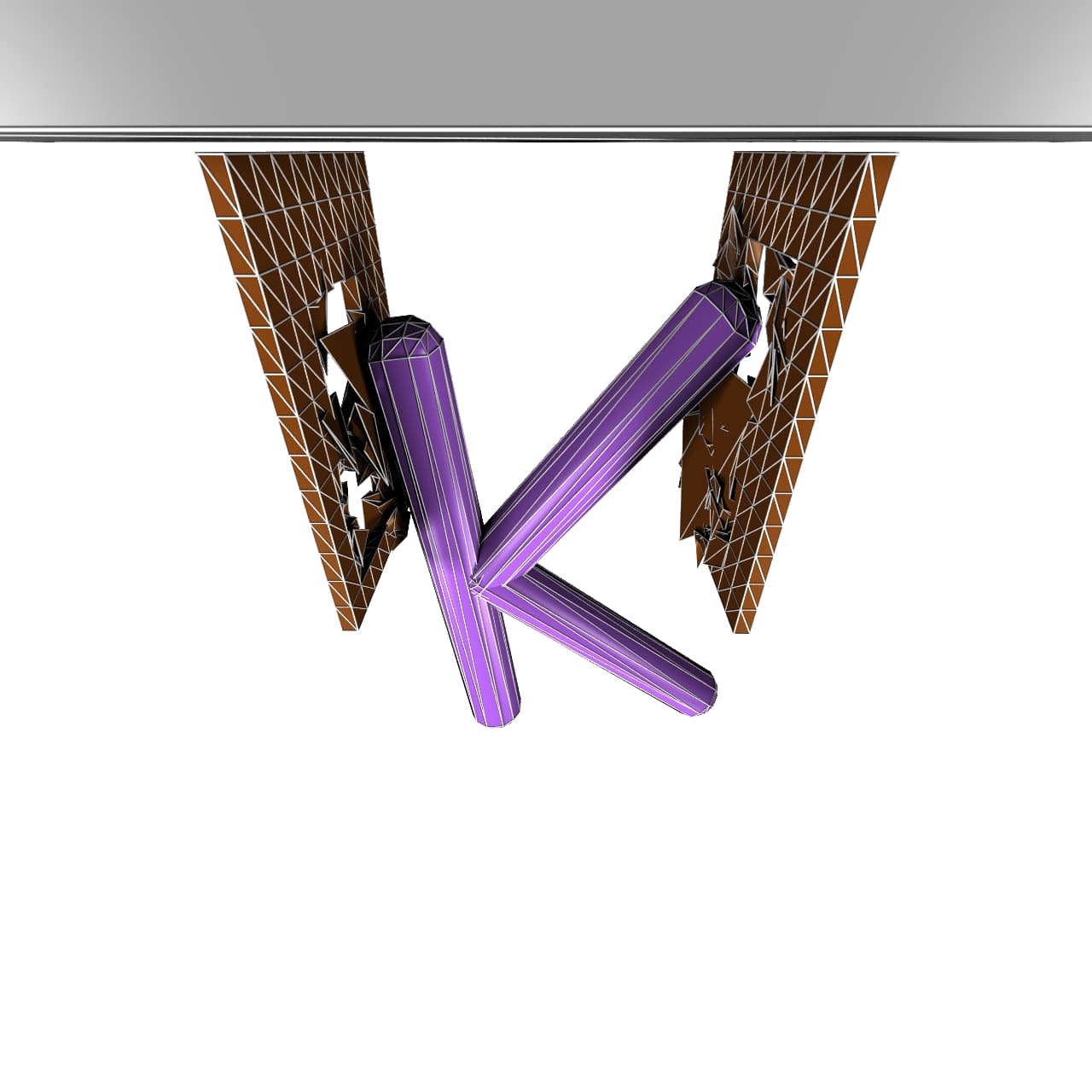}\\
\includegraphics[trim=0 15cm 0 0,clip,width=.24\textwidth,frame]{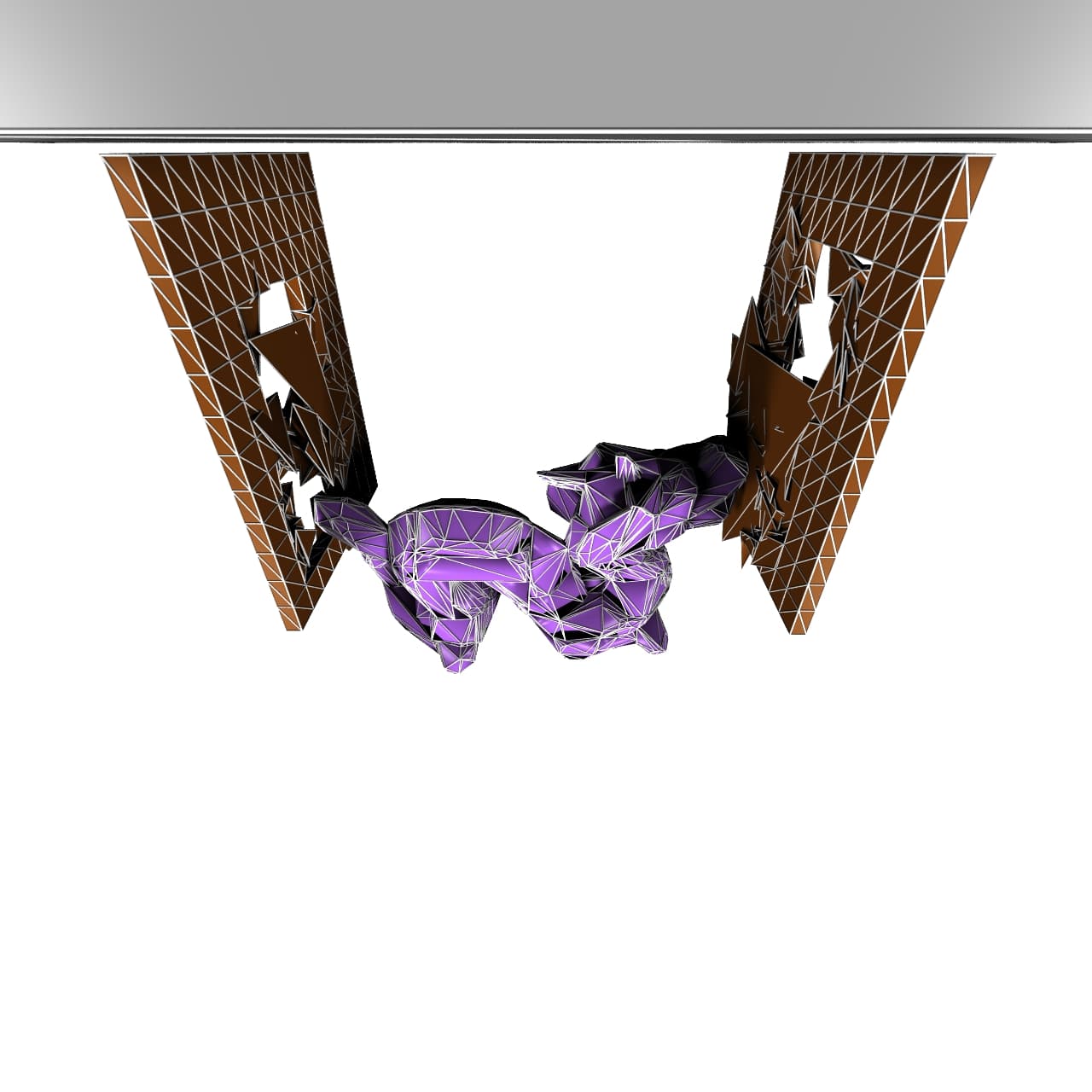}&
\includegraphics[trim=0 15cm 0 0,clip,width=.24\textwidth,frame]{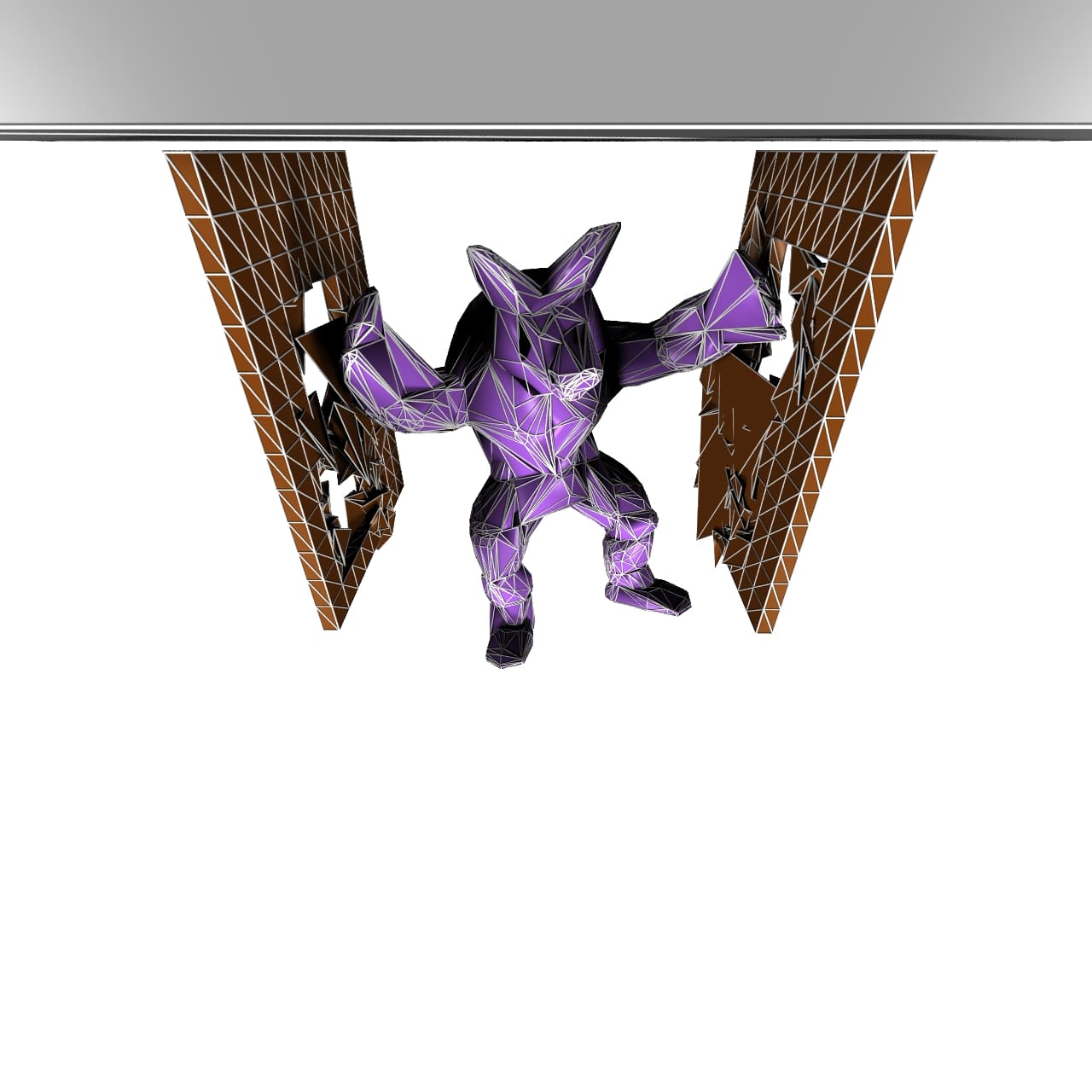}
\end{tabular}
\caption{The results of successfully grasping 4 different objects with a parallel-jaw gripper optimized using our approach, inducing both topological and geometric shape change.}
\label{fig:two_gripper}
\end{figure}
\begin{figure}[ht]
\centering
\setlength{\tabcolsep}{0px}
\begin{tabular}{cc}
\includegraphics[trim=0 4cm 0 0,clip,width=.23\textwidth]{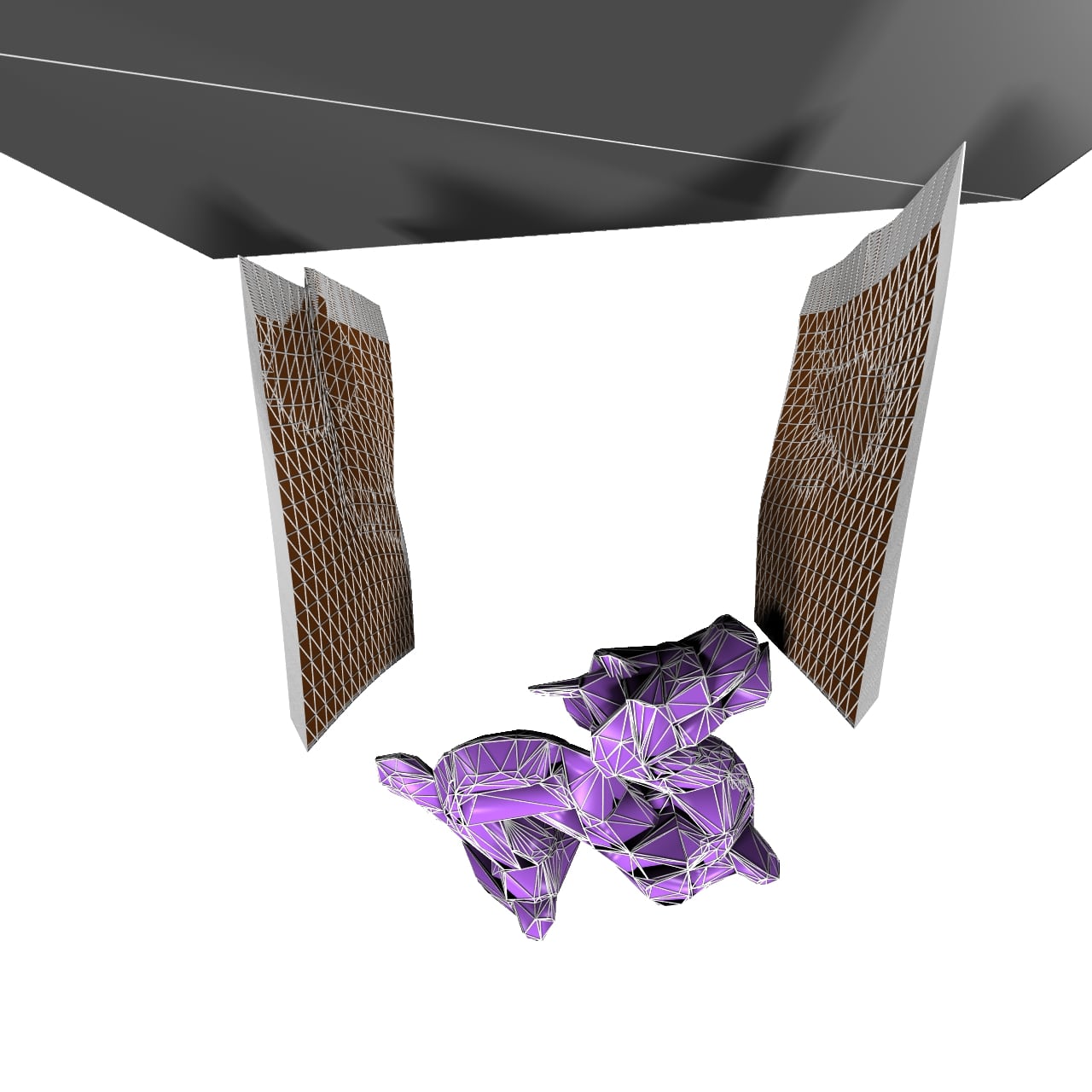}&
\includegraphics[trim=0 4cm 0 0,clip,width=.23\textwidth]{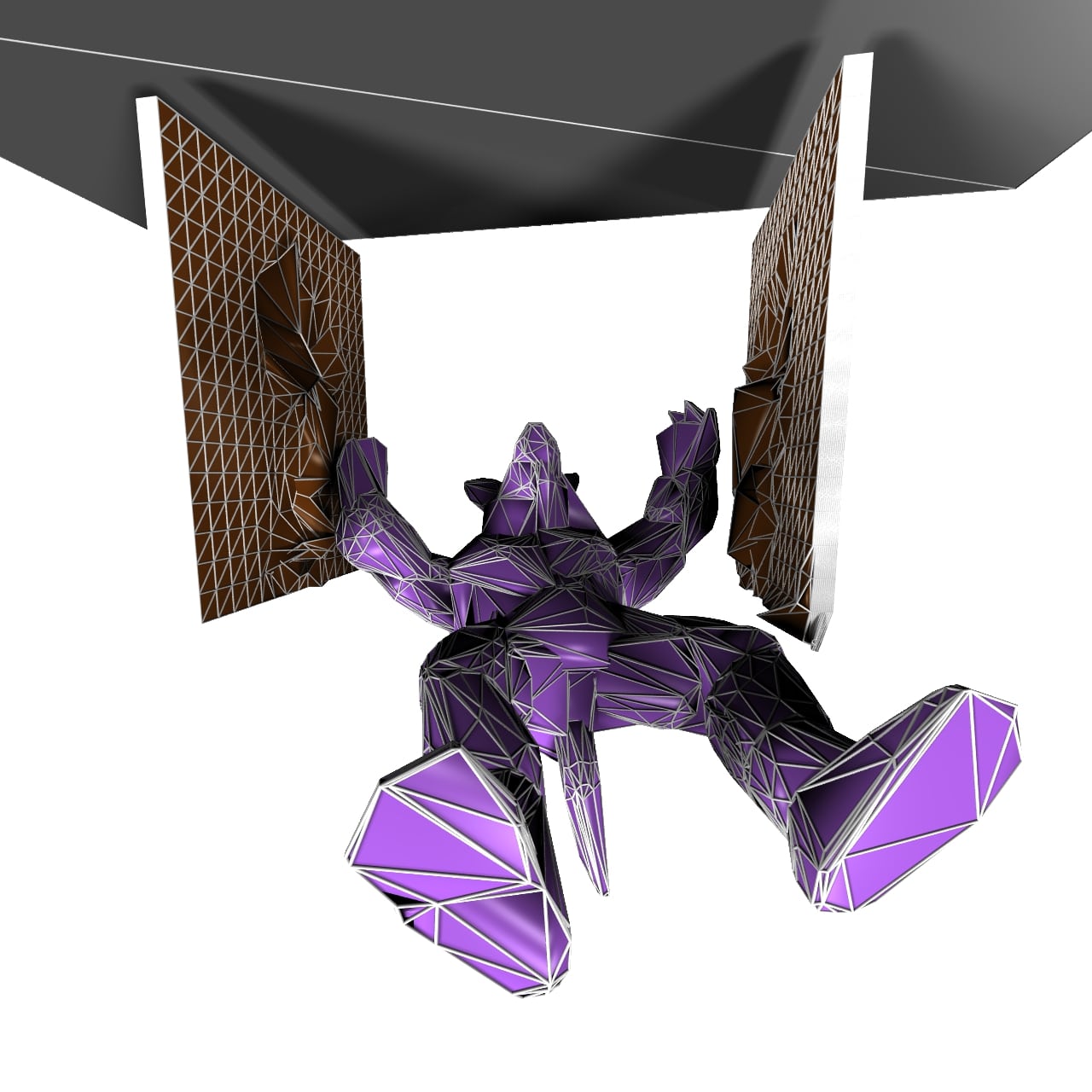}\\
\end{tabular}
\caption{The grippers optimized with mesh-based methods using coarse cages (left) and fine cages (right) both fail to grasp certain objects.}
\label{fig:two_gripper_fail}
\end{figure}
\subsection{Benchmark: Universal Parallel-Jaw Gripper}
In~\prettyref{fig:two_gripper}, we use our method to optimize the shape of a parallel-jaw gripper to grasp a set of four different objects and reach the target position located in random orientation. The gripper has 3 links and 5 degrees of freedom, and for this benchmark, we use a preset parameterization of the controller. Specifically, we fix the pose of the closed gripper, denoted as $\theta_\star$, and we set all the position target $\theta_\star^t=\theta_\star$ throughout the trajectory. Our design goal is to have the gripper firmly grasp each objects by penalizing any relative motions between the ball and the gripper: $L(\theta^H)=\|COM_{\text{object}}(\theta^H)-COM_{\text{gripper}}(\theta^H)\|^2$. In each iteration, the gripper sequentially grasps 4 objects and reaches a randomly assigned target position, and we optimize these 4 trajectories collectively. Due to a low friction coefficient, none of these objects can be grasp successfully using initial gripper shape. We run iterations of Adam with a horizon of $H=180$ for each trajectory and the learning rate $\bar{\alpha_d}=5e^{-3}$. When using our approach, we decompose each jaw into 156 convex blocks, and we find our optimization leads to both geometry and topology changes of the gripper, enabling it to grasp all objects successfully. Meanwhile, mesh-based methods using neither coarse cages nor fine cages can grasp all these objects, the failure cases are shown in~\prettyref{fig:two_gripper_fail} and the convergence history is summarized in~\prettyref{fig:parallel_jaw_history}.

\subsection{Benchmark: Three-Fingered Hand}
In~\prettyref{fig:three_gripper}, We optimized the shape of a three-fingered robot gripper to grasp 3 different objects and reach the target position located in any orientation. The gripper has 6 links and 6 degrees of freedom. Using V-HACD, we represent each link using 4 convex polyhedrons. In each iteration, the gripper sequentially grasps 3 objects and reaches a randomly assigned target position, while we optimize these 3 trajectories collectively. We run iterations of Adam with the horizon $H=200$ for each trajectory and the learning rate $\bar{\alpha_d}=1e^{-3}$. 
\begin{wrapfigure}{r}{0.24\textwidth}
\centering
\includegraphics[width=.23\textwidth]{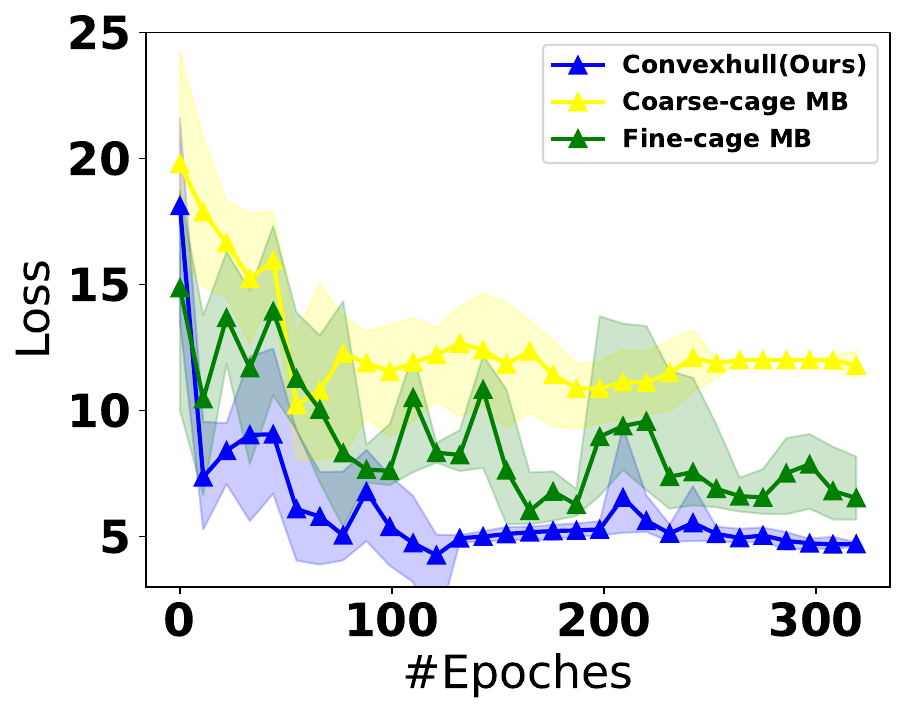}
\caption{The convergence history of optimizing the parallel-jaw gripper using various methods.}
\label{fig:parallel_jaw_history}
\end{wrapfigure}
Specifically, we fix the pose of the closed gripper, denoted as $\theta_\star$, and we set all the position target $\theta_\star^t=\theta_\star$ throughout the trajectory with PD gains being $k_p=3e^{2}$ and $k_d=3e^{0}$. For each target position, we then move the base of the gripper to that target position, so there are no control parameters. Our design goal is to have the gripper firmly grasp each objects by penalizing any relative motions between the ball and the gripper: $L(\theta^H)=\|COM_{\text{ball}}(\theta^H)-COM_{\text{gripper}}(\theta^H)\|^2$. We compare our approach with mesh-based method with fine cages. The convergence history of this benchmark is summarized in~\prettyref{fig:three_gripper}. We found that both methods performed well in this benchmark and our approach takes around $500$ iterations to converge, while mesh-based method takes $400$ iterations. Again, this is not surprising because the optimized gripper shape does not require topological changes. However, the optimized shape using mesh-based method has some thin structures that are less suitable for manufacturing, our method creates bulky lumps on the robot links that are more manufacture-friendly.

\begin{figure}[ht]
\centering
\setlength{\tabcolsep}{0px}
\begin{tabular}{cc}
\includegraphics[trim=0 6cm 0 0,clip,height=.19\textwidth]{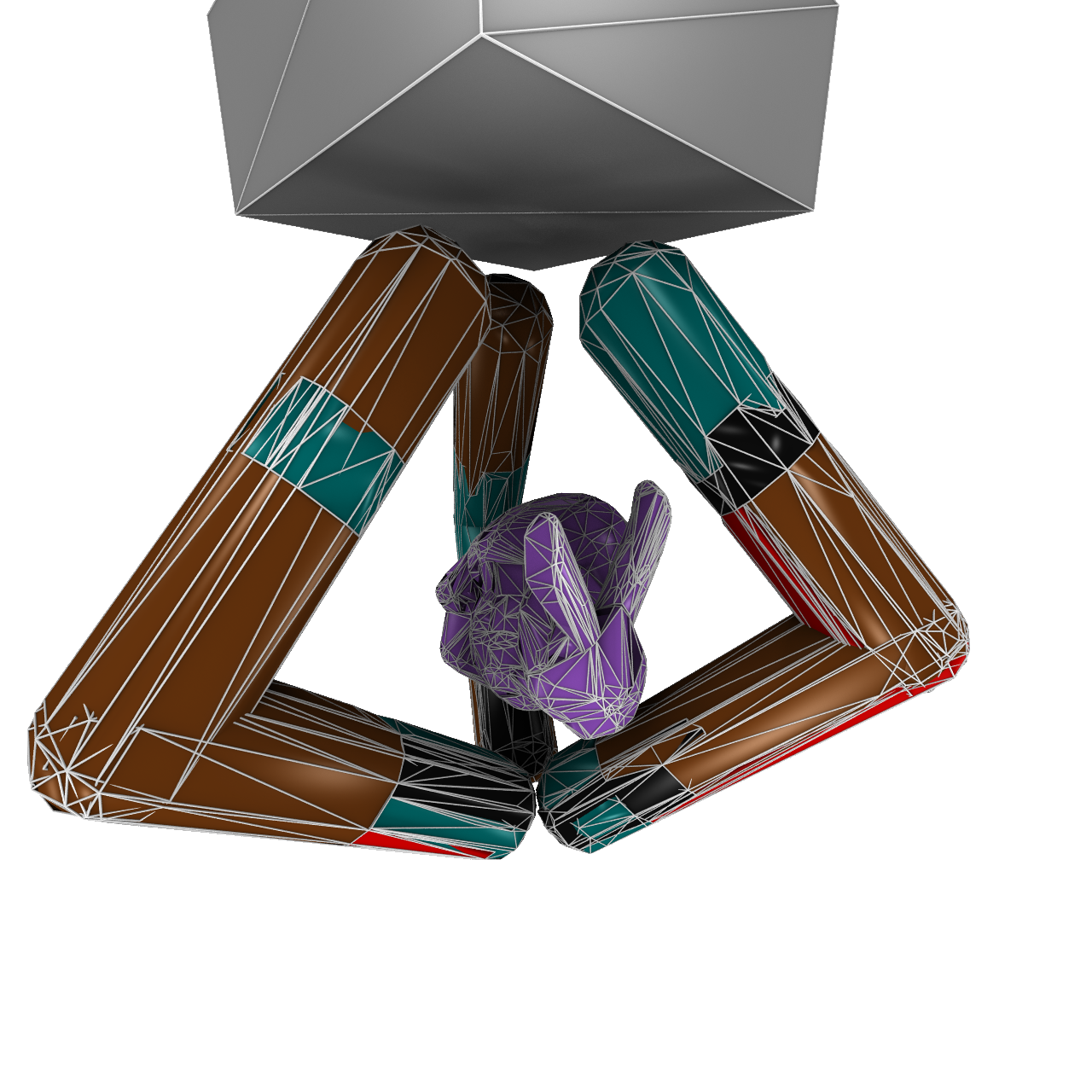}&
\includegraphics[height=.19\textwidth]{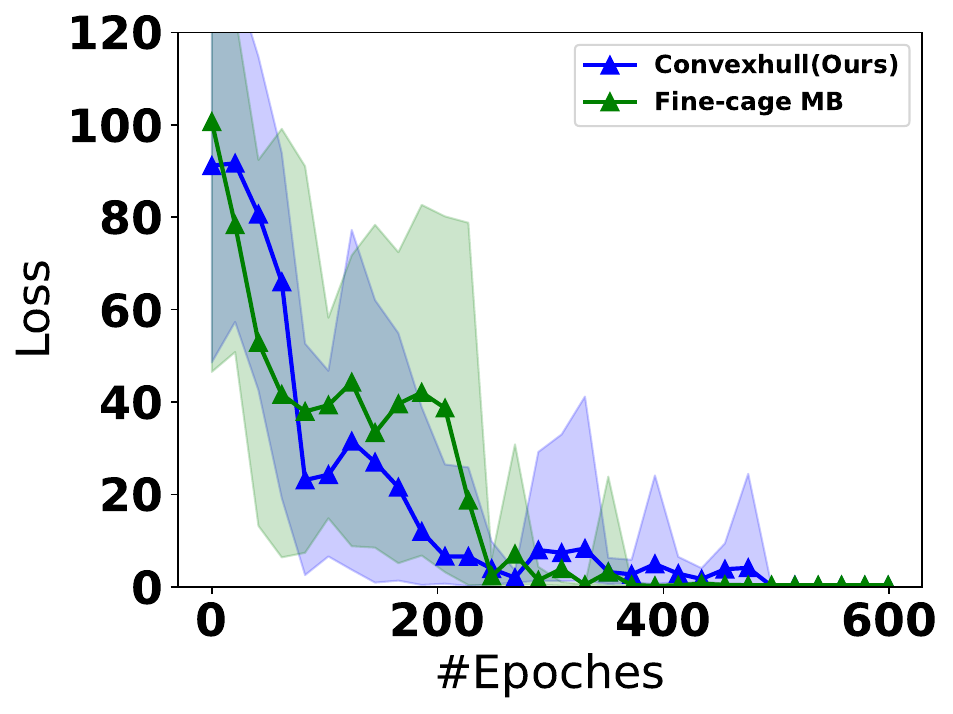}\\
\includegraphics[trim=0 6cm 0 0,clip,height=.19\textwidth]{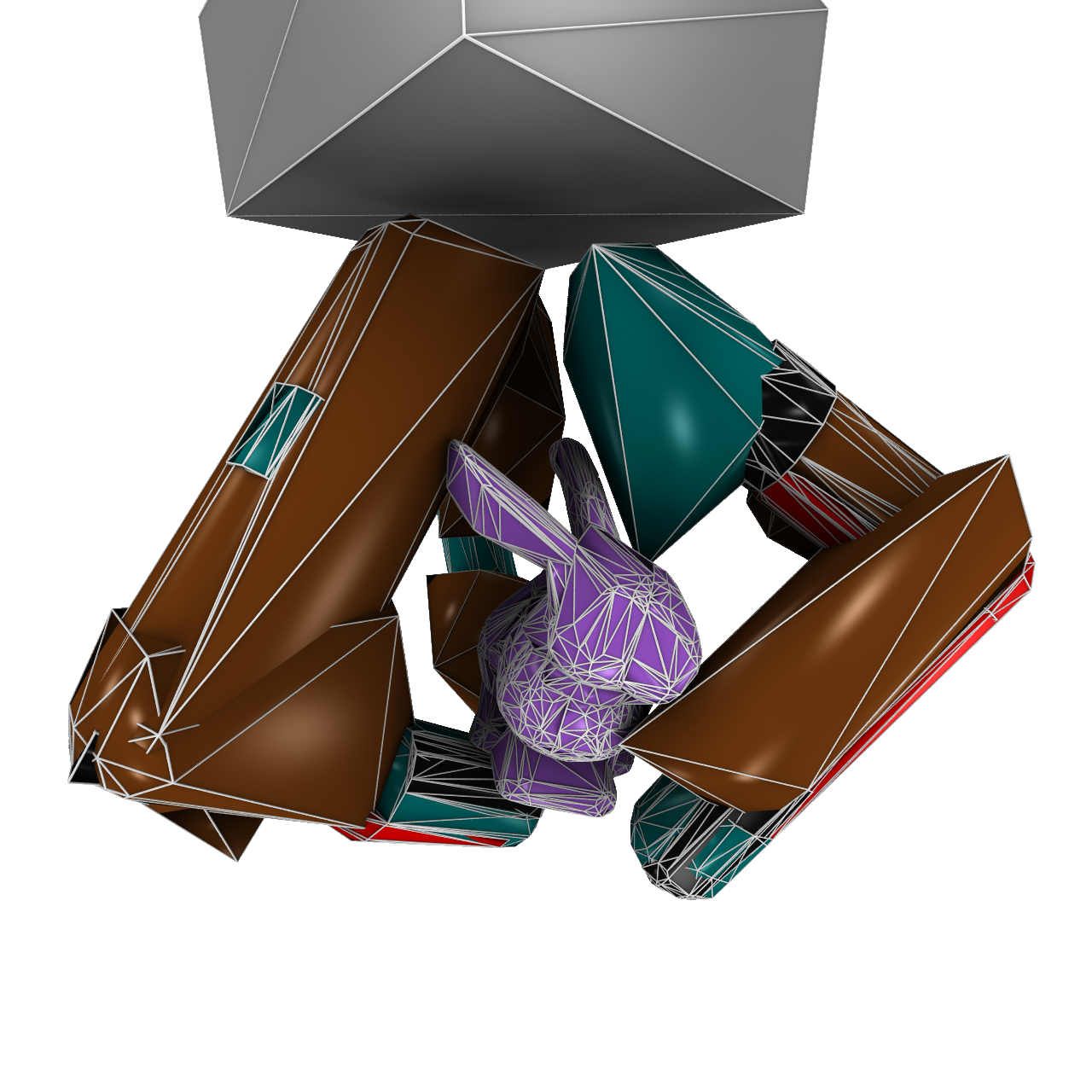}&
\includegraphics[trim=0 6cm 0 0,clip,height=.19\textwidth]{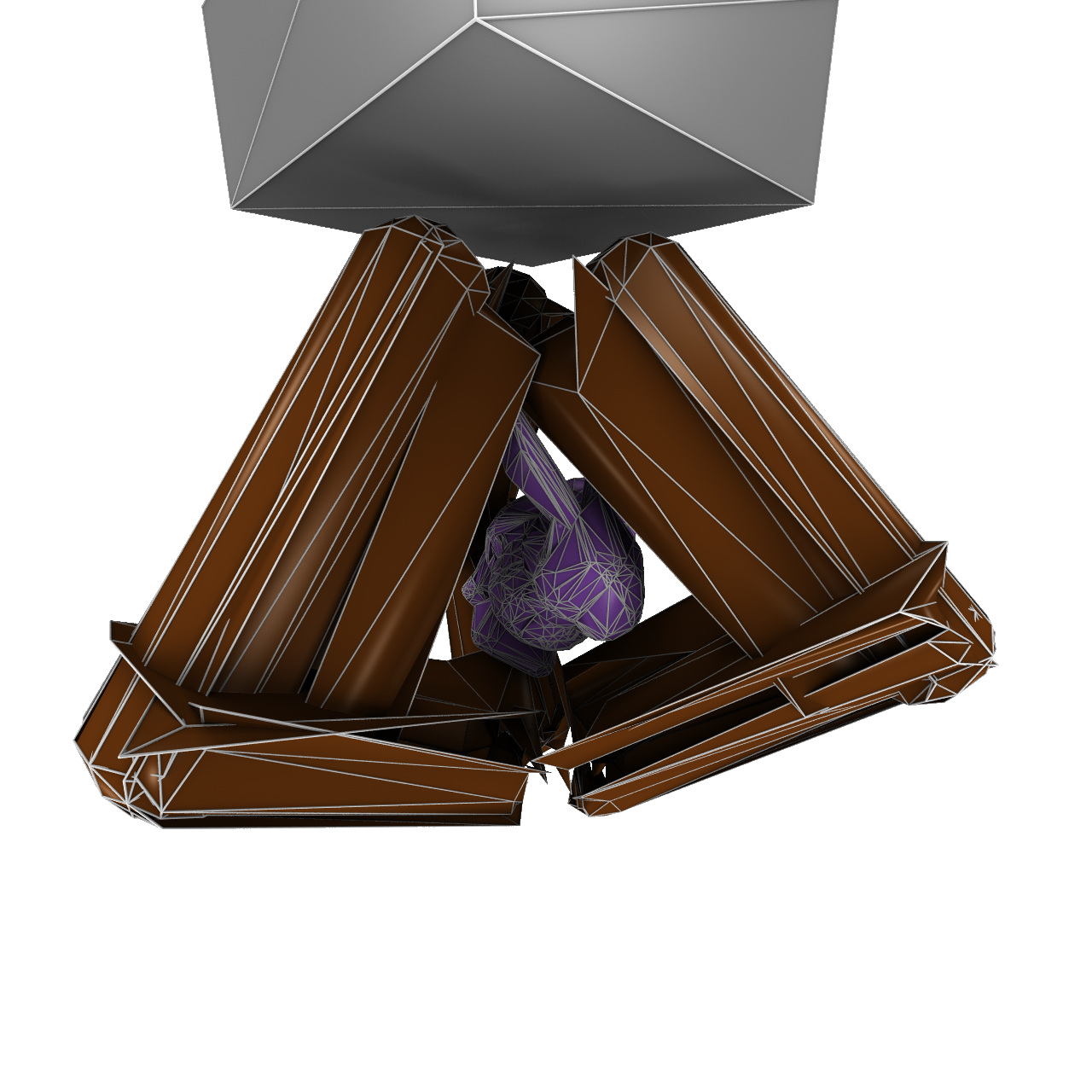}
\end{tabular}
\caption{We optimize the three-fingered gripper to grasp three objects. We visualize the optimial gripper shape (top left), the optimized shape using our method (bottom left) and mesh-based method (bottom right). The convergence history of both methods are illustrated in the top right corner.}
\label{fig:three_gripper}
\end{figure}

%% file: conclusion.tex
\section{Conclusion \& Discussion}
We introduce SDRS, a novel robot simulator that features provable differentiability with respect to state, control, and robot shape parameters, accommodating large geometric and topological changes of robot link shapes. 
There are several avenues for enhancing our differentiable simulator. Firstly, our current contact potential energy model, which requires solving a small optimization problem at each evaluation, leads to relatively high computational demands compared to mesh-based contact potential models, even with the use of warm-start initial guesses. In the future, we plan to leverage parallel processors, such as GPUs, to accelerate this process. Secondly, our simulator does not guarantee collision-free states throughout continuous articulated motions. While we have not encountered this issue in our experiments, future research incorporating CCD for articulated motions will be crucial for ensuring robust simulations. Further, to circumvent the need for differentiating integrals over convex hulls in $I_{ij}$, we currently assume that the robot's mass is concentrated at the vertices of the convex hulls, which may lead to large discretization error. Moving forward, we plan to develop more precise formulas for the exact derivatives of $I_{ij}$. Lastly, our method relies solely on the gradient-based optimization and cannot design discrete robot parameters such as joint structure, which requires sampling-based optimization algorithms as done in~\cite{yuan2022transformact}.

%% file: appendix/contact.tex
\section{\label{sec:ContactProof}Well-defined Contact Potential}
In the analysis below, we always assume that $s\in(0,1)$. Under very mild assumptions, we show that our contact potential $U_c^{ij,i'j'}(d,\theta)$ is a well-defined, twice-differentiable function of its parameters $\theta$ and $d$. We further show that the potential induces opposing forces on the two convex hulls $H_{ij}$ and $H_{i'j'}$. The main complication comes from the fact that  $U_c^{ij,i'j'}$ relies on an internal optimization with respect to the separating plane, let us denote its set of locally optimal solution as:
\begin{align*}
p_\star=\TWO{n_\star}{n_{0\star}}\in
\argmin{p}\;U_{c\star}^{ij,i'j'}(d,\theta,p).
\end{align*}
We begin by showing the following basic facts:
\begin{lemma}
\label{lem:contactBasic}
i) For any $U_c^{ij,i'j'}<\infty$, $\|n\|\in(0,1)$; ii) $U_c^{ij,i'j'}$ is a well-defined function of $d,\theta$; iii) If $\dist(H_{ij},H_{i'j'})\geq2s/(1-s)$, then $U_c^{ij,i'j'}=0$; iv) $\lim_{\dist(H_{ij},H_{i'j'})\to0}U_c^{ij,i'j'}=\infty$.
\end{lemma}
\begin{proof}
i) If $\|n\|\geq1$, then the first term $P_s(1-\|n\|)$ in $U_{c\star}^{ij,i'j'}$ is undefined. If $\|n\|=0$, then the remaining terms in $U_{c\star}^{ij,i'j'}$ reads:
\begin{align*}
\sum_{m=1}^M \left[P_s(-o)+P_s(o)\right].
\end{align*}
At least one of these terms are undefined, so $\|n\|\in(0,1)$.

ii) By construction, each term in $U_{c\star}^{ij,i'j'}$ is a convex function in $p$, so it has only one (global) minimum, proving that $U_c^{ij,i'j'}$ is single-valued, i.e., a well-defined function of $d,\theta$.

iii) In this case, we must have two points $x_{ij}\in H_{ij}$ and $x_{i'j'}\in H_{i'j'}$ realizing the closest distance, i.e., $\dist(x_{ij},x_{i'j'})\geq2s/(1-s)$. Now, let us set $p$ as the scaled, middle separating plane between $x_{ij}$ and $x_{i'j'}$ by setting:
\begin{align*}
n=\frac{x_{i'j'}-x_{ij}}{\|x_{i'j'}+x_{ij}\|}(1-s)\text{ and }
o=-n^T\frac{x_{i'j'}+x_{ij}}{2}(1-s).
\end{align*}
Plugging this solution into $U_{c\star}^{ij,i'j'}$ and we have $U_{c\star}^{ij,i'j'}=0$ by the local support of $P_s$. But by ii), $U_{c\star}^{ij,i'j'}(d,\theta,p)$ is convex in $p$ and its minimal value can only be even smaller. But $U_{c\star}^{ij,i'j'}\geq0$ by construction, so we conclude that $U_c^{ij,i'j'}=0$ as desired.

iv) Given an arbitrary large number $M>0$, there must be some $s_0>0$ such that $P_s(x)>M$ on $(0,s_0)$. We consider any two convex polyhedrons such that $\dist(H_{ij},H_{i'j'})<s_0/(1-s_0)$. In this case, for any separating plane $p$, there must be some point, assumed to be $[T_{i'}(d,\theta)x_{i'j'}^m]$ without a loss of generality, such that:
\begin{align*}
0<\frac{p^TT_{i'}(d,\theta)x_{i'j'}^m}{\|n\|}<\frac{s_0}{1-s_0},
\end{align*}
where we divide by $\|n\|$ to normalize the separating plane, leading to a standard point-to-plane distance. Now consider two cases. Case I: If $\|n\|>1-s_0$, then the first term $P_s(1-\|n\|)>M$, leading to $U_c^{ij,i'j'}>M$. Case II: If $\|n\|\leq1-s_0$, then we have:
\begin{align*}
0<p^TT_{i'}(d,\theta)x_{i'j'}^m<s_0\Rightarrow 
P_s(p^TT_{i'}(d,\theta)x_{i'j'}^m)>M,
\end{align*}
again leading to $U_c^{ij,i'j'}>M$, from which the result follows.
\end{proof}
The above result already allows us to use BVH to quickly prune non-intersecting convex polyhedrons when their distance is larger than $2s/(1-s)$.

\subsection{Local Differentiability}
To proceed, we first prove differentiable properties when $U_c^{ij,i'j'}>0$ and then extend these properties to all cases.
\begin{lemma}
\label{lem:contactPositive}
If $U_c^{ij,i'j'}>0$, then i) at least one of the terms of form $P_s(-p_\star^TT_i(d,\theta)x_{ij}^m)$ or $P_s(p_\star^TT_{i'}(d,\theta)x_{i'j'}^m)$ is positive and ii) $P_s(1-\|n_\star\|)>0$.
\end{lemma}
\begin{proof}
i) If all the terms of form $P_s(-p_\star^TT_i(d,\theta)x_{ij}^m)=0$ and $P_s(p_\star^TT_{i'}(d,\theta)x_{i'j'}^m)=0$, we must have: 
\begin{align*}
\FPP{P_s(-p_\star^TT_i(d,\theta)x_{ij}^m)}{n_\star}=
\FPP{P_s(p_\star^TT_{i'}(d,\theta)x_{i'j'}^m)}{n_\star}=0.
\end{align*}
Next, by~\prettyref{lem:contactBasic} i), we know that $\|n\|\in(0,1)$. But $U_c^{ij,i'j'}>0$ by our assumption, so we have $\|n\|\in(1-s,1)$, so:
\begin{align*}
\FPP{P_s(1-\|n_\star\|)}{n_\star}\neq0.
\end{align*}
Combined, we know that:
\begin{align*}
\FPP{U_{c\star}^{ij,i'j'}}{n_\star}\neq0,
\end{align*}
contradicting the fact that $p_\star$ is optimal.

ii) If $P_s(1-\|n_\star\|)=0$, we know by i) that some other terms are positive and the following inequality holds for all positive terms:
\begin{align*}
p_\star^T\FPP{P_s(-p_\star^TT_i(d,\theta)x_{ij}^m)}{p_\star}<0\\
p_\star^T\FPP{P_s(p_\star^TT_{i'}(d,\theta)x_{i'j'}^m)}{p_\star}<0.
\end{align*}
But since $P_s(1-\|n_\star\|)=0$, we have by definition:
\begin{align*}
p_\star^T\FPP{P_s(1-\|n_\star\|)}{p_\star}=0.
\end{align*}
Combined, we know that:
\begin{align*}
p_\star^T\FPP{U_{c\star}^{ij,i'j'}}{p_\star}<0,
\end{align*}
again contradicting the fact that $p_\star$ is optimal.
\end{proof}
\begin{lemma}
\label{lem:contactLocalDiff}
If $U_c^{ij,i'j'}>0$, then i) $U_{c\star}^{ij,i'j'}$ is a strictly convex function; and ii) $U_c^{ij,i'j'}$ is a locally differentiable function of $d,\theta$ with:
\begin{align}
\label{eq:derivative}
\FDD{U_c^{ij,i'j'}}{\TWO{d}{\theta}}=
\FPP{U_{c\star}^{ij,i'j'}(d,\theta,p_\star)}{\TWO{d}{\theta}}.
\end{align}
Finally, iii) $U_c^{ij,i'j'}$ is locally twice-differentiable function of $d,\theta$.
\end{lemma}
\begin{proof}
i) We show that the Hessian of $U_{c\star}^{ij,i'j'}$ at $p_\star$ is non-singular. To this end, we consider a vector $0\neq q\in\mathbb{R}^4$ and prove that:
\begin{align*}
q^T\FPPT{U_{c\star}^{ij,i'j'}}{p_\star}q\neq0.
\end{align*}
We establish contradiction otherwise. By~\prettyref{lem:contactBasic} ii), we know that each term in $U_{c\star}^{ij,i'j'}$ is convex and contributes a positive semi-definite Hessian matrix. We first derive the Hessian of the term $P_s(1-\|n_\star\|)$ in the following closed form:
\begin{equation}
\label{eq:hessian}
\ResizedEq{\FPPT{P_s(1-\|n_\star\|)}{n_\star}=
P_s''\frac{n_\star}{\|n_\star\|}\frac{n_\star}{\|n_\star\|}^T-
P_s'\frac{1}{\|n_\star\|}\left[I-\frac{n_\star}{\|n_\star\|}\frac{n_\star}{\|n_\star\|}^T\right].}
\end{equation}
By~\prettyref{lem:contactPositive} ii), we have:
\begin{align*}
P_s(1-\|n_\star\|)>0\text{ and } 
\FPPT{P_s(1-\|n_\star\|)}{n_\star}\succ0.
\end{align*}
Therefore, the first three elements of $q$ are zero. Since $0\neq q$, we must have the fourth component of $q$ is non-zero. But by~\prettyref{lem:contactPositive} i), we know that one of the other terms in $U_{c\star}^{ij,i'j'}$ is non-zero at $p_\star$. Without a loss of generality, we can assume:
\begin{align*}
&P_s(-p_\star^TT_i(d,\theta)x_{ij}^m)>0\text{ and }\\
&\FPPT{P_s(-p_\star^TT_i(d,\theta)x_{ij}^m)}{p_\star}=P_s''[T_i(d,\theta)x_{ij}^m][T_i(d,\theta)x_{ij}^m]^T\neq0.
\end{align*}
Note we assume homogeneous coordinates for $T_i(d,\theta)x_{ij}^m$. As a result, the fourth component of $q$ must be zero, leading to a contradiction.

ii) Due to i), we can invoke the inverse function theorem, in a local neighborhood, we can assume that $U_c^{ij,i'j'}>0$, $p_\star$ is a (single-valued) function $p_\star(d,\theta)$, and:
\begin{align}
\label{eq:planeDerivative}
\FPP{p_\star}{\TWO{d}{\theta}}=-\left[\FPPT{U_{c\star}^{ij,i'j'}}{p_\star}\right]^{-1}\FPPTT{U_{c\star}^{ij,i'j'}}{p_\star}{\TWO{d}{\theta}}.
\end{align}
Note the above derivative is well-defined because our piecewise function $P_s$ is twice-differentiable by construction. By the chain rule, we have:
\begin{align*}
\FDD{U_c^{ij,i'j'}}{\TWO{d}{\theta}}=
\FPP{U_{c\star}^{ij,i'j'}}{\TWO{d}{\theta}}+
\FPP{U_{c\star}^{ij,i'j'}}{p_\star}\FPP{p_\star}{\TWO{d}{\theta}}=
\FPP{U_{c\star}^{ij,i'j'}}{\TWO{d}{\theta}},
\end{align*}
due to the vanish of gradient at $p_\star$. iii) In the same local neighborhood, we can differentiable the above equation once more to yield:
\begin{equation}
\label{eq:derivative2}
\ResizedEq{\FDDT{U_c^{ij,i'j'}}{\TWO{d}{\theta}}=
\FDD{}{\TWO{d}{\theta}}\FPP{U_{c\star}^{ij,i'j'}}{\TWO{d}{\theta}}=
\FPPT{U_{c\star}^{ij,i'j'}}{\TWO{d}{\theta}}+
\FPPTT{U_{c\star}^{ij,i'j'}}{\TWO{d}{\theta}}{p_\star}\FPP{p_\star}{\TWO{d}{\theta}}.}
\end{equation}
The above step only relies on the second derivatives of $U_{c\star}^{ij,i'j'}$ and is thus well-defined by the construction of $P_s$.
\end{proof}

\subsection{Global Differentiability}
\begin{lemma}
\label{lem:contactGlobalDiff}
$U_c^{ij,i'j'}$ is a globally differentiable function of $d,\theta$ when $U_c^{ij,i'j'}<\infty$ with the derivative~\prettyref{eq:derivative} holds everywhere.
\end{lemma}
\begin{proof}
To extend the property from the local neighborhood, we consider three cases. Case I: If $U_c^{ij,i'j'}>0$ then differentiability follows from~\prettyref{lem:contactLocalDiff}. Case II: If $U_c^{ij,i'j'}=0$ in a local neighborhood, then differentiability follows trivially and agrees with~\prettyref{eq:derivative} by direct verification. Case III: In the last case, we have $U_c^{ij,i'j'}(d,\theta)=0$ but there is a sequence of $\TWO{d_n}{\theta_n}\to\TWO{d}{\theta}$ and $U_c^{ij,i'j'}(d_n,\theta_n)>0$. In such case, the implicit function theorem fails. To overcome this difficulty, we now define the following auxiliary function:
\begin{align*}
g_n(\beta)=U_c^{ij,i'j'}(\beta(d_n-d)+d,\beta(\theta_n-\theta)+\theta).
\end{align*}
Clearly, we have $g_n(0)=0$ and $g_n(1)>0$. We can then define $\bar{\beta}=\sup\{\beta|g_n(\beta)=0\}$. Since $U_c^{iji'j'}$ is a continuous function, we must have $\bar{\beta}\in(0,1)$. Now since $g_n>0$ on $(\bar{\beta},1)$,  $g_n'$ can be computed using~\prettyref{eq:derivative} on $(\bar{\beta},1)$. Applying intermediary value theorem on $(\bar{\beta},1)$, we can find $\beta_n\in(\bar{\beta},1)$ satisfying:
\begin{equation*}
\ResizedEq{&U_c^{ij,i'j'}(d_n,\theta_n)=g_n(1)=g_n'(\beta_n)(1-\bar{\beta})\\
=&\TWO{(d_n-d)^T}{(\theta_n-\theta)^T}\FDD{U_c^{ij,i'j'}}{\TWO{d}{\theta}}
\Big|_{\TWO{\beta_n(d_n-d)+d}{\beta_n(\theta_n-\theta)+\theta}}(1-\bar{\beta}).}
\end{equation*}
Using this result and we derive the following limit:
\begin{equation*}
\ResizedEq{0\leq&\liminf_{n\to\infty}\frac{U_c^{ij,i'j'}(d_n,\theta_n)-U_c^{ij,i'j'}(d,\theta)}{\|\TWO{d_n-d}{\theta_n-\theta}\|}\\
\leq&\limsup_{n\to\infty}\frac{U_c^{ij,i'j'}(d_n,\theta_n)-U_c^{ij,i'j'}(d,\theta)}{\|\TWO{d_n-d}{\theta_n-\theta}\|}\\
\leq&\limsup_{n\to\infty}\left\|\FDD{U_c^{ij,i'j'}}{\TWO{d}{\theta}}
\Big|_{\TWO{\beta_n(d_n-d)+d}{\beta_n(\theta_n-\theta)+\theta}}\right\|=0,}
\end{equation*}
where the last equality is derived by the fact that the derivative at $\TWO{\beta_n(d_n-d)+d}{\beta_n(\theta_n-\theta)+\theta}$ follows from~\prettyref{eq:derivative}, which tends to zero when $U_c^{ij,i'j'}$ approaches zero. Indeed, when $U_c^{ij,i'j'}$ approaches zero, all the terms approach zero consistently, so all the first and second derivatives, i.e. $P_s'$ and $P_s''$, also tend to zero. With other terms being bounded, we conclude that \prettyref{eq:derivative} tends to zero. This result shows that any directional derivative tends to zero in the last case and agrees with~\prettyref{eq:derivative}, thus all is proved.
\end{proof}

\subsection{Global Twice-Differentiability}
Next, we prove twice-differentiability using a similar argument as in~\prettyref{lem:contactGlobalDiff}, but we need the following auxiliary result:
\begin{lemma}
\label{lem:contactHessianVanish}
Suppose there is a sequence $\TWO{d_n}{\theta_n}\to\TWO{d}{\theta}$ such that $U_c^{ij,i'j'}(d_n,\theta_n)>0$ and $U_c^{ij,i'j'}(d,\theta)=0$, then:
\begin{align*}
\lim_{n\to\infty}\FDDT{U_c^{ij,i'j'}}{\TWO{d}{\theta}}\Big|_{\TWO{d_n}{\theta_n}}=0.
\end{align*}
\end{lemma}
\begin{proof}
There are two terms in~\prettyref{eq:derivative2} and for the first term:
\begin{align}
\lim_{n\to\infty}\FPPT{U_{c\star}^{ij,i'j'}}{\TWO{d_n}{\theta_n}}=0,
\end{align}
due to the property of $P_s$. Therefore, our main goal is to show the following second term in~\prettyref{eq:derivative2} also tend to zero, which is largely non-trivial:
\begin{align}
\label{eq:hessianTerm2}
-\FPPTT{U_{c\star}^{ij,i'j'}}{\TWO{d_n}{\theta_n}}{p_\star}
\left[\FPPT{U_{c\star}^{ij,i'j'}}{p_\star}\right]^{-1}
\FPPTT{U_{c\star}^{ij,i'j'}}{p_\star}{\TWO{d_n}{\theta_n}}.
\end{align}
To begin with, we claim the following property, which can be directly verified for our specific $P_s(x)=\max(0,(x-s)^4/x^5)$:
\begin{align}
\label{eq:penaltyLimit}
\lim_{x\to s}P_s'/P_s''=0\text{ and }
\lim_{x\to s}[P_s'']^{1.5}/P_s'=6\sqrt{3}/s^{2.5}.
\end{align}
Next, we derive the following identity by the optimality of $p_\star$:
\begin{equation}
\label{eq:optimalityCond}
\ResizedEq{0=&p_\star^T\FPP{U_{c\star}^{ij,i'j'}}{p_\star}=-P_s'(1-\|n_\star\|)\|n_\star\|+\\
&\sum_{m=1}^M\left[\bar{P}_s'(-p_\star^TT_i(d_n,\theta_n)x_{ij}^m)+\bar{P}_s'(p_\star^TT_{i'}(d_n,\theta_n)x_{i'j'}^m)\right].}
\end{equation}
where we define $\bar{P}_s'(x)=xP_s'(x)$. Due to~\prettyref{eq:penaltyLimit}, as $\lim_{n\to\infty}U_c^{ij,i'j'}(d_n,\theta_n)=0$, we can choose sufficiently large $n$ such that: $-p_\star^TT_i(d_n,\theta_n)x_{ij}^m>s/2$ for all $x_{ij}^m$ and $p_\star^TT_{i'}(d_n,\theta_n)x_{i'j'}^m>s/2$ for all $x_{i'j'}^m$, $1>\|n_\star\|>1-s$. Finally, because all the terms in~\prettyref{eq:optimalityCond} are positive, we have the following inequality for $x_{ij}^m$ and $x_{i'j'}^m$ chosen as above:
\begin{equation}
\label{eq:limitDerivativePInequality}
\ResizedEq{
&-P_s'(1-\|n_\star\|)\geq-P_s'(1-\|n_\star\|)\|n_\star\|\\
\geq&-[-p_\star^TT_i(d_n,\theta_n)x_{ij}^m]P_s'(-p_\star^TT_i(d_n,\theta_n)x_{ij}^m)\\
\geq&-\frac{s}{2}P_s'(-p_\star^TT_i(d_n,\theta_n)x_{ij}^m)\\
&-P_s'(1-\|n_\star\|)\geq-P_s'(1-\|n_\star\|)\|n_\star\|\\
\geq&-[p_\star^TT_{i'}(d_n,\theta_n)x_{i'j'}^m]P_s'(p_\star^TT_{i'}(d_n,\theta_n)x_{i'j'}^m)\\
\geq&-\frac{s}{2}P_s'(p_\star^TT_{i'}(d_n,\theta_n)x_{i'j'}^m).}
\end{equation}
Again due to~\prettyref{eq:penaltyLimit}, we can choose sufficiently large $n$ such that $P_s''(1-\|n_\star\|)\geq-P_s'(1-\|n_\star\|)$ so that we have the following estimate due to~\prettyref{eq:hessian}:
\begin{equation}
\label{eq:hessianPartA}
\ResizedEq{\FPPT{P_s(1-\|n_\star\|)}{n_\star}=&
P_s''\frac{n_\star}{\|n_\star\|}\frac{n_\star}{\|n_\star\|}^T-
P_s'\frac{1}{\|n_\star\|}\left[I-\frac{n_\star}{\|n_\star\|}\frac{n_\star}{\|n_\star\|}^T\right]\\
\succeq&-P_s'\frac{n_\star}{\|n_\star\|}\frac{n_\star}{\|n_\star\|}^T-
P_s'\left[I-\frac{n_\star}{\|n_\star\|}\frac{n_\star}{\|n_\star\|}^T\right]\\
=&P_s'(1-\|n_\star\|)I.}
\end{equation}
Finally, among all the terms of form $-\bar{P}_s'(-p_\star^TT_i(d_n,\theta_n)x_{ij}^m)$ and $-\bar{P}_s'(p_\star^TT_{i'}(d_n,\theta_n)x_{i'j'}^m)$ in~\prettyref{eq:optimalityCond}, we can assume some term $-\bar{P}_s'(-p_\star^TT_i(d_n,\theta_n)x_{ij}^{m\star})$ achieves the maximal value. (It is also possible that some term $-\bar{P}_s'(p_\star^TT_{i'}(d_n,\theta_n)x_{i'j'}^{m\star})$ on the other convex hull achieves the maximal value, and the case is symmetric.) Once again due to~\prettyref{eq:penaltyLimit} we can choose sufficiently large $n$ such that $P_s''(-p_\star^TT_i(d_n,\theta_n)x_{ij}^{m\star})\geq-P_s'(-p_\star^TT_i(d_n,\theta_n)x_{ij}^{m\star})$. For such maximal term, we have the following estimate:
\begin{equation}
\label{eq:hessianPartB}
\ResizedEq{&2MsP_s''(-p_\star^TT_i(d_n,\theta_n)x_{ij}^{m\star})\\
\geq&-2MsP_s'(-p_\star^TT_i(d_n,\theta_n)x_{ij}^{m\star})\\
\geq&-2M\bar{P}_s'(-p_\star^TT_i(d_n,\theta_n)x_{ij}^{m\star})\\
\geq&-P_s'(1-\|n_\star\|)\|n_\star\|\geq-P_s'(1-\|n_\star\|)(1-s).}
\end{equation}
Combining~\prettyref{eq:hessianPartA} and~\prettyref{eq:hessianPartB}, we have the following estimate of the norm of Hessian:
\begin{align*}
\FPPT{U_{c\star}^{ij,i'j'}}{p_\star}\succeq&
\MTT{\FPPT{P_s(1-\|n_\star\|)}{p_\star}}{}{}{P_s''(-p_\star^TT_i(d_n,\theta_n)x_{ij}^{m\star})}\\
\succeq&-P_s'(1-\|n_\star\|)\min\left[1,\frac{1-s}{2Ms}\right]I,
\end{align*}
 which implies that:
\begin{align}
\label{eq:hessianNorm}
\left\|\left[\FPPT{U_{c\star}^{ij,i'j'}}{p_\star}\right]^{-1}\right\|\leq\frac{1}{-P_s'(1-\|n_\star\|)\min\left[1,\frac{2Ms}{1-s}\right]}.
\end{align}
We are now ready to bound~\prettyref{eq:hessianTerm2} by noting that the mixed derivative can be expanded as:
\begin{align*}
&\FPPTT{U_{c\star}^{ij,i'j'}}{p_\star}{\TWO{d_n}{\theta_n}}=\\
&\sum_{m=1}^MP_s'(-p_\star^TT_i(d_n,\theta_n)x_{ij}^m)\left[-\FPP{T_ix_{ij}^m}{\TWO{d_n}{\theta_n}}\right]+\\
&\sum_{m=1}^MP_s''(-p_\star^TT_i(d_n,\theta_n)x_{ij}^m)\left[-T_ix_{ij}^m\right]\FPP{\left[-p_\star^TT_i(d_n,\theta_n)x_{ij}^m\right]}{\TWO{d_n}{\theta_n}}^T+\\
&\sum_{m=1}^MP_s'(p_\star^TT_{i'}(d_n,\theta_n)x_{i'j'}^m)\left[\FPP{T_{i'}x_{i'j'}^m}{\TWO{d_n}{\theta_n}}\right]+\\
&\sum_{m=1}^MP_s''(p_\star^TT_{i'}(d_n,\theta_n)x_{i'j'}^m)\left[T_{i'}x_{i'j'}^m\right]\FPP{\left[p_\star^TT_{i'}(d_n,\theta_n)x_{i'j'}^m\right]}{\TWO{d_n}{\theta_n}}^T.
\end{align*}
Combining~\prettyref{eq:hessianNorm} and with other terms being bounded, we see that~\prettyref{eq:hessianTerm2} is a summation of following terms:
\begin{equation*}
\ResizedEq{I\triangleq&O(P_s'(-p_\star^TT_i(d_n,\theta_n)x_{ij}^m)
P_s'(-p_\star^TT_{i'}(d_n,\theta_n)x_{i'j'}^{m'})P_s'(1-\|n_\star\|)^{-1})\\
II\triangleq&O(P_s''(-p_\star^TT_i(d_n,\theta_n)x_{ij}^m)
P_s'(-p_\star^TT_{i'}(d_n,\theta_n)x_{i'j'}^{m'})P_s'(1-\|n_\star\|)^{-1})\\
III\triangleq&O(P_s''(-p_\star^TT_i(d_n,\theta_n)x_{ij}^m)
P_s''(-p_\star^TT_{i'}(d_n,\theta_n)x_{i'j'}^{m'})P_s'(1-\|n_\star\|)^{-1}).}
\end{equation*}
And we show that each term tends to zero, for term of type I, we use~\prettyref{eq:limitDerivativePInequality} to derive:
\begin{equation*}
\ResizedEq{I=&O(P_s'(-p_\star^TT_i(d_n,\theta_n)x_{ij}^m)
P_s'(-p_\star^TT_{i'}(d_n,\theta_n)x_{i'j'}^{m'})P_s'(1-\|n_\star\|)^{-1})\\
\leq&O(P_s'(-p_\star^TT_i(d_n,\theta_n)x_{ij}^m)
P_s'(1-\|n_\star\|)P_s'(1-\|n_\star\|)^{-1})\\
=&O(P_s'(-p_\star^TT_i(d_n,\theta_n)x_{ij}^m))\to0}
\end{equation*}
For term of type II, we use~\prettyref{eq:limitDerivativePInequality} and choose sufficiently large $n$ to have $P_s''(-p_\star^TT_i(d_n,\theta_n)x_{ij}^m)\geq-P_s'(-p_\star^TT_i(d_n,\theta_n)x_{ij}^m)$ due to~\prettyref{eq:penaltyLimit}, yielding the following estimate:
\begin{equation*}
\ResizedEq{II=&O(P_s''(-p_\star^TT_i(d_n,\theta_n)x_{ij}^m)
P_s'(-p_\star^TT_{i'}(d_n,\theta_n)x_{i'j'}^{m'})P_s'(1-\|n_\star\|)^{-1})\\
\leq&O(P_s'(-p_\star^TT_i(d_n,\theta_n)x_{ij}^m)
P_s'(1-\|n_\star\|)P_s'(1-\|n_\star\|)^{-1})\\
=&O(P_s'(-p_\star^TT_i(d_n,\theta_n)x_{ij}^m))\to0.}
\end{equation*}
For term of type III, ~\prettyref{eq:limitDerivativePInequality} and~\prettyref{eq:penaltyLimit} yields:
\begin{equation*}
\ResizedEq{III=&O(P_s''(-p_\star^TT_i(d_n,\theta_n)x_{ij}^m)
P_s''(-p_\star^TT_{i'}(d_n,\theta_n)x_{i'j'}^{m'})P_s'(1-\|n_\star\|)^{-1})\\
=&O([P_s'(-p_\star^TT_i(d_n,\theta_n)x_{ij}^m)
P_s'(-p_\star^TT_{i'}(d_n,\theta_n)x_{i'j'}^{m'})]^{2/3}P_s'(1-\|n_\star\|)^{-1})\\
\leq&O(P_s'(1-\|n_\star\|)^{1/3})\to0,}
\end{equation*}
thus all is proved.
\end{proof}

\begin{lemma}
\label{lem:contactGlobalDiff2}
$U_c^{ij,i'j'}$ is a globally twice-differentiable function of $d,\theta$ when $U_c^{ij,i'j'}<\infty$ with the second-derivative~\prettyref{eq:derivative2} i) well-defined and ii) holds everywhere.
\end{lemma}
\begin{proof}
i) \prettyref{eq:derivative2} is well-defined when $U_c^{ij,i'j'}>0$. When $U_c^{ij,i'j'}=0$, we have:
\begin{align*}
\FPPT{U_{c\star}^{ij,i'j'}}{\TWO{d}{\theta}}=0\text{ and }
\FPPTT{U_{c\star}^{ij,i'j'}}{\TWO{d}{\theta}}{p_\star}=0,
\end{align*}
and \prettyref{eq:derivative2} is well-defined by the convention $0\cdot\infty=0$.

ii) Following the same argument as~\prettyref{lem:contactGlobalDiff}, we establish twice-differentiability for case I and case II. For case III, where the inverse function theorem fails, define:
\begin{align*}
g_n^k(\beta)=\left[\FDD{U_c^{ij,i'j'}}{\TWO{d}{\theta}}\Big|_{\TWO{\beta(d_n-d)+d}{\beta(\theta_n-\theta)+\theta}}\right]^k,
\end{align*}
and invoke the intermediary value theorem to yield:
\begin{align*}
&g_n^k(\beta)=\TWO{(d_n-d)^T}{(\theta_n-\theta)^T}\\
&\left[\FDDT{U_c^{ij,i'j'}}{\TWO{d}{\theta}}\Big|_{\TWO{\beta_n(d_n-d)+d}{\beta_n(\theta_n-\theta)+\theta}}\right]^k(1-\bar{\beta}),
\end{align*}
where $[\bullet]^k$ here indicates the $k$th element or column of the gradient vector or the Hessian matrix, respectively. The limiting argument leads to:
\begin{equation*}
\ResizedEq{0\leq&\liminf_{n\to\infty}\frac{|g_n^k(1)|}{\|\TWO{d_n-d}{\theta_n-\theta}\|}
\leq\limsup_{n\to\infty}\frac{|g_n^k(1)|}{\|\TWO{d_n-d}{\theta_n-\theta}\|}\\
\leq&\limsup_{n\to\infty}\left\|\left[\FDDT{U_c^{ij,i'j'}}{\TWO{d}{\theta}}\Big|_{\TWO{\beta_n(d_n-d)+d}{\beta_n(\theta_n-\theta)+\theta}}\right]^k\right\|\\
\leq&\limsup_{n\to\infty}\left\|\FDDT{U_c^{ij,i'j'}}{\TWO{d}{\theta}}\Big|_{\TWO{\beta_n d_n}{\beta_n \theta_n}}\right\|=0,}
\end{equation*}
where the last equality is due to~\prettyref{lem:contactHessianVanish}, verifying~\prettyref{eq:derivative2} and all is proved.
\end{proof}

\subsection{Momentum Preservation}
Finally, by using the globally valid derivative formula, we are ready to show the preservation of linear and angular momentum:
\begin{lemma}
$U_c^{ij,i'j'}$ preserves i) linear and ii) angular momentum of $H_{ij}$ and $H_{i'j'}$.
\end{lemma}
\begin{proof}
i) The total negative force applied on $H_{ij}$ and $H_{i'j'}$ is:
\begin{align*}
f_{ij}\triangleq\sum_{m=1}^M\FDD{U_c^{ij,i'j'}}{[T_i(d,\theta)x_{ij}^m]_3}
=\sum_{m=1}^M-P_s'(-p_\star^TT_i(d,\theta)x_{ij}^m)n_\star\\
f_{i'j'}\triangleq\sum_{m=1}^M\FDD{U_c^{ij,i'j'}}{[T_{i'}(d,\theta)x_{i'j'}^m]_3}=\sum_{m=1}^MP_s'(p_\star^TT_{i'}(d,\theta)x_{i'j'}^m)n_\star,
\end{align*}
where we have used the derivative~\prettyref{eq:derivative}. Combined, we have:
\begin{align*}
f_{ij}+f_{i'j'}=\FPP{U_{c\star}^{ij,i'j'}}{o_\star}n_\star=0,
\end{align*}
by the optimality of $p_\star$, proving linear momentum preservation.

ii) The total negative torque applied on $H_{ij}$ and $H_{i'j'}$ is:
\begin{align*}
\tau_{ij}\triangleq&\sum_{m=1}^M[T_i(d,\theta)x_{ij}^m]_3\times\FDD{U_c^{ij,i'j'}}{[T_i(d,\theta)x_{ij}^m]_3}\\
=&\sum_{m=1}^M-P_s'(-p_\star^TT_i(d,\theta)x_{ij}^m)[T_i(d,\theta)x_{ij}^m]_3\times n_\star\\
\tau_{i'j'}\triangleq&\sum_{m=1}^M[T_{i'}(d,\theta)x_{i'j'}^m]_3\times\FDD{U_c^{ij,i'j'}}{[T_{i'}(d,\theta)x_{i'j'}^m]_3}\\
=&\sum_{m=1}^MP_s'(p_\star^TT_{i'}(d,\theta)x_{i'j'}^m)[T_{i'}(d,\theta)x_{i'j'}^m]_3\times n_\star,
\end{align*}
where we have used the derivative~\prettyref{eq:derivative}. Combined, we have:
\begin{align*}
&\tau_{ij}+\tau_{i'j'}=
\FPP{\left[U_{c\star}^{ij,i'j'}-P_s(1-\|n_s\|)\right]}{n_\star}\times n_\star\\
=&-\FPP{P_s(1-\|n_\star\|)}{n_\star}\times n_\star=0,
\end{align*}
again by the optimality of $p_\star$, proving angular momentum preservation. 
\end{proof}
We summarize our results in the following theorem:
\begin{framed}
\begin{theorem}
\label{thm:contact}
The contact potential $U_c^{ij,i'j'}$ has the following properties:
\begin{itemize}
\item $U_c^{ij,i'j'}=0$ when $\dist(H_{ij},H_{i'j'})\geq2s/(1-s)$.
\item $\lim_{\dist(H_{ij},H_{i'j'})\to0}U_c^{ij,i'j'}=\infty$.
\item $U_c^{ij,i'j'}=0$ is globally twice-differentiable in $d$ and $\theta$.
\item $U_c^{ij,i'j'}$ preserves linear and angular momentum.
\end{itemize}
\end{theorem}
\end{framed}

%% file: appendix/friction.tex
\section{\label{sec:FrictionProof}Well-defined Frictional Damping Potential}
Similar to the case with contact potential, we show that the frictional damping potential is also globally twice differentiable and preserves the linear and angular momentum. 
\begin{lemma}
i) $U_{f\star}^{ij,i'j'}$ is a convex function in $\TWO{\omega}{u}$ so $U_f^{ij,i'j'}$ is well-defined; ii) If $U_{f\star}^{ij,i'j'}\not\equiv0$, $U_{f\star}^{ij,i'j'}$ is strictly convex in $u$ and either strictly convex in $\omega$ as well or flat in $\omega$.
\end{lemma}
\begin{proof}
i) For each term in $U_{f\star}^{ij,i'j'}$, the only function related to $\TWO{\omega}{u}$ is $\sqrt{\left\|\dot{x}_{ij,\parallel}^m-\dot{p}_{ij,\parallel}^m\right\|^2+\epsilon}$, which is convex, proving well-definedness.

ii) If any $\mathcal{A}_\epsilon(f_{ij,\perp}^m)>0$, i.e. $U_{f\star}^{ij,i'j'}\not\equiv0$, the function has positive definite Hessian in $u$. If the coefficient of $\omega$, i.e.:
\begin{align*}
[B^{ij,i'j'}]^T[n_\star^t\times[T_i(d,\theta^{t})x_{ij}^m]_3]\neq0,
\end{align*}
then the Hessian in $\omega$ is also positive. If all the coefficients vanish for terms with $\mathcal{A}_\epsilon(f_{ij,\perp}^m)>0$, then $U_{f\star}^{ij,i'j'}$ is independent of $\omega$, i.e. flat.
\end{proof}
However, although the function is well-defined, the possibility that $U_{f\star}^{ij,i'j'}$ is not strictly convex in $\omega$ can significantly complicate our analysis. Fortunately, we can remove such singular situation by shifting the coordinate system along some vector in the tangent plane. For example, let us denote $B_1^{ij,i'j'}$ as the first basis (column) of the tangent projection matrix $B^{ij,i'j'}$. Let us consider the following shifted tangent velocity:
\begin{equation*}
\ResizedEq{&\dot{p}_{ij,\parallel,\alpha}^m(u,\omega)\triangleq\omega [B^{ij,i'j'}]^T[n_\star^t\times[[T_i(d,\theta^{t})x_{ij}^m]_3+B_1^{ij,i'j'}\alpha]]+u\\
&D_{ij,\alpha}^m(d,\theta^{t+1},\theta^t,u,\omega)\triangleq
\mu\Delta t\mathcal{A}_\epsilon(f_{ij,\perp}^m)
\sqrt{\left\|\dot{x}_{ij,\parallel}^m-\dot{p}_{ij,\parallel,\alpha}^m\right\|^2+\epsilon}\\
&U_{f\star,\alpha}^{ij,i'j'}\triangleq\sum_m[D_{ij,\alpha}^m+D_{i'j',\alpha}^m]\\
&U_{f,\alpha}^{ij,i'j'}=\fmin{u,\omega}\;U_{f\star,\alpha}^{ij,i'j'}(d,\theta^{t+1},\theta^t,u,\omega),}
\end{equation*}
where we denote the shifted version by the subscript $\alpha$. We establish the following result to remove the singular configuration:
\begin{lemma}
\label{lem:bigAlpha}
i) $U_{f,\alpha}^{ij,i'j'}=U_f^{ij,i'j'}$ and, ii) if $U_{f\star,\alpha}^{ij,i'j'}\not\equiv0$, it is strictly convex for sufficiently large $\alpha$.
\end{lemma}
\begin{proof}
i) Let us denote by $u_{\star,\alpha}$ and $\omega_{\star,\alpha}$ the optimal solution in $U_{f,\alpha}^{ij,i'j'}$, then our desired result follows from the following fact:
\begin{equation*}
\ResizedEq{U_{f,\alpha}^{ij,i'j'}
\leq&U_{f\star,\alpha}^{ij,i'j'}(d,\theta^{t+1},\theta^t,u_\star-\omega_\star[B^{ij,i'j'}]^T[n_\star^t\times B_1^{ij,i'j'}\alpha],\omega_\star)\\
=&U_{f\star}^{ij,i'j'}
(d,\theta^{t+1},\theta^t,u_\star,\omega_\star)=U_f^{ij,i'j'},}
\end{equation*}
and
\begin{equation*}
\ResizedEq{U_f^{ij,i'j'}
\leq&U_{f\star}^{ij,i'j'}(d,\theta^{t+1},\theta^t,u_{\star,\alpha}+\omega_{\star,\alpha}[B^{ij,i'j'}]^T[n_\star^t\times B_1^{ij,i'j'}\alpha],\omega_{\star,\alpha})\\
=&U_{f\star,\alpha}^{ij,i'j'}
(d,\theta^{t+1},\theta^t,u_{\star,\alpha},\omega_{\star,\alpha})=U_{f,\alpha}^{ij,i'j'}.}
\end{equation*}

ii) We can set $\alpha$ sufficiently large to make:
\begin{equation*}
\ResizedEq{\|[B^{ij,i'j'}]^T[n_\star^t\times[T_i(d,\theta^{t})x_{ij}^m]_3]\|\ll
\|[B^{ij,i'j'}]^T[n_\star^t\times[B_1^{ij,i'j'}\alpha]\|,}
\end{equation*}
leading to non-zero coefficient of $\omega$ in any $\dot{p}_{ij,\parallel,\alpha}^m$, making $U_{f\star,\alpha}^{ij,i'j'}$ strictly convex in $\omega$ if any $\mathcal{A}_\epsilon(f_{ij,\perp}^m)>0$, i.e. $U_{f\star,\alpha}^{ij,i'j'}\not\equiv0$.
\end{proof}

\subsection{Local \& Global Differentiability}
\begin{lemma}
\label{lem:frictionLocalDiff}
If $U_{f\star}^{ij,i'j'}$ is strictly convex in both $\omega$ and $u$, then $U_f^{ij,i'j'}$ is locally differentiable with derivative being:
\begin{align}
\label{eq:derivativeFriction}
\FDD{U_f^{ij,i'j'}}{\THREE{d}{\theta^{t+1}}{\theta^t}}=
\FPP{U_{f\star}^{ij,i'j'}(d,\theta^{t+1},\theta^t,u_\star,\omega_\star)}{\THREE{d}{\theta^{t+1}}{\theta^t}}.
\end{align}
It is also has well-defined mixed-second-derivative being:
\begin{equation}
\label{eq:derivativeFriction2}
\ResizedEq{&\FDDTT{U_f^{ij,i'j'}}{\THREE{d}{\theta^{t+1}}{\theta^t}}{\theta^{t+1}}\\
=&\FPPTT{U_{f\star}^{ij,i'j'}}{\THREE{d}{\theta^{t+1}}{\theta^t}}{\theta^{t+1}}+
\FPPTT{U_{f\star}^{ij,i'j'}}{\THREE{d}{\theta^{t+1}}{\theta^t}}{\TWO{u_\star}{\omega_\star}}\FPP{\TWO{u_\star}{\omega_\star}}{\theta^{t+1}}.}
\end{equation}
\end{lemma}
\begin{proof}
If $U_{f\star}^{ij,i'j'}$ is strictly convex, then it is strictly convex in a local neighborhood. Invoking the inverse function theorem yields:
\begin{equation*}
\ResizedEq{&\FPP{\TWO{u_\star}{\omega_\star}}{\THREE{d}{\theta^{t+1}}{\theta^t}}=-\left[\FPPT{U_{f\star}^{ij,i'j'}}{\TWO{u_\star}{\omega_\star}}\right]^{-1}\FPPTT{U_{f\star}^{ij,i'j'}}{\TWO{u_\star}{\omega_\star}}{\THREE{d}{\theta^{t+1}}{\theta^t}}\\
&\FDD{U_f^{ij,i'j'}}{\THREE{d}{\theta^{t+1}}{\theta^t}}=
\FPP{U_{f\star}^{ij,i'j'}}{\THREE{d}{\theta^{t+1}}{\theta^t}}+
\FPP{U_{f\star}^{ij,i'j'}}{\TWO{u_\star}{\omega_\star}}\FPP{\TWO{u_\star}{\omega_\star}}{\THREE{d}{\theta^{t+1}}{\theta^t}}
=\FPP{U_{f\star}^{ij,i'j'}}{\THREE{d}{\theta^{t+1}}{\theta^t}},}
\end{equation*}
due to the vanish of gradient at $\TWO{u_\star}{\omega_\star}$. The second derivative follows by the chain rule.
\end{proof}

\begin{lemma}
\label{lem:frictionGradientVanish}
If there is a sequence $\THREE{d_n}{\theta_n^{t+1}}{\theta_n^t}\to\THREE{d}{\theta^{t+1}}{\theta^t}$ where $U_{f\star}^{ij,i'j'}(d_n,\theta_n^{t+1},\theta_n^t,u,\omega)\not\equiv0$ but $U_{f\star}^{ij,i'j'}(d,\theta^{t+1},\theta^t,u,\omega)\equiv0$ in $u$ and $\omega$, then:
\begin{align*}
\lim_{n\to\infty}
\FDD{U_f^{ij,i'j'}}{\THREE{d_n}{\theta_n^{t+1}}{\theta_n^t}}=0.
\end{align*}
\end{lemma}
\begin{proof}
By construction, the only situation where $U_{f\star}^{ij,i'j'}\equiv0$ is when all $f_{ij,\perp}^m=0$ and all $f_{i'j',\perp}^m=0$. We have the following estimates:
\begin{equation}
\label{eq:diffFrictionD1}
\ResizedEq{\FPP{D_{ij}^m}{\theta_n^{t+1}}
=\mu\Delta t\mathcal{A}_\epsilon(f_{ij,\perp}^m)
\FPP{\dot{x}_{ij,\parallel}^m}{\theta_n^{t+1}}^T\frac{\dot{x}_{ij,\parallel}^m-\dot{p}_{ij,\parallel}^m}{\sqrt{\left\|\dot{x}_{ij,\parallel}^m-\dot{p}_{ij,\parallel}^m\right\|^2+\epsilon}}\to0,}
\end{equation}
which is because $\mathcal{A}_\epsilon(f_{ij,\perp}^m)\to0$ and other terms are bounded. Further, for the derivative with respect to $d$ and $\theta^t$, we have the following estimates:
\begin{equation}
\label{eq:diffFrictionD2}
\ResizedEq{\FPP{D_{ij}^m}{\TWO{d_n}{\theta_n^t}}
=&\mu\Delta t\mathcal{A}_\epsilon(f_{ij,\perp}^m)
\FPP{\dot{x}_{ij,\parallel}^m}{\TWO{d_n}{\theta_n^t}}^T\frac{\dot{x}_{ij,\parallel}^m-\dot{p}_{ij,\parallel}^m}{\sqrt{\left\|\dot{x}_{ij,\parallel}^m-\dot{p}_{ij,\parallel}^m\right\|^2+\epsilon}}+\\
&\mu\Delta t\FPP{f_{ij,\perp}^m}{\TWO{d_n}{\theta_n^t}}^T\FDD{\mathcal{A}_\epsilon}{f_{ij,\perp}^m}\sqrt{\left\|\dot{x}_{ij,\parallel}^m-\dot{p}_{ij,\parallel}^m\right\|^2+\epsilon}\to0.}
\end{equation}
Indeed, the first righthand side term above tends to zero because $\mathcal{A}_\epsilon(f_{ij,\perp}^m)\to0$. The second righthand side term above tends to zero because:
\begin{equation}
\label{eq:forceVanish}
\ResizedEq{\FPP{f_{ij,\perp}^m}{\TWO{d_n}{\theta_n^t}}=&
P_s'\FPP{n_\star^t}{\TWO{d_n}{\theta_n^t}}+
P_s''n_\star^t\FPP{[p_\star^t]^TT_i(d_n,\theta_n^t)x_{ij}^m}{\TWO{d_n}{\theta_n^t}}^T\to0,}
\end{equation}
with other terms bounded. To verify~\prettyref{eq:forceVanish}, readers can use the same argument as in~\prettyref{lem:contactHessianVanish}. By expanding the derivatives in~\prettyref{eq:forceVanish} using~\prettyref{eq:planeDerivative}, it can be shown that the derivative is a linear combinations of the following terms:
\begin{align*}
&O(P_s'(\bullet)P_s'(\bullet)P_s'(1-\|n_\star\|)^{-1})\\
&O(P_s'(\bullet)P_s''(\bullet)P_s'(1-\|n_\star\|)^{-1})\\
&O(P_s''(\bullet)P_s''(\bullet)P_s'(1-\|n_\star\|)^{-1}),
\end{align*}
each of which vanishes. Finally, the same analysis above apply to $D_{i'j'}^m$. Our desired result follows by plugging the above estimates into~\prettyref{eq:derivativeFriction}.
\end{proof}

\begin{lemma}
\label{lem:frictionGlobalDiff}
$U_f^{ij,i'j'}$ is a globally differentiable function of $d,\theta^{t+1},\theta^t$ with the derivative~\prettyref{eq:derivativeFriction} holds everywhere.
\end{lemma}
\begin{proof}
We need to consider several cases. Case I: If $U_{f\star}^{ij,i'j'}$ is strictly convex in both $\omega$ and $u$, then the result follows by~\prettyref{lem:frictionLocalDiff}. Case II: If $U_{f\star}^{ij,i'j'}$ is strictly convex in $u$ but not $\omega$, then fix a large constant $\alpha$ using~\prettyref{lem:bigAlpha} and resort to case I. Case III: If the Hessian is a zero matrix, then $U_{f\star}^{ij,i'j'}\equiv0$ as a function of $\omega$ and $u$ and there are two sub-cases. Case III.1: If $U_{f\star}^{ij,i'j'}\equiv0$ in a local neighborhood of $\THREE{d}{\theta^{t+1}}{\theta^t}$, then differentiability follows trivially. Case III.2: If there is a sequence $\THREE{d_n}{\theta_n^{t+1}}{\theta_n^t}\to\THREE{d}{\theta^{t+1}}{\theta^t}$ where $U_{f\star}^{ij,i'j'}(d_n,\theta_n^{t+1},\theta_n^t,u,\omega)\not\equiv0$, then we can establish differentiability by applying the limiting argument as in~\prettyref{lem:contactGlobalDiff}, yielding:
\begin{equation}
\ResizedEq{0\leq&\liminf_{n\to\infty}\frac{U_f^{ij,i'j'}(d_n,\theta_n^{t+1},\theta_n^t)-U_f^{ij,i'j'}(d,\theta^{t+1},\theta^t)}{\|\THREE{d_n-d}{\theta_n^{t+1}-\theta^{t+1}}{\theta_n^t-\theta^t}\|}\\
\leq&\limsup_{n\to\infty}\frac{U_f^{ij,i'j'}(d_n,\theta_n^{t+1},\theta_n^t)-U_f^{ij,i'j'}(d,\theta^{t+1},\theta^t)}{\|\THREE{d_n-d}{\theta_n^{t+1}-\theta^{t+1}}{\theta_n^t-\theta^t}\|}\\
\leq&\limsup_{n\to\infty}\left\|\FDD{U_f^{ij,i'j'}}{\THREE{d}{\theta^{t+1}}{\theta^t}}\Big|_{\THREE{\beta_n(d_n-d)+d}{\beta_n(\theta_n^{t+1}-\theta^{t+1})+\theta^{t+1}}{\beta_n(\theta_n^t-\theta^t)+\theta^t}}\right\|=0,}
\end{equation}
where the last limit follows from~\prettyref{lem:frictionGradientVanish}, thus all is proved.
\end{proof}

\subsection{Global Twice-Differentiability}
We largely follow~\prettyref{lem:contactGlobalDiff2} to prove twice differentiability, for which we need to establish some limits.
\begin{lemma}
\label{lem:frictionDerivativeBoundaryVanish}
If $U_{f\star}^{ij,i'j'}(d,\theta^{t+1},\theta^t,u,\omega)\equiv0$ in $u$ and $\omega$, then:
\begin{align*}
&\FPPTT{U_{f\star}^{ij,i'j'}}{\THREE{d}{\theta^{t+1}}{\theta^t}}{\TWO{u_\star}{\omega_\star}}=0.
\end{align*}
\end{lemma}
\begin{proof}
We show that each term in the derivative vanish in this case. By taking derivative with respect to $\TWO{u_\star}{\omega_\star}$ in~\prettyref{eq:diffFrictionD1} and~\prettyref{eq:diffFrictionD2}, we have:
\begin{equation*}
\ResizedEq{&\FPPTT{D_{ij}^m}{\theta^{t+1}}{\TWO{u_\star}{\omega_\star}}
=O(\mathcal{A}_\epsilon(f_{ij,\perp}^m))=0\\
&\FPPTT{D_{ij}^m}{\TWO{d_n}{\theta_n^t}}{\TWO{u_\star}{\omega_\star}}=O(\mathcal{A}_\epsilon(f_{ij,\perp}^m))+
O\left(\left\|\FPP{f_{ij,\perp}^m}{\TWO{d_n}{\theta_n^t}}\right\|\right)=0,}
\end{equation*}
and all is proved. Here all the terms other than those included in $O(\bullet)$ are bounded and the vanishing of $O(\bullet)$ terms have been proved in~\prettyref{lem:frictionGradientVanish}.
\end{proof}

\begin{lemma}
\label{lem:frictionGlobalDiff2}
The mixed-second-derivative~\prettyref{eq:derivativeFriction2} i) well-defined and ii) holds everywhere.
\end{lemma}
\begin{proof}
i) follows from~\prettyref{lem:frictionDerivativeBoundaryVanish} and the convention that $0\cdot\infty=0$. ii) can be proved by considering the set of cases in~\prettyref{lem:frictionGlobalDiff}.

Case I: If $U_{f\star}^{ij,i'j'}$ is strictly convex in both $\omega$ and $u$, then the result follows by~\prettyref{lem:frictionLocalDiff}. Case II: If $U_{f\star}^{ij,i'j'}$ is strictly convex in $u$ but not $\omega$, then fix a large constant $\alpha$ using~\prettyref{lem:bigAlpha} and resort to case I. Case III: If the Hessian is a zero matrix, then $U_{f\star}^{ij,i'j'}\equiv0$ as a function of $\omega$ and $u$. Such case can only happen when all the terms of form $\mathcal{A}_\epsilon(f_{ij,\perp}^m)=0$. But these terms are independent of $\theta^{t+1}$, so $U_{f\star}^{ij,i'j'}\equiv0$ in a local $\theta^{t+1}$-neighborhood and we immediately have the mixed-second-derivative vanishes due to~\prettyref{lem:frictionDerivativeBoundaryVanish} and agrees with~\prettyref{eq:derivativeFriction2}.
\end{proof}

\subsection{Momentum Preservation}
We use the globally valid derivative formula to show the preservation of linear and in-plane angular momentum. Unfortunately, we cannot show the full preservation of angular momentum, which is a limitation of our method that is worth further exploration. 
\begin{lemma}
$U_f^{ij,i'j'}$ preserves i) linear and ii) in-plane angular momentum along $n_\star^t$ of $H_{ij}$ and $H_{i'j'}$.
\end{lemma}
\begin{proof}
i) The total negative force applied on $H_{ij}$ and $H_{i'j'}$ is:
\begin{align*}
f_{ij}=&\sum_{m=1}^M\mu\mathcal{A}_\epsilon(f_{ij,\perp}^m)B^{ij,i'j'}\frac{\dot{x}_{ij,\parallel}^m-\dot{p}_{ij,\parallel}^m}{\sqrt{\left\|\dot{x}_{ij,\parallel}^m-\dot{p}_{ij,\parallel}^m\right\|^2+\epsilon}}\\
f_{i'j'}=&\sum_{m=1}^M\mu\mathcal{A}_\epsilon(f_{i'j',\perp}^m)B^{ij,i'j'}\frac{\dot{x}_{i'j',\parallel}^m-\dot{p}_{i'j',\parallel}^m}{\sqrt{\left\|\dot{x}_{i'j',\parallel}^m-\dot{p}_{i'j',\parallel}^m\right\|^2+\epsilon}}.
\end{align*}
Combined, we have:
\begin{align*}
f_{ij}+f_{i'j'}=\frac{-1}{\Delta t}B^{ij,i'j'}\FPP{U_{f\star}^{ij,i'j'}}{u_\star}=0,
\end{align*}
by the optimality of $u_\star$, proving linear momentum preservation. 

ii) The total negative torque applied on $H_{ij}$ and $H_{i'j'}$ is:
\begin{equation}
\ResizedEq{\tau_{ij}=&\sum_{m=1}^M\mu\mathcal{A}_\epsilon(f_{ij,\perp}^m)[T_i(d,\theta^{t})x_{ij}^m]\times B^{ij,i'j'}\frac{\dot{x}_{ij,\parallel}^m-\dot{p}_{ij,\parallel}^m}{\sqrt{\left\|\dot{x}_{ij,\parallel}^m-\dot{p}_{ij,\parallel}^m\right\|^2+\epsilon}}\\
\tau_{i'j'}=&\sum_{m=1}^M\mu\mathcal{A}_\epsilon(f_{i'j',\perp}^m)[T_{i'}(d,\theta^{t})x_{i'j'}^m]\times B^{ij,i'j'}\frac{\dot{x}_{i'j',\parallel}^m-\dot{p}_{i'j',\parallel}^m}{\sqrt{\left\|\dot{x}_{i'j',\parallel}^m-\dot{p}_{i'j',\parallel}^m\right\|^2+\epsilon}}.}
\end{equation}
By the optimality in $\omega_\star$, we have:
\begin{align*}
[n_\star^t]^T(\tau_{ij}+\tau_{i'j'})=\frac{-1}{\Delta t}\FPP{U_{f\star}^{ij,i'j'}}{\omega_\star}=0,
\end{align*}
proving the preservation of in-plane angular momentum.
\end{proof}
We summarize our results in the following theorem:
\begin{framed}
\begin{theorem}
\label{thm:friction}
The frictional damping potential $U_f^{ij,i'j'}$ has the following properties:
\begin{itemize}
\item Globally, $U_f^{ij,i'j'}=0$ is differentiable and has mixed-second-derivative:
\begin{align*}
\FDDTT{U_f^{ij,i'j'}}{\THREE{d}{\theta^{t+1}}{\theta^t}}{\theta^{t+1}}.
\end{align*}
\item $U_c^{ij,i'j'}$ preserves linear and in-plane angular momentum along $n_\star^t$.
\end{itemize}
\end{theorem}
\end{framed}

%% file: appendix/differentiable.tex
\section{\label{sec:DifferentiabilityProof}Differentiability of SDRS}
We argument informally in this section. The SDRS function is formulated as an unconstrained optimization, with the objective being:
\begin{align*}
\mathcal{O}(d,\theta^{t+1},\theta^t,\theta^{t-1},u^t).
\end{align*}
At the local minimum, we have:
\begin{align*}
\FPP{\mathcal{O}}{\theta^{t+1}}=0.
\end{align*}
Now suppose:
\begin{align*}
\FPPT{\mathcal{O}}{\theta^{t+1}}\succ0,
\end{align*}
and we can invoke the implicit function theorem everywhere, then we have:
\begin{equation}
\label{eq:SDRSDerivative}
\ResizedEq{\FPP{\theta^{t+1}}{\FOUR{d}{\theta^t}{\theta^{t-1}}{u^t}}=
-\left[\FPPT{\mathcal{O}}{\theta^{t+1}}\right]^{-1}\FPPTT{\mathcal{O}}{\theta^{t+1}}{\FOUR{d}{\theta^t}{\theta^{t-1}}{u^t}}.}
\end{equation}
Note that all the derivatives in~\prettyref{eq:SDRSDerivative} are well-defined due to~\prettyref{thm:contact} and~\prettyref{thm:friction}. But if we would like to show that $\theta^{t+1}$ is globally differentiable, then we need to be able to invoke the implicit function theorem everywhere. A sufficient condition for this is to show that: 
\begin{align*}
\FPPT{\mathcal{O}}{\theta^{t+1}}\succ0,
\end{align*}
everywhere, i.e., $\mathcal{O}$ is a strictly convex function of $\theta^{t+1}$ for any choice of $\FOUR{d}{\theta^t}{\theta^{t-1}}{u^t}$. We informally argue that this is possible by choosing sufficiently small $\Delta t$. Note that in the objective function, only the following term is related to the timestep size:
\begin{align*}
\rho\sum_{m=1}^M\frac{\Delta t^2}{2}\|\ddot{x}_{ij}^m\|^2=O(\Delta t^{-2}),
\end{align*}
the Hessian matrix of which is called the generalized mass matrix (scaled by $\Delta t^{-2}$), which is positive definite. We can choose sufficiently small $\Delta t$ such that the generalized mass matrix contributes much large positive eigenvalues to $\FPPT{\mathcal{O}}{\theta^{t+1}}$ than all other terms, making it positive definite everywhere.

%% file: references.bib
@inproceedings{takahashi2021differentiable,
  title={Differentiable fluids with solid coupling for learning and control},
  author={Takahashi, Tetsuya and Liang, Junbang and Qiao, Yi-Ling and Lin, Ming C},
  booktitle={Proceedings of the AAAI Conference on Artificial Intelligence},
  volume={35},
  number={7},
  pages={6138--6146},
  year={2021}
}

@article{liang2019differentiable,
  title={Differentiable cloth simulation for inverse problems},
  author={Liang, Junbang and Lin, Ming and Koltun, Vladlen},
  journal={Advances in Neural Information Processing Systems},
  volume={32},
  year={2019}
}

@inproceedings{qiao2021efficient,
  title={Efficient differentiable simulation of articulated bodies},
  author={Qiao, Yi-Ling and Liang, Junbang and Koltun, Vladlen and Lin, Ming C},
  booktitle={International Conference on Machine Learning},
  pages={8661--8671},
  year={2021},
  organization={PMLR}
}

@inproceedings{pan2018time,
  title={Time Integrating Articulated Body Dynamics Using Position-Based Collocation Methods},
  author={Pan, Zherong and Manocha, Dinesh},
  booktitle={International Workshop on the Algorithmic Foundations of Robotics},
  pages={673--688},
  year={2018},
  organization={Springer}
}

@inproceedings{deng2020cvxnet,
  title={Cvxnet: Learnable convex decomposition},
  author={Deng, Boyang and Genova, Kyle and Yazdani, Soroosh and Bouaziz, Sofien and Hinton, Geoffrey and Tagliasacchi, Andrea},
  booktitle={Proceedings of the IEEE/CVF Conference on Computer Vision and Pattern Recognition},
  pages={31--44},
  year={2020}
}

@book{borwein2006convex,
  title={Convex Analysis},
  author={Borwein, Jonathan and Lewis, Adrian},
  year={2006},
  publisher={Springer}
}

@inproceedings{causey1998gripper,
  title={Gripper design guidelines for modular manufacturing},
  author={Causey, Greg C and Quinn, Roger D},
  booktitle={Proceedings. 1998 IEEE International Conference on Robotics and Automation (Cat. No. 98CH36146)},
  volume={2},
  pages={1453--1458},
  year={1998},
  organization={IEEE}
}

@article{burdick2003minimalist,
  title={Minimalist jumping robots for celestial exploration},
  author={Burdick, Joel and Fiorini, Paolo},
  journal={The International Journal of Robotics Research},
  volume={22},
  number={7-8},
  pages={653--674},
  year={2003},
  publisher={SAGE Publications}
}

@article{10.1177/0278364916648388,
author = {Hubicki, Christian and Grimes, Jesse and Jones, Mikhail and Renjewski, Daniel and Spr\"{o}witz, Alexander and Abate, Andy and Hurst, Jonathan},
title = {ATRIAS},
year = {2016},
issue_date = {10 2016},
publisher = {Sage Publications, Inc.},
address = {USA},
volume = {35},
number = {12},
issn = {0278-3649},
url = {https://doi.org/10.1177/0278364916648388},
doi = {10.1177/0278364916648388},
journal = {Int. J. Rob. Res.},
month = {10},
pages = {1497–1521},
numpages = {25}
}

@inproceedings{kajita2010biped,
  title={Biped walking stabilization based on linear inverted pendulum tracking},
  author={Kajita, Shuuji and Morisawa, Mitsuharu and Miura, Kanako and Nakaoka, Shin'ichiro and Harada, Kensuke and Kaneko, Kenji and Kanehiro, Fumio and Yokoi, Kazuhito},
  booktitle={2010 IEEE/RSJ International Conference on Intelligent Robots and Systems},
  pages={4489--4496},
  year={2010},
  organization={IEEE}
}

@article{poulakakis2009spring,
  title={The spring loaded inverted pendulum as the hybrid zero dynamics of an asymmetric hopper},
  author={Poulakakis, Ioannis and Grizzle, Jessy W},
  journal={IEEE Transactions on Automatic Control},
  volume={54},
  number={8},
  pages={1779--1793},
  year={2009},
  publisher={IEEE}
}

@inproceedings{kim2021mo,
  title={MO-BBO: Multi-Objective Bilevel Bayesian Optimization for Robot and Behavior Co-Design},
  author={Kim, Yeonju and Pan, Zherong and Hauser, Kris},
  booktitle={2021 IEEE International Conference on Robotics and Automation (ICRA)},
  pages={9877--9883},
  year={2021},
  organization={IEEE}
}

@INPROCEEDINGS{Xu-RSS-21, 
    AUTHOR    = {Jie Xu AND Tao Chen AND Lara Zlokapa AND Michael Foshey AND Wojciech Matusik AND Shinjiro Sueda AND Pulkit Agrawal}, 
    TITLE     = {{An End-to-End Differentiable Framework for Contact-Aware Robot Design}}, 
    BOOKTITLE = {Proceedings of Robotics: Science and Systems}, 
    YEAR      = {2021}, 
    ADDRESS   = {Virtual}, 
    MONTH     = {7}, 
    DOI       = {10.15607/RSS.2021.XVII.008} 
}

@article{zhao2020robogrammar,
  title={Robogrammar: graph grammar for terrain-optimized robot design},
  author={Zhao, Allan and Xu, Jie and Konakovi{\'c}-Lukovi{\'c}, Mina and Hughes, Josephine and Spielberg, Andrew and Rus, Daniela and Matusik, Wojciech},
  journal={ACM Transactions on Graphics (TOG)},
  volume={39},
  number={6},
  pages={1--16},
  year={2020},
  publisher={ACM New York, NY, USA}
}

@article{ha2018computational,
  title={Computational co-optimization of design parameters and motion trajectories for robotic systems},
  author={Ha, Sehoon and Coros, Stelian and Alspach, Alexander and Kim, Joohyung and Yamane, Katsu},
  journal={The International Journal of Robotics Research},
  volume={37},
  number={13-14},
  pages={1521--1536},
  year={2018},
  publisher={SAGE Publications Sage UK: London, England}
}

@inproceedings{liao2019data,
  title={Data-efficient learning of morphology and controller for a microrobot},
  author={Liao, Thomas and Wang, Grant and Yang, Brian and Lee, Rene and Pister, Kristofer and Levine, Sergey and Calandra, Roberto},
  booktitle={2019 International Conference on Robotics and Automation (ICRA)},
  pages={2488--2494},
  year={2019},
  organization={IEEE}
}

@inproceedings{hu2019chainqueen,
  title={Chainqueen: A real-time differentiable physical simulator for soft robotics},
  author={Hu, Yuanming and Liu, Jiancheng and Spielberg, Andrew and Tenenbaum, Joshua B and Freeman, William T and Wu, Jiajun and Rus, Daniela and Matusik, Wojciech},
  booktitle={2019 International conference on robotics and automation (ICRA)},
  pages={6265--6271},
  year={2019},
  organization={IEEE}
}

@article{song2020learning,
  title={Learning to slide unknown objects with differentiable physics simulations},
  author={Song, Changkyu and Boularias, Abdeslam},
  journal={arXiv preprint arXiv:2005.05456},
  year={2020}
}

@article{chen2022imitation,
  title={Imitation learning via differentiable physics},
  author={Chen, Siwei and Ma, Xiao and Xu, Zhongwen},
  journal={arXiv preprint arXiv:2206.04873},
  year={2022}
}

@article{geilinger2020add,
  title={Add: Analytically differentiable dynamics for multi-body systems with frictional contact},
  author={Geilinger, Moritz and Hahn, David and Zehnder, Jonas and B{\"a}cher, Moritz and Thomaszewski, Bernhard and Coros, Stelian},
  journal={ACM Transactions on Graphics (TOG)},
  volume={39},
  number={6},
  pages={1--15},
  year={2020},
  publisher={ACM New York, NY, USA}
}

@inproceedings{werling2021fast,
  title={Fast and feature-complete differentiable physics engine for articulated rigid bodies with contact constraints},
  author={Werling, Keenon and Omens, Dalton and Lee, Jeongseok and Exarchos, Ioannis and Liu, C Karen},
  booktitle={Robotics: Science and Systems},
  year={2021}
}

@article{honig2018trajectory,
  title={Trajectory planning for quadrotor swarms},
  author={H{\"o}nig, Wolfgang and Preiss, James A and Kumar, TK Satish and Sukhatme, Gaurav S and Ayanian, Nora},
  journal={IEEE Transactions on Robotics},
  volume={34},
  number={4},
  pages={856--869},
  year={2018},
  publisher={IEEE}
}

@inproceedings{wang2022curriculum,
  title={Curriculum-based Co-design of Morphology and Control of Voxel-based Soft Robots},
  author={Wang, Yuxing and Wu, Shuang and Fu, Haobo and Fu, Qiang and Zhang, Tiantian and Chang, Yongzhe and Wang, Xueqian},
  booktitle={The Eleventh International Conference on Learning Representations},
  year={2022}
}

@article{fadini2022simulation,
  title={Simulation aided co-design for robust robot optimization},
  author={Fadini, Gabriele and Flayols, Thomas and Del Prete, Andrea and Sou{\`e}res, Philippe},
  journal={IEEE Robotics and Automation Letters},
  volume={7},
  number={4},
  pages={11306--11313},
  year={2022},
  publisher={IEEE}
}

@inproceedings{dinev2022versatile,
  title={A versatile co-design approach for dynamic legged robots},
  author={Dinev, Traiko and Mastalli, Carlos and Ivan, Vladimir and Tonneau, Steve and Vijayakumar, Sethu},
  booktitle={2022 IEEE/RSJ International Conference on Intelligent Robots and Systems (IROS)},
  pages={10343--10349},
  year={2022},
  organization={IEEE}
}

@article{ma2021diffaqua,
  title={DiffAqua: A Differentiable Computational Design Pipeline for Soft Underwater Swimmers with Shape Interpolation},
  author={Ma, Pingchuan and Du, Tao and Zhang, John Z and Wu, Kui and Spielberg, Andrew and Katzschmann, Robert K and Matusik, Wojciech},
  journal={ACM Transactions on Graphics (TOG)},
  volume={40},
  number={4},
  pages={132},
  year={2021},
  publisher={ACM New York, NY, USA}
}

@article{su2023generalized,
  title={A Generalized Constitutive Model for Versatile MPM Simulation and Inverse Learning with Differentiable Physics},
  author={Su, Haozhe and Li, Xuan and Xue, Tao and Jiang, Chenfanfu and Aanjaneya, Mridul},
  journal={Proceedings of the ACM on Computer Graphics and Interactive Techniques},
  volume={6},
  number={3},
  pages={1--20},
  year={2023},
  publisher={ACM New York, NY, USA}
}

@article{stuyck2023diffxpbd,
  title={DiffXPBD: Differentiable Position-Based Simulation of Compliant Constraint Dynamics},
  author={Stuyck, Tuur and Chen, Hsiao-yu},
  journal={Proceedings of the ACM on Computer Graphics and Interactive Techniques},
  volume={6},
  number={3},
  pages={1--14},
  year={2023},
  publisher={ACM New York, NY, USA}
}

@article{de2018end,
  title={End-to-end differentiable physics for learning and control},
  author={de Avila Belbute-Peres, Filipe and Smith, Kevin and Allen, Kelsey and Tenenbaum, Josh and Kolter, J Zico},
  journal={Advances in neural information processing systems},
  volume={31},
  year={2018}
}

@article{ma2022risp,
  title={Risp: Rendering-invariant state predictor with differentiable simulation and rendering for cross-domain parameter estimation},
  author={Ma, Pingchuan and Du, Tao and Tenenbaum, Joshua B and Matusik, Wojciech and Gan, Chuang},
  journal={arXiv preprint arXiv:2205.05678},
  year={2022}
}

@article{guendelman2003nonconvex,
  title={Nonconvex rigid bodies with stacking},
  author={Guendelman, Eran and Bridson, Robert and Fedkiw, Ronald},
  journal={ACM transactions on graphics (TOG)},
  volume={22},
  number={3},
  pages={871--878},
  year={2003},
  publisher={ACM New York, NY, USA}
}

@inproceedings{catto2005iterative,
  title={Iterative dynamics with temporal coherence},
  author={Catto, Erin},
  booktitle={Game developer conference},
  volume={2},
  number={5},
  year={2005}
}

@inproceedings{Harmon2009ACM,
  author = {David Harmon and Etienne Vouga and Breannan Smith and Rasmus Tamstorf and Eitan Grinspun},
  title = {Asynchronous contact mechanics},
  journal = {SIGGRAPH '09 (ACM Transactions on Graphics)},
  year = {2009},
  isbn = {978-1-60558-726-4},
  publisher = {ACM},
  address = {New York, NY, USA}
}

@article{Li2020IPC,
  author = {Minchen Li and Zachary Ferguson and Teseo Schneider and Timothy Langlois and Denis Zorin and Daniele Panozzo and Chenfanfu Jiang and Danny M. Kaufman},
  title = {Incremental Potential Contact: Intersection- and Inversion-free Large Deformation Dynamics},
  journal = {ACM Trans. Graph. (SIGGRAPH)},
  year = {2020},
  volume = {39},
  number = {4},
  articleno = {49}
}

@article{iusem2003convergence,
  title={On the convergence properties of the projected gradient method for convex optimization},
  author={Iusem, Alfredo N},
  journal={Computational \& Applied Mathematics},
  volume={22},
  pages={37--52},
  year={2003},
  publisher={SciELO Brasil}
}

@ARTICLE{5719567,
  author={Tan, Jie and Liu, Karen and Turk, Greg},
  journal={IEEE Computer Graphics and Applications}, 
  title={Stable Proportional-Derivative Controllers}, 
  year={2011},
  volume={31},
  number={4},
  pages={34-44},
  doi={10.1109/MCG.2011.30}}

@article{chen2005tutorial,
  title={A tutorial on $\nu$-support vector machines},
  author={Chen, Pai-Hsuen and Lin, Chih-Jen and Sch{\"o}lkopf, Bernhard},
  journal={Applied Stochastic Models in Business and Industry},
  volume={21},
  number={2},
  pages={111--136},
  year={2005},
  publisher={Wiley Online Library}
}

@incollection{kaufman2008staggered,
  title={Staggered projections for frictional contact in multibody systems},
  author={Kaufman, Danny M and Sueda, Shinjiro and James, Doug L and Pai, Dinesh K},
  booktitle={ACM SIGGRAPH Asia 2008 papers},
  pages={1--11},
  year={2008}
}

@article{mamou2016volumetric,
  title={Volumetric hierarchical approximate convex decomposition},
  author={Mamou, Khaled and Lengyel, E and Peters, A},
  journal={Game engine gems},
  volume={3},
  pages={141--158},
  year={2016},
  publisher={AK Peters}
}

@incollection{martin2011example,
  title={Example-based elastic materials},
  author={Martin, Sebastian and Thomaszewski, Bernhard and Grinspun, Eitan and Gross, Markus},
  booktitle={ACM SIGGRAPH 2011 papers},
  pages={1--8},
  year={2011}
}

@article{li2021codimensional,
  title={Codimensional incremental potential contact},
  author={Li, Minchen and Kaufman, Danny M and Jiang, Chenfanfu},
  journal={ACM Transactions on Graphics (TOG)},
  volume={40},
  number={4},
  pages={1--24},
  year={2021},
  publisher={ACM New York, NY, USA}
}

@article{lindemann2009gilbert,
  title={The gilbert-johnson-keerthi distance algorithm},
  author={Lindemann, Patrick},
  journal={Algorithms in Media Informatics},
  year={2009}
}

@inproceedings{tassa2012synthesis,
  title={Synthesis and stabilization of complex behaviors through online trajectory optimization},
  author={Tassa, Yuval and Erez, Tom and Todorov, Emanuel},
  booktitle={2012 IEEE/RSJ International Conference on Intelligent Robots and Systems},
  pages={4906--4913},
  year={2012},
  organization={IEEE}
}

@article{ferguson2021intersection,
  title={Intersection-free rigid body dynamics},
  author={Ferguson, Zachary and Li, Minchen and Schneider, Teseo and Gil-Ureta, Francisca and Langlois, Timothy and Jiang, Chenfanfu and Zorin, Denis and Kaufman, Danny M and Panozzo, Daniele},
  journal={ACM Transactions on Graphics},
  volume={40},
  number={4},
  year={2021}
}

@article{suthaharan2016support,
  title={Support vector machine},
  author={Suthaharan, Shan and Suthaharan, Shan},
  journal={Machine learning models and algorithms for big data classification: thinking with examples for effective learning},
  pages={207--235},
  year={2016},
  publisher={Springer}
}

@inproceedings{erez2015simulation,
  title={Simulation tools for model-based robotics: Comparison of bullet, havok, mujoco, ode and physx},
  author={Erez, Tom and Tassa, Yuval and Todorov, Emanuel},
  booktitle={2015 IEEE international conference on robotics and automation (ICRA)},
  pages={4397--4404},
  year={2015},
  organization={IEEE}
}

@article{paszke2017automatic,
  title={Automatic differentiation in pytorch},
  author={Paszke, Adam and Gross, Sam and Chintala, Soumith and Chanan, Gregory and Yang, Edward and DeVito, Zachary and Lin, Zeming and Desmaison, Alban and Antiga, Luca and Lerer, Adam},
  year={2017}
}

@article{grant2009cvx,
  title={CVX users’ guide},
  author={Grant, Michael and Boyd, Stephen and Ye, Yinyu},
  journal={online: http://www. stanford. edu/boyd/software. html},
  year={2009},
  publisher={Citeseer}
}

@article{howelllecleach2022,
	title={Dojo: A Differentiable Physics Engine for Robotics},
	author={Howell, Taylor and Le Cleac'h, Simon and Bruedigam, Jan and Kolter, Zico and Schwager, Mac and Manchester, Zachary},
	journal={arXiv preprint arXiv:2203.00806},
	url={https://arxiv.org/abs/2203.00806},
	year={2022}
}

@inproceedings{tracy2023differentiable,
  title={Differentiable collision detection for a set of convex primitives},
  author={Tracy, Kevin and Howell, Taylor A and Manchester, Zachary},
  booktitle={2023 IEEE International Conference on Robotics and Automation (ICRA)},
  pages={3663--3670},
  year={2023},
  organization={IEEE}
}

@inproceedings{
yuan2022transformact,
title={Transform2Act: Learning a Transform-and-Control Policy for Efficient Agent Design},
author={Ye Yuan and Yuda Song and Zhengyi Luo and Wen Sun and Kris M. Kitani},
booktitle={International Conference on Learning Representations},
year={2022},
url={https://openreview.net/forum?id=UcDUxjPYWSr}
}

@article{10.1145/2816795.2818137,
author = {Megaro, Vittorio and Thomaszewski, Bernhard and Nitti, Maurizio and Hilliges, Otmar and Gross, Markus and Coros, Stelian},
title = {Interactive design of 3D-printable robotic creatures},
year = {2015},
issue_date = {November 2015},
publisher = {Association for Computing Machinery},
address = {New York, NY, USA},
volume = {34},
number = {6},
issn = {0730-0301},
url = {https://doi.org/10.1145/2816795.2818137},
doi = {10.1145/2816795.2818137},
journal = {ACM Trans. Graph.},
month = nov,
articleno = {216},
numpages = {9}
}

@article{gradsim,
  title   = {gradSim: Differentiable simulation for system identification and visuomotor control},
  author  = {Krishna Murthy Jatavallabhula and Miles Macklin and Florian Golemo and Vikram Voleti and Linda Petrini and Martin Weiss and Breandan Considine and Jerome Parent-Levesque and Kevin Xie and Kenny Erleben and Liam Paull and Florian Shkurti and Derek Nowrouzezahrai and Sanja Fidler},
  journal = {International Conference on Learning Representations (ICLR)},
  year    = {2021},
  url     = {https://openreview.net/forum?id=c_E8kFWfhp0},
}

@INPROCEEDINGS{heiden2021disect, 
    AUTHOR    = {Eric Heiden AND Miles Macklin AND Yashraj S Narang AND Dieter Fox AND Animesh Garg AND Fabio Ramos}, 
    TITLE     = {{DiSECt: A Differentiable Simulation Engine for Autonomous Robotic Cutting}}, 
    BOOKTITLE = {Proceedings of Robotics: Science and Systems}, 
    YEAR      = {2021}, 
    ADDRESS   = {Virtual}, 
    MONTH     = {7}, 
    DOI       = {10.15607/RSS.2021.XVII.067} 
}

@article{du2021_diffpd,
    author = {Du, Tao and Wu, Kui and Ma, Pingchuan and Wah, Sebastien and Spielberg, Andrew and Rus, Daniela and Matusik, Wojciech},
    title = {DiffPD: Differentiable Projective Dynamics},
    year = {2021},
    issue_date = {April 2022},
    publisher = {Association for Computing Machinery},
    address = {New York, NY, USA},
    volume = {41},
    number = {2},
    issn = {0730-0301},
    url = {https://doi.org/10.1145/3490168},
    doi = {10.1145/3490168},
    journal = {ACM Trans. Graph.},
    month = {11},
    articleno = {13},
    numpages = {21}
}

@article{huang2024differentiable,
  title={Differentiable solver for time-dependent deformation problems with contact},
  author={Huang, Zizhou and Tozoni, Davi Colli and Gjoka, Arvi and Ferguson, Zachary and Schneider, Teseo and Panozzo, Daniele and Zorin, Denis},
  journal={ACM Transactions on Graphics},
  volume={43},
  number={3},
  pages={1--30},
  year={2024},
  publisher={ACM New York, NY}
}

@inproceedings{gjoka2024soft,
  title={Soft Pneumatic Actuator Design using Differentiable Simulation},
  author={Gjoka, Arvi and Knoop, Espen and B{\"a}cher, Moritz and Zorin, Denis and Panozzo, Daniele},
  booktitle={ACM SIGGRAPH 2024 Conference Papers},
  pages={1--11},
  year={2024}
}

@article{maloisel2023optimal,
  title={Optimal Design of Robotic Character Kinematics},
  author={Maloisel, Guirec and Schumacher, Christian and Knoop, Espen and Grandia, Ruben and B{\"a}cher, Moritz},
  journal={ACM Transactions on Graphics (TOG)},
  volume={42},
  number={6},
  pages={1--15},
  year={2023},
  publisher={ACM New York, NY, USA}
}

@article{10.1145/3453477,
author = {Zhang, Ran and Auzinger, Thomas and Bickel, Bernd},
title = {Computational Design of Planar Multistable Compliant Structures},
year = {2021},
issue_date = {October 2021},
publisher = {Association for Computing Machinery},
address = {New York, NY, USA},
volume = {40},
number = {5},
issn = {0730-0301},
url = {https://doi.org/10.1145/3453477},
doi = {10.1145/3453477},
journal = {ACM Trans. Graph.},
month = oct,
articleno = {186},
numpages = {16}
}

@article{spielberg2019learning,
  title={Learning-in-the-loop optimization: End-to-end control and co-design of soft robots through learned deep latent representations},
  author={Spielberg, Andrew and Zhao, Allan and Hu, Yuanming and Du, Tao and Matusik, Wojciech and Rus, Daniela},
  journal={Advances in Neural Information Processing Systems},
  volume={32},
  year={2019}
}

@inproceedings{10.1145/3309486.3340246,
author = {Pan, Zherong and Ren, Bo and Manocha, Dinesh},
title = {GPU-based contact-aware trajectory optimization using a smooth force model},
year = {2019},
isbn = {9781450366779},
publisher = {Association for Computing Machinery},
address = {New York, NY, USA},
url = {https://doi.org/10.1145/3309486.3340246},
doi = {10.1145/3309486.3340246},
booktitle = {Proceedings of the 18th Annual ACM SIGGRAPH/Eurographics Symposium on Computer Animation},
articleno = {4},
numpages = {12},
series = {SCA '19}
}

@article{10.1145/2185520.2185539,
author = {Mordatch, Igor and Todorov, Emanuel and Popovi\'{c}, Zoran},
title = {Discovery of complex behaviors through contact-invariant optimization},
year = {2012},
issue_date = {July 2012},
publisher = {Association for Computing Machinery},
address = {New York, NY, USA},
volume = {31},
number = {4},
issn = {0730-0301},
url = {https://doi.org/10.1145/2185520.2185539},
doi = {10.1145/2185520.2185539},
journal = {ACM Trans. Graph.},
month = jul,
articleno = {43},
numpages = {8}
}

@article{hazard2020automated,
  title={Automated design of robotic hands for in-hand manipulation tasks},
  author={Hazard, Christopher and Pollard, Nancy and Coros, Stelian},
  journal={International Journal of Humanoid Robotics},
  volume={17},
  number={01},
  pages={1950029},
  year={2020},
  publisher={World Scientific}
}

@article{desai2018interactive,
  title={Interactive co-design of form and function for legged robots using the adjoint method},
  author={Desai, Ruta and Li, Beichen and Yuan, Ye and Coros, Stelian},
  journal={arXiv preprint arXiv:1801.00385},
  year={2018}
}

@article{kingma2014adam,
  title={Adam: A method for stochastic optimization},
  author={Kingma, Diederik P},
  journal={arXiv preprint arXiv:1412.6980},
  year={2014}
}
